\documentclass[twoside,11pt]{article}
\pdfoutput=1
\usepackage{jmlr2e}

\usepackage{dsfont}
\usepackage{amsmath}
\newcommand{\argmax}{\operatornamewithlimits{argmax}}
\newcommand{\argmin}{\operatornamewithlimits{argmin}}
\newcommand{\sign}{\operatorname{sign}}

\usepackage{algorithm}
\usepackage{algpseudocode}

\DeclareMathAlphabet{\mathcal}{OT1}{pzc}{m}{it}

\usepackage{booktabs}
\usepackage{array}
\usepackage{esdiff}
\usepackage{color}

\firstpageno{1}

\begin{document}

\title{Optimal Nudging: Solving Average-Reward Semi-Markov Decision Processes as a Minimal Sequence of Cumulative Tasks}

\author{\name Reinaldo Uribe Muriel \email r-uribe@uniandes.edu.co \\
        \name Fernando Lozano \email flozano@uniandes.edu.co \\
        \addr Department of Electric and Electronic Engineering\\
        Universidad de los Andes\\
        Bogot\'{a}, DC 111711, Colombia
        \AND
        \name Charles Anderson \email anderson@cs.colostate.edu\\
        \addr Department of Computer Science\\
              Colorado State University\\
              Fort Collins, CO 80523-1873, USA}

\editor{\noindent\hspace{-1.5cm}\colorbox{white}{\phantom{EditorE}}}

\maketitle

\begin{abstract}
This paper describes a novel method to solve average-reward semi-Markov decision processes, by reducing them to a minimal sequence of cumulative reward problems. The usual solution methods for this type of problems update the gain (optimal average reward) immediately after observing the result of taking an action. The alternative introduced, \emph{optimal nudging}, relies instead on setting the gain to some fixed value, which transitorily makes the problem a cumulative-reward task, solving it by any standard reinforcement learning method, and only then updating the gain in a way that minimizes uncertainty in a \texttt{minmax} sense. The rule for optimal gain update is derived by exploiting the geometric features of the $w-l$ space, a simple mapping of the space of policies. The total number of cumulative reward tasks that need to be solved is shown to be small. Some experiments are presented to explore the features of the algorithm and to compare its performance with other approaches.
\end{abstract}

\begin{keywords}
  Reinforcement Learning, Average Rewards, Semi-Markov Decision Processes.
\end{keywords}

\section{Introduction}
Consider a simple game, some solitaire variation, for example, or a board game against a fixed opponent. Assume that neither draws nor unending matches are possible in this game and that the only payouts are $+\$1$ on winning and $-\$1$ on losing. If each new game position (\emph{state}) depends stochastically only on the preceding one and the move (\emph{action}) made, but not on the history of positions and moves before that, then the game can be modelled as a Markov decision process. Further, since all matches terminate, the process is \emph{episodic}.

\emph{Learning} the game, or \emph{solving} the decision process, is equivalent to finding the best playing strategy (\emph{policy}), that is, determining what moves to make on each position in order to maximize the probability of winning / expected payout. This is the type of problem commonly solved by cumulative-reward reinforcement learning methods.

Now, assume that, given the nature of this game, as is often the case, the winning probability is optimized by some \emph{cautious} policy, whose gameplay favours avoiding risks and hence results in relatively long matches. For example, assume that for this game a policy is known to win with certainty in 100 moves. On the other hand, typically some policies trade performance (winning probability) for speed (shorter episode lengths). Assume another known policy, obviously sub-optimal in the sense of expected payout, has a winning probability of just 0.6, but it takes only 10 moves in average to terminate.

If one is going to play a single episode, doubtlessly the first strategy is the best available, since following it winning is guaranteed. However, over a sequence of games, the second policy may outperform the `optimal' one in a very important sense: if each move \emph{costs} the same (for instance, if all take the same amount of time to complete), whereas the policy that always wins receives in average \$0.01/move, the other strategy earns twice as much. Thus, over an hour, or a day, or a lifetime of playing, the ostensibly sub-optimal game strategy will double the earnings of the apparent optimizer. This is a consequence of the fact that the second policy has a higher \emph{average reward}, receiving a larger payout per action taken. Finding policies that are optimal in this sense is the problem solved by average-reward reinforcement learning.

In a more general case, if each move has associated a different cost, such as the time it would take to move a token on a board the number of steps dictated by a die, then the problem would be \emph{average-reward semi-Markov}, and the goal would change to finding a policy, possibly different from either of the two discussed above, that maximizes the expected amount of payout received per unit of action cost.

Average-reward and semi-Markov tasks arise naturally in the areas of repeated episodic tasks, as in the example just discussed, queuing theory, autonomous robotics, and quality of service in communications, among many others.

This paper presents a new algorithm to solve average-reward and semi-Markov decision processes. The traditional solutions to this kind of problems require a large number of \emph{samples}, where a sample is usually the observation of the effect of taking an action from a state: the cost of the action, the reward received and the resulting next state. For each sample, the algorithms basically update the gain (average reward) of the task and the gain-adjusted value of, that is, what a good idea is it, taking that action from that state.

Some of the methods in the literature that follow this solution template are R-learning \citep{schwartz1993reinforcement}, the Algorithms 3 and 4 by \citet{singh1994reinforcement}, SMART \citep{das1999solving}, and the ``New Algorithm'' by \cite{gosavi2004reinforcement}. Table \ref{tab:avgrw} in Section 3.2 introduces a more comprehensive taxonomy of solution methods for average-reward and semi-Markov problems.

The method introduced in this paper, \emph{optimal nudging}, operates differently. Instead of rushing to update the gain after each sample, it is temporarily fixed to some value, resulting in a cumulative-reward task that is solved (by any method), and then, based on the solution found, the gain is updated in a way that minimizes the uncertainty range known to contain its optimum.

The main contribution of this paper is the introduction of a novel algorithm to solve semi-Markov (and simpler average-reward) decision processes by reducing them to a minimal sequence of cumulative-reward tasks, that can be solved by any of the fast, robust existing methods for that kind of problems. Hence, we refer to the method used to solve this tasks as a `black box'.

The central step of the optimal nudging algorithm is a rule for updating the gain between calls to the `black-box' solver, in such a way that after solving the resulting cumulative-reward task, the worst case for the associated uncertainty around the value of the optimal gain will be the smallest possible.

The update rule exploits what we have called a ``Bertsekas split'' of each task as well as the geometry of the $w-l$ space, a mapping of the policy space into the interior of a small triangle, in which the convergence of the solutions of the cumulative-reward tasks to the optimal solution of the average-reward problem can be easily intuited and visualized.

In addition to this, the derivation of the optimal nudging update rule yields an early termination condition, related to sign changes in the value of a reference state between successive iterations for which the same policy is optimal. This condition is unique to optimal nudging, and no equivalent is possible for the preceding algorithms in the literature.

The complexity of optimal nudging, understood as the number of calls to the ``black-box'' routine, is shown to be at worst logarithmic on the (inverse) desired final uncertainty and on an upper bound on the optimal gain. The number of samples required in each call is in principle inherited from the ``black box'', but also depends strongly on whether, for example, transfer learning is used and state values are not reset between iterations.

Among other advantages of the proposed algorithm over other methods discussed, two key ones are requiring adjustment of less parameters, and having to perform less updates per sample.

Finally, the experimental results presented show that the performance of optimal nudging, even without fine tuning, is similar or better to that of the best usual algorithms. The experiments also illuminate some particular features of the algorithm, particularly the great advantage of having the early termination condition.

The rest of the paper is structured as follows. Section 2 formalizes the problem, defining the different types of Markov decision processes (cumulative- and average-reward and semi-Markov) under an unified notation, introducing the important unichain condition and describing why it is important to assume that it holds.

Section 3 presents a summary of solution methods to the three types of processes, emphasizing the distinctions between dynamic programming and model-based and model-free reinforcement learning algorithms. This section also introduces a new taxonomy of the average-reward algorithms from the literature that allows us to propose a generic algorithm that encompasses all of them.
Special attention is given in this Section to the family of stochastic shortest path methods, from which the concept of the Bertsekas split is extracted.
Finally, a motivating example task is introduced to compare the performance of some of the traditional algorithms and optimal nudging.

In Section 4, the core derivation of the optimal nudging algorithm is presented, starting from the idea of \emph{nudging} and the definition of the $w-l$ space and enclosing triangles. The early termination condition by zero crossing is presented as an special case of reduction of enclosing triangles, and the exploration of optimal reduction leads to the main Theorem and final proposition of the algorithm.

Section 5 describes the complexity of the algorithm by showing that it outperforms a simpler version of nudging for which the computation of complexity is straightforward.

Finally, Section 6 presents results for a number of experimental set-ups and in Section 7 some conclusions and lines for future work are discussed.

\section{Problem Definition}
In this section, Markov decision processes are described and three different reward maximization problems are introduced: expected cumulative reinforcement, average-reward, and average reward semi-Markov models.
Average-reward problems are a subset of semi-Markov average reward problems.
This paper introduces a method to solve both kinds of average-reward problems as a minimal sequence of cumulative-reward, episodic processes. Two relevant common assumptions in average reward models, that the unichain condition holds \citep{ross1970average} and a recurrent state exists \citep{bertsekas1998new,abounadi2002learning}, are described and discussed at the end of the section.

In all cases, an agent in an environment observes its current state and can take actions that, following some static distribution, lead it to a new state and result in a real-valued reinforcement/reward. It is assumed that the Markov property holds, so the next state depends only on the current state and action taken, but not on the history of previous states and actions.

\subsection{Markov Decision Processes}
An infinite-horizon Markov decision process \citep[MDP,][]{sutton1998reinforcement,puterman1994markov} is defined minimally as a four-tuple $\langle \mathcal{S,A,P,R} \rangle$. $\mathcal{S}$ is the set of states in the environment. $\mathcal{A}$ is the set of actions, with $\mathcal{A}_s$ equal to the subset of actions available to take from state $s$ and $\mathcal{A}=\bigcup_{s\in \mathcal{S}}\mathcal{A}_s$. We assume that both $\mathcal{S}$ and $\mathcal{A}$ are finite.
The stationary function $\mathcal{P:S\times S\times A}\rightarrow[0,1]$ defines the transition probabilities of the system. After taking action $a$ from state $s$, the resulting state is $s'$ with probability $\mathcal{P}^a_{ss'}=P(s'|s,a)$. Likewise, $\mathcal{R}^a_{ss'}$ denotes the real-valued reward observed after taking action $a$ and transitioning from state $s$ to $s'$. For notational simplicity, we define $r(s,a)=\mathbb{E}\left[\mathcal{R}^a_{ss'}|s,a\right]$.
At decision epoch $t$, the agent is in state $s_t$, takes action $a_t$, transitions to state $s_{t+1}$ and receives reinforcement $r_{t+1}$, which has expectation $r(s_t,a_t)$.

If the task is \emph{episodic}, there must be a terminating state, defined as transitioning to itself with probability 1 and reward 0. Without loss of generality, multiple terminating states can be treated as a single one.

An element $\pi:\mathcal{S\rightarrow A}$ of the \emph{policy} space $\Pi$ is a rule or strategy that dictates for each state $s$ which action to take, $\pi(s)$. We are only concerned with deterministic policies, in which each state has associated a single action, to take with probability one. This is not too restrictive, since \citet{puterman1994markov} has shown that, if an optimal policy exists, an optimal deterministic policy exists as well. Moreover, policies are assumed herein to be stationary. The \emph{value} of a policy from a given state, $v^\pi(s)$, is the expected cumulative reward observed starting from $s$ and following $\pi$,
\begin{align}
v^\pi(s) &= \mathbb{E}\left[ \sum_{t=0}^\infty \gamma^t r(s_t,\pi(s_t))\; |\; s_0=s,\,\pi\right] \;\;, \label{eq:value}
\end{align}
where $0<\gamma\leq 1$ is a discount factor with $\gamma=1$ corresponding to no discount.

The goal is to find a policy that maximizes the expected reward. Thus, an \emph{optimal policy} $\pi^*$ has maximum value for each state; that is,
\begin{align*}
\pi^*(s) &\in \argmax_{\pi\in\Pi}v^\pi(s)\quad\forall s\in\mathcal{S}\;\;,\\
\intertext{so}
v^*(s)&= v^{\pi^*}(s)  \geq v^\pi(s)\quad\forall s\in\mathcal{S},\,\pi\in\Pi\;\;.
\end{align*}

\begin{remark}\label{rem:discount}
The discount factor ensures convergence of the infinite sum in the policy values Equation (\ref{eq:value}), so it is used to make value bounded if rewards are bounded, even in problems and for policies for which episodes have infinite duration. Ostensibly, introducing it makes rewards received sooner more desirable than those received later, which would make it useful when the goal is to optimize a measure of immediate (or average) reward. However, for the purposes of this paper, the relevant policies for the discussed resulting MDPs will be assumed to terminate eventually with probability one from all states, so the infinite sum will converge even without discount. Furthermore, the perceived advantages of discount are less sturdy than initially apparent \citep{mahadevan1994discount}, and discounting is not guaranteed to lead to gain optimality \citep{uribe2011discount}. Thus, no discount will be used in this paper ($\gamma=1$).
\end{remark}

\subsection{Average Reward MDPs}
The aim of the average reward model in infinite-horizon MDPs is to maximize the reward received \emph{per step} \citep{puterman1994markov,mahadevan1996average}. Without discount all non-zero-valued policies would have signed infinite value, so the goal must change to obtaining the largest positive or the smallest negative rewards as frequently as possible. In this case, the \emph{gain} of a policy from a state is defined as the average reward received per action taken following that policy from the state,
\begin{align*}
	\rho_{AR}^\pi(s)&=\lim_{n\rightarrow\infty}\frac{1}{n}\mathbb{E}\left[\sum_{t=0}^{n-1}r(s_t,\pi(s_t))\;|\;s_0=s,\,\pi\right]\;\;.
\end{align*}

A gain-optimal policy, $\pi_{AR}^*$, has maximum average reward, $\rho_{AR}^*=\rho_{AR}^{\pi^*}$, for all states,
\begin{align*}
	\rho_{AR}^*(s) \geq \rho_{AR}^\pi(s)\quad\forall s\in\mathcal{S},\,\pi\in\Pi\;\;.
\end{align*}

A finer typology of optimal policies in average-reward problems discriminates \emph{bias-optimal} policies which, besides being gain-optimal, also maximize the transient reward received \emph{before} the observed average approaches $\rho_{AR}^*$. For a discussion of the differences, see the book by \citet{puterman1994markov}. This paper will focus on the problem of finding gain-optimal policies.

\subsection{Discrete Time Semi-MDPs}
In the average-reward model all state transitions weigh equally. Equivalently, all actions from all states are considered as having the same---unity---duration or cost. In semi-Markov decision processes \citep[SMDPs,][]{ross1970average} the goal remains maximizing the average reward received per action, but all actions are not required to have the same weight.

\subsubsection{Transition Times}
The usual description of SMDPs assumes that, after taking an action, the time to transition to a new state is not constant \citep{feinberg1994constrained,das1999solving,baykal2007semi,ghavamzadeh2007hierarchical}. Formally, at decision epoch $t$ the agent is in state $s_t$ and only after an average of $N_t$ seconds of having taken action $a_t$ it evolves to state $s_{t+1}$ and observes reward $r_{t+1}$.

The transition time function, then, is $\mathcal{N: S\times S\times A} \rightarrow\mathds{R_+}$ (where $\mathds{R_+}$ is the set of positive real numbers).
Also, since reward can possibly lump between decision epochs, its expectation, $r(s_t,a_t)$, is marginalized over expected transition times as well as new states. Consequently, the gain of a policy from a state becomes
\begin{align*}
	\rho_{SM}^\pi(s)=\lim_{n\rightarrow\infty}\frac{\displaystyle\mathbb{E}\left[ \sum_{t=0}^{n-1} r(s_t,\pi(s_t)) \,|\,s_0=s,\, \pi \right]}
	{\displaystyle\mathbb{E}\left[ \sum_{t=0}^{n-1} N_{t} \,|\,s_0=s,\, \pi \right]}\;\;.
\end{align*}

\subsubsection{Action Costs}
We propose an alternative interpretation of the SMDP framework, in which taking all actions can yield constant-length transitions, while consuming varying amounts of some resources (for example time, but also energy, money or any combination thereof). This results in the agent observing a real-valued action cost $k_{t+1}$, which is not necessarily related to the reward $r_{t+1}$, received after taking action $a_t$ from state $s_t$. As above, the assumption is that cost depends on the initial and final states and the action taken and has an expectation of the form $k(s,a)$. In general, all costs are supposed to be positive, but for the purposes of this paper this is relaxed to requiring that all policies have positive expected cost from all states. Likewise, without loss of generality it will be assumed that all action costs either are zero or have expected magnitude greater than or equal to one,
\begin{align*}
	|k(s,a)|&\geq 1\quad\forall k(s,a)\neq 0\;\;.
\end{align*}

In this model, a policy $\pi$ has expected cost
\begin{align*}
	c^\pi(s) &= \lim_{n\rightarrow\infty}\mathbb{E}\left[ \sum_{t=0}^{n-1} k(s_t,\pi(s_t))\,|\,s_0=s,\,\pi \right]\;\;,\\
	\intertext{with}
	c^\pi(s) &\geq 1\quad\forall s\in\mathcal{S},\,\pi\in\Pi\;\;.
\end{align*}

Observe that both definitions are analytically equivalent. That is, $N_t$ and $k(s_t,\pi(s_t))$ have the same role in the gain. Although their definition and interpretations varies---expected time to transition versus expected action cost---both give origin to identical problems with gain $\rho_{SM}^\pi(s)=\frac{v^\pi(s)}{c^\pi(s)}$.

Naturally, if all action costs or transition times are equal, the semi-Markov model reduces to average rewards, up to scale, and both problems are identical if the costs/times are unity-valued. For notational simplicity, from now on we will refer to the gain in both problems simply as $\rho$.

\subsubsection{Optimal policies}
A policy $\pi^*$ with gain
\begin{align*}
	\rho^{\pi^*}(s)=\rho^*(s)&=\lim_{n\rightarrow\infty}\frac{\displaystyle\mathbb{E}\left[ \sum_{t=0}^{n-1} r(s_t,\pi(s_t))\,|\,s_0=s,\,\pi^* \right]}
	{\displaystyle\mathbb{E}\left[ \sum_{t=0}^{n-1} k(s_t,\pi(s_t))\,|\,s_0=s,\,\pi^* \right]} =\frac{v^{\pi^*}(s)}{c^{\pi^*}(s)}\\
	\intertext{is gain-optimal if}
	\rho^*(s) &\geq \rho^\pi(s)\quad\forall s\in\mathcal{S},\,\pi\in\Pi\;\;,
\end{align*}
similarly to the way it was defined for the average-reward problem.

\begin{remark}
Observe that the gain-optimal policy does not necessarily maximize $v^\pi$, nor does it minimize $c^\pi$. It only optimizes their ratio.
\end{remark}

The following two sections discuss two technical assumptions that are commonly used in the average-reward and semi-Markov decision process literature to simplify analysis, guaranteeing that optimal stationary policies exist.

\subsection{The Unichain Assumption}
The transition probabilities of a fixed deterministic policy $\pi\in\Pi$ define a stochastic matrix, that is, the transition matrix of a homogeneous Markov chain on $\mathcal{S}$. In that embedded chain, a state is called \emph{transient} if, after a visit, there is a non-zero probability of never returning to it. A state is \emph{recurrent} if it is not transient. A recurrent state will be visited in finite time with probability one. A \emph{recurrent class} is a set of recurrent states such that no outside states can be reached by states inside the set. \citep{kemeny1960finite}

An MDP is called \emph{multichain} if at least one policy has more than one recurrent class, and $unichain$ if every policy has only one recurrent class. In an unichain problem, for all $\pi\in\Pi$, the state space can be partitioned as
\begin{align}
	\mathcal{S}=R^\pi\cup T^\pi\;\;,\label{eq:tranrec}
\end{align}
where $R^\pi$ is the single recurrent class and $T^\pi$ is a (possibly empty) transient set. Observe that these partitions can be unique to each policy; the assumption is to have a single recurrent class per policy, not a single one for the whole MDP.

If the MDP is multichain, a single optimality expression may not suffice to describe the gain of the optimal policy, stationary optimal policies may not exist, and theory and algorithms are more complex. On the other hand, if it is unichain, clearly for any given $\pi$ all states will have the same gain, $\rho^\pi(s)=\rho^\pi$, which simplifies the analysis and is a sufficient condition for the existence of stationary, gain-optimal policies. \citep{puterman1994markov}

Consequently, most of the literature on average reward MDPs and SMDPs relies on the assumption that the underlying model is unichain \citep[see for example][and references thereon]{mahadevan1996average,ghavamzadeh2007hierarchical}. Nevertheless, the problem of deciding whether a given MDP is unichain is not trivial. In fact, \citet{kallenberg2002classification} posed the problem of whether a polynomial algorithm exists to determine if an MDP is unichain, which was answered negatively by \citet{tsitsiklis2007np}, who proved that it is \textit{NP-}hard.

\subsection{Recurrent States}
The term $recurrent$ is also used, confusingly, to describe a state of the decision process that belongs to a recurrent class of every policy. The expression ``recurrent state'' will only be used in this sense from now on in this paper. Multi- and unichain processes may or may not have recurrent states. However, \citet{feinberg2008polynomial} proved that a recurrent state can be found or shown not to exist in polynomial time (on $|\mathcal{S}|$ and $|\mathcal{A}|$), and give methods for doing so. They also proved that, if a recurrent state exists, the unichain condition can be decided in polynomial time to hold, and proposed an algorithm for doing so.

Instead of actually using those methods, which would require a full knowledge of the transition probabilities that we do not assume, we emphasize the central role of the recurrent states, when they exist or can be induced, in simplifying analysis.

\begin{remark}
Provisionally, for the derivation below we will require all problems to be unichain and to have a recurrent state. However these two requirements will be further qualified in the experimental results.
\end{remark}

\section{Overview of Solution Methods and Related Work}
This section summarizes the most relevant solution methods for cumulative-reward MDPs and average-reward SMDPs, with special emphasis on stochastic shortest path algorithms. Our proposed solution to average-reward SMDPs will use what we call a Bertsekas split from these algorithms, to convert the problem into a minimal sequence of MDPs, each of which can be solved by any of the existing cumulative-reward methods. At the end of the section, a simple experimental task from \citet{sutton1998reinforcement} is presented to examine the performance of the discussed methods and to motivate our subsequent derivation.

\subsection{Cumulative Rewards}\label{ssec:cumrews}
Cumulative-reward MDPs have been widely studied. The survey of \citet{kaelbling1996reinforcement} and the books by \citet{bertsekas1996neuro}, \citet{sutton1998reinforcement}, and \citet{szepesvari2010algorithms} include comprehensive reviews of approaches and algorithms to solve MDPs (also called reinforcement learning problems). We will present a brief summary and propose a taxonomy of methods that suit our approach of accessing a ``black box'' reinforcement learning solver.

In general, instead of trying to find policies that maximize state value from Equation (\ref{eq:value}) directly, solution methods seek policies that optimize the \emph{state-action} pair (or simply `\emph{action}') value function,
\begin{align}
Q^\pi(s,\,a) &= \mathbb{E}\left[ \sum_{t=0}^\infty r(s_t,a_t)\; |\; s_0=s,\,a_0=a,\,\pi\right] \;\;, \label{eq:qvalue}
\end{align}
which is defined as the expected cumulative reinforcement after taking action $a$ in state $s$ and following policy $\pi$ thereafter.

The action value of an optimal policy $\pi^*$ corresponds to the solution of the following, equivalent, versions of the Bellman optimality equation \citep{sutton1998reinforcement},
\begin{align}
Q^*(s,\,a) &= \sum_{s'}{\mathcal{P}^a_{ss'}\left[\mathcal{R}^a_{ss'} + \max_{a'}Q^*(s',\,a')\right]}\;\;,\label{eq:qprob}\\
 &= \mathbb{E}\left[ r_{t+1} + \max_{a'}Q^*(s',\,a') | s_t=s,\,a_t=a\right]\;\;.\label{eq:qexpt}
\end{align}

\emph{Dynamic programming} methods assume complete knowledge of the transitions $\mathcal{P}$ and rewards $\mathcal{R}$ and seek to solve Equation (\ref{eq:qprob}) directly. An iteration of a type algorithm finds (\emph{policy evaluation}) or approximates (\emph{value iteration}) the value of the current policy and subsequently sets as current a policy that is greedy with respect to the values found (\emph{generalized policy iteration}). \citet{puterman1994markov} provides a very comprehensive summary of dynamic programming methods, including the use of linear programming to solve this kind of problems.

If the transition probabilities are unknown, in order to maximize action values it is necessary to sample actions, state transitions, and rewards in the environment. \emph{Model-based} methods use these observations to approximate $\mathcal{P}$, and then that approximation to find $Q^*$ and $\pi^*$ using dynamic programming. Methods in this family usually rely on complexity bounds guaranteeing performance after a number of samples \citep[or \emph{sample complexity},][]{kakade2003sample} bounded by a polynomial (that is, efficient) on the sizes of the state and action sets, as well as other parameters. \footnote{These other parameters often include the expression $\displaystyle\frac{1}{1-\gamma}$, where $\gamma$ is the discount factor, which is obviously problematic when, as we assume, $\gamma=1$. However, polynomial bounds also exist for undiscounted cases.}

The earliest and most studied model-based methods are PAC-MDP algorithms (efficiently probably approximately correct on Markov decision processes), which minimize with high probability the number of future steps on which the agent will not receive near-optimal reinforcements. $E^3$ \citep{kearns1998near,kearns2002near}, \emph{sparse sampling} \citep{kearns2002sparse}, $R_{MAX}$ \citep{brafman2003rmax}, MBIE \citep{strehl2005theoretical}, and $V_{MAX}$ \citep{rao2012v} are notable examples of this family of algorithms. Kakade's (\citeyear{kakade2003sample}) and Strehl's (\citeyear{strehl2007probably}) dissertations, and the paper by \citet{strehl2009reinforcement} provide extensive theoretical discussions on a broad range of PAC-MDP algorithms.

Another learning framework for which model-based methods exists is KWIK \citep[\emph{knows what it knows},][]{li2008knows,walsh2010integrating}. In it, at any decision epoch, the agent must return an approximation of the transition probability corresponding to the observed state, action and next state. This approximation must be arbitrarily precise with high probability. Alternatively, the agent can acknowledge its ignorance, produce a ``$\perp$'' output, and, from the observed, unknown transition, learn. The goal in this framework is to find a bound on the number of $\perp$ outputs, and for this bound to be polynomial on some appropriate parameters, including $|\mathcal{S}|$ and $|\mathcal{A}|$.

\emph{Model-free} methods, on the other hand, use transition and reward observations to approximate action values replacing expectations by samples in Equation (\ref{eq:qexpt}). The two main algorithms, from which most variations in the literature derive, are SARSA and Q-learning \citep{sutton1998reinforcement}. Both belong to the class of temporal difference methods. SARSA is an \emph{on-policy} algorithm that improves approximating the value of the current policy, using the update rule
\begin{align*}
Q_{t+1}(s_t,a_t) \leftarrow (1-\alpha_t)\, Q_t(s_t,a_t) + \alpha_t\,(r_{t+1} + Q_t(s_{t+1},a_{t+1}))\;\;,
\end{align*}
where $a_{t+1}$ is an action selected from $Q$ for state $s_{t+1}$. Q-learning is an \emph{off-policy} algorithm that, while following samples obtained acting from the current values of $Q$, approximates the value of the optimal policy, updating with the rule
\begin{align}
Q_{t+1}(s_t,a_t) \leftarrow (1-\alpha_t)\, Q_t(s_t,a_t) + \alpha_t\,(r_{t+1} + \max_a Q_t(s_{t+1},a))\;\;.
\end{align}
In both cases, $\alpha_t$ is a learning rate.

Q-learning has been widely studied and used for practical applications since its proposal by \citet{watkins1989learning}. In general, for an appropriately decaying learning rate such that
\begin{align}
\sum_{t=0}^\infty  \alpha_t &= \infty\label{eq:alphacond1}\\
\intertext{and}
\sum_{t=0}^\infty \alpha_t^2&<\infty\label{eq:alphacond2}\;\;,
\end{align}
and under the assumption that all states are visited and all actions taken infinitely often, it is proven to converge asymptotically to the optimal value with probability one \citep{watkins1992q}. Furthermore, in discounted settings, PAC convergence bounds exist for the case in which every state--action pair $(s,\,a)$ keeps an independent learning rate of the form $\frac{1}{ \{ 1 +|\text{visits to }(s,\,a)| \} }$ \citep{szepesvari1998asymptotic}, and for $Q$-updates in the case when a \emph{parallel sampler} $PS(\mathcal{M})$ \citep{kearns1999finite}, which on every call returns transition/reward observations for every state--action pair, is available \citep{even2004learning,azar2011speedy}.

An additional PAC-MDP, model-free version of Q-learning of interest is \emph{delayed Q-learning} \citep{strehl2006pac}. Although the main derivation of it is for discounted settings, as is usual for this kind of algorithms, building on the work of \citet{kakade2003sample} a variation is briefly discussed in which there is no discount but rather a \emph{hard horizon} assumption, in which only the next $H$ action-choices of the agent contribute to the value function. In this case, the bound is that, with probability $(1-\delta)$, the agent will follow an $\epsilon$-optimal policy, that is a policy derived from an approximation of $Q$ that does not differ from the optimal values more than $\epsilon$, on all but
\begin{align*}
O\left(\frac{|\mathcal{S}|\,|\mathcal{A}|H^5}{\epsilon^4}\,L(\cdot)\right)
\end{align*}
steps, where $L(\cdot)$ is a logarithmic function on the appropriate parameters ($|\mathcal{S}|\,\,|\mathcal{A}|,\,H,\,\frac{1}{\epsilon},\,\frac{1}{\delta}$). Usually, this latter term is dropped and the bound is instead expressed as
\begin{align*}
\tilde{O}\left(\frac{|\mathcal{S}|\,|\mathcal{A}|H^5}{\epsilon^4}\right)\;\;.
\end{align*}

Our method of solving SMDPs will assume access to a learning method for finite-horizon, undiscounted MDPs. The requirements of the solution it should provide are discussed in the analysis of the algorithm.

\subsection{Average Rewards}
As mentioned above, average-reward MDP problems are the subset of average-reward SMDPs for which all transitions take one time unit, or all actions have equal, unity cost. Thus, it would be sufficient to consider the larger set. However, average-reward problems have been the subject of more research, and the resulting algorithms are easily extended to the semi-Markov framework, by multiplying gain by cost in the relevant optimality equations, so both will be presented jointly here.

In this section we introduce a novel taxonomy of the differing parameters in the update rules of the main published solution methods for average-reward--including semi-Markov--tasks. This allows us to propose and discuss a generic algorithm that covers all existing solutions, and yields a compact summary of them, presented in Table \ref{tab:avgrw} below.

Policies are evaluated in this context using the \emph{average-adjusted sum of rewards} \citep{puterman1994markov,abounadi2002learning,ghavamzadeh2007hierarchical} value function:
\begin{align*}
	H^\pi(s) &= \lim_{n\rightarrow\infty} \mathbb{E} \left[ \sum_{t=0}^{n-1} (r(s_t,\pi(s_t)) - k(s_t,\pi(s_t))\,\rho^\pi)\,|\,s_0=s,\,\pi \right]\;\;,
\end{align*}
which measures ``how good'' the state $s$ is, under $\pi$, with respect to the average $\rho^\pi$. The corresponding Bellman equation, whose solutions include the gain-optimal policies, is
\begin{align}
	H^*(s) &= r(s,\pi^*(s)) - k(s,\pi^*(s))\,\rho^* + \mathbb{E}_{\pi^*}\left[ H^*(s') \right] \;\;,\label{eq:bellman}
\end{align}
where the expectation on the right hand side is over following the optimal policy for any $s'$.

\citet{puterman1994markov} and \citet{mahadevan1996average} present comprehensive discussions of dynamic programming methods to solve average-reward problems. The solution principle is similar to the one used in cumulative reward tasks: value evaluation followed by policy iteration. However, an approximation of the average rewards of the policy being evaluated must be either computed or approximated from successive iterates. The parametric variation of average-reward value iteration due to \citet{bertsekas1998new} is central to our method and will be discussed in depth below. For SMDPs, \citet{das1999solving} discuss specific synchronous and asynchronous versions of the \emph{relative value iteration} algorithm due to \citet{white1963dynamic}.

Among the model-based methods listed above for cumulative reward problems, $E^3$ and $R_{MAX}$ originally have definitions on average reward models, including in their PAC-MDP bounds polynomial terms on a parameter called the optimal \emph{$\epsilon$-mixing time}, defined as the smallest time after which the observed average reward of the optimal policy actually becomes $\epsilon$-close to $\rho^*$.

In a related framework, also with probability $(1-\delta)$ as in PAC-MDP, the UCRL2 algorithm of \citet{jaksch2010near} attempts to minimize the total \emph{regret} (difference with the accumulated rewards of a gain-optimal policy) over a $\mathcal{T}$-step horizon. The regret of this algorithm is bounded by $\tilde{O}(\Delta|\mathcal{S}|\sqrt{|\mathcal{A|T}})$, where the diameter parameter $\Delta$ of the MDP is defined as the time it takes to move from any state to any other state using an appropriate policy. Observe that for $\Delta$ to be finite, any state must be reachable from any other, so the problem must be \emph{communicating}, which is a more rigid assumption than the unichain condition \citep{puterman1994markov}. Similarly, the REGAL algorithm of \citet{bartlett2009regal} has a regret bound $\tilde{O}(H|\mathcal{S}|\sqrt{|\mathcal{A|T}})$, where $H$ is a bound on the span of the optimal bias vector. In this case, the underlying process is required to be \emph{weakly communicating}, that is, for the subsets $R^\pi$ and $T^\pi$ of recurrent and transient states in Equation (\ref{eq:tranrec}) to be the same for all $\pi\in\Pi$. This is also a more rigid assumption than the unichain condition.

Regarding PAC-MDP methods, no ``model free'' algorithms similar to delayed-Q are known at present for average reward problems. \citet{mahadevan1996average} discusses, without complexity analysis, a model-based approach due to \citet{jalali1989computationally}, and further refined by \citet{tadepalli1998model} into the \emph{H-learning} algorithm, in which relative value iteration is applied to transition probability matrices and gain is approximated from observed samples.

Model-free methods for average reward problems, with access to observations of state transitions and associated rewards, are based on the (gain-adjusted) Q-value update
\begin{align}
	Q_{t+1}(s_t,a_t) &\leftarrow \left(1-\alpha_t\right)\,Q_t(s_t,a_t) + \alpha_t\left( r_{t+1} - \rho_t\,k_{t+1} + \max_aQ_t(s_{t+1},a) \right)\;\;,\label{eq:adjq}
\end{align}
where $\alpha_t$ is a learning rate and $\rho_t$ is the current estimation of the average reward.

\algsetblock[Name]{Repeat}{Stop}{4}{0.5cm}
\begin{algorithm}[H]
\caption{Generic SMDP solver}
\label{alg:generic}
\begin{algorithmic}[0]
	\State Initialize ($\pi$, $\rho$, and $H$ or $Q$)
	\Repeat \textbf{ forever}
		\State Act
		\State Learn approximation to value of current $\pi$
		\State Update $\pi$ from learned values
		\State Update $\rho$
\end{algorithmic}
\end{algorithm}

A close analysis of the literature reveals that H-learning and related model-based algorithms, as well as methods based on the update in Equation (\ref{eq:adjq}) can be described using the generic Algorithm \ref{alg:generic}. The ``Act'' step corresponds to the observation of (usually) one $\langle s,\,a,\,s',\,r,\,k\rangle$ tuple following the current version of $\pi$. A degree of \emph{exploration} is commonly introduced at this stage; instead of taking a best-known action from $\argmax_{\hat{a}\in\mathcal{A}_s}Q(s,\,\hat{a})$, a suboptimal action is chosen. For instance, in the traditional \emph{$\varepsilon$-greedy} action selection method, with probability $\varepsilon$ a random action is chosen uniformly. \citet{sutton1998reinforcement} discuss a number of exploration strategies. The exploration/exploitation trade-off, that is, when and how to explore and learn and when to use the knowledge for reward maximization, is a very active research field in reinforcement learning. All of the PAC-MDP and related algorithms listed above are based on an ``optimism in the face of uncertainty'' scheme \citep{lai1985asymptotically}, initializing $H$ or $Q$ to an upper bound on value for all states, to address more or less explicitly the problem of optimal exploration. 

In H-learning, the learning stage of Algorithm \ref{alg:generic} includes updating the approximate transition probabilities for the state and action just observed and then estimating the state value using a version of Equation (\ref{eq:bellman}) with the updated probabilities. In model-free methods, summarized in Table \ref{tab:avgrw}, the learning stage is usually the 1-step update of Equation (\ref{eq:adjq}).

The $\rho$ update is also commonly done after one step, but there are a number of different update rules. Algorithms in the literature vary in two dimensions. The first is \emph{when} to update. Some compute an updated approximation of $\rho$ after every action while others do it only if a best-known action was taken, $a\in\argmax_{\hat{a}\in\mathcal{A}_s}Q(s,\,\hat{a})$. The second dimension is the way the updates are done. A natural approach is to compute $\rho_t$ as the ratio of the sample rewards and the sample costs,
\begin{align}
\rho_{\tau+1} &= \frac{\displaystyle\sum_{i=1}^{\tau+1} r_i}{\displaystyle\sum_{i=1}^{\tau+1} k_i}\;\;,\label{eq:ratioupdt}
\end{align}
where, $i=1\cdots\tau$ may indicate all decision epochs or only those on which greedy actions were taken, depending on when the algorithm updates. We refer to this as the \emph{ratio} update.

Alternatively, the \emph{corrected} update is of the form
\begin{align*}
\rho_{\tau+1} &= (1-\beta_\tau)\,\rho_{\tau} + \frac{\beta_\tau}{k_{\tau+1}}\,\left( r_{\tau+1} + \max_aQ_\tau(s_{\tau+1,\,a}) - \max_aQ_\tau(s_\tau,\,a) \right)\;\;,
\end{align*}
whereas in the \emph{term-wise corrected} update, separately
\begin{align*}
v_{\tau+1} &= (1-\beta_\tau)\,v_{\tau} + \beta_\tau\,r_{\tau+1}\;\;,\\
c_{\tau+1} &= (1-\beta_\tau)\,c_{\tau} + \beta_\tau\,k_{\tau+1}\;\;,\\
\intertext{and}
\rho_{\tau+1} &= \frac{v_{\tau+1}}{c_{\tau+1}}\;\;.
\end{align*}

In the last two cases, $\beta_\tau$ is a learning rate. In addition to when and how to perform the $\rho$ updates, algorithms in the literature also vary in the model used for the learning rates, $\alpha_t$ and $\beta_\tau$. The simplest models take both parameters to be \emph{constant}, equal to $\alpha$ and $\beta$ for all $t$ (or $\tau$). 
As is the case for Q-learning, convergence is proved for sequences of $\alpha_t$ (and now $\beta_\tau$) for which the conditions in Equations (\ref{eq:alphacond1}) and (\ref{eq:alphacond2}) hold. We call these \emph{decaying} learning rates. A simple decaying learning rate is of the form $\alpha_{t} = \frac{1}{1+t}$. It can be easily shown that this rate gives raise to the \emph{ratio} $\rho$ updates of Equation (\ref{eq:ratioupdt}). Some methods require keeping an \emph{individual} (decaying) learning rate for each state-action pair. A type of update for which Equations (\ref{eq:alphacond1}) and (\ref{eq:alphacond2})---and the associated convergence guarantees---hold, and which may have practical advantages is the ``search-then-converge'' procedure of \citet{darken1992learning}, called \emph{DCM} after its authors. A DCM $\alpha_t$ update would be, for example,
\begin{align*}
\alpha_t &= \frac{\alpha_0}{1+\frac{t^2}{\alpha_\tau+t}}\;\;,
\end{align*}
where $\alpha_0$ and $\alpha_\tau$ are constants.

\begin{table}[t]
\begin{center}
\begin{small}
\begin{tabular}{m{5.5cm}m{1.7cm}m{2cm}m{1.7cm}m{1.5cm}}
\toprule
Method & $\alpha_t$ & Only update\newline if greedy $a$? & $\rho$ & $\beta_\tau$\\
\midrule
R-learning\newline{\footnotesize \citet{schwartz1993reinforcement}} & Constant & Yes & Corrected & Constant\\
Algorithm 3\newline{\footnotesize \citet{singh1994reinforcement}} & Constant & No & Corrected & Constant\\
Algorithm 4\newline{\footnotesize \citet{singh1994reinforcement}} & Constant & Yes & Ratio & ---\\
SMART\newline{\footnotesize \citet{das1999solving}} & DCM & No & Ratio & ---\\
``New algorithm''\newline{\footnotesize \citet{gosavi2004reinforcement}} & Individual & No & Ratio & Decaying\\
Robbins-Monro Version\newline{\footnotesize \citet{gosavi2004reinforcement}} & Individual & No & Term-wise corrected & Decaying\\
\midrule
AAC\newline{\footnotesize \citet{jalali1989computationally}} & & & Ratio & ---\\
H-Learning\newline{\footnotesize \citet{tadepalli1998model}} & &Yes &Corrected &Decaying\\
\midrule
MAXQ\newline{\footnotesize \citet{ghavamzadeh2001continuous}} & Constant & Yes & Ratio & ---\\
HAR\newline{\footnotesize \citet{ghavamzadeh2007hierarchical}} & & & Ratio &---\\
\bottomrule
\end{tabular}
\end{small}
\caption{Summary of learning rates and $\rho$ updates of model-free, model-based and hierarchical average-reward algorithms.}
\label{tab:avgrw}
\end{center}
\end{table}

Table \ref{tab:avgrw} describes the $\rho$ updates and learning rates of the model-free average reward algorithms found in the literature, together, when applicable, with those for two model-based (AAC and H-learning) and two hierarchical algorithms (MAXQ and HAR).

\subsection{Stochastic Shortest Path H and Q-Learning}
We focus our interest on an additional model-free average-reward algorithm due to \citet{abounadi2002learning}, suggested by a dynamic programming method by \citet{bertsekas1998new}, which connects an average-reward problem with a parametrized family of (cumulative reward) stochastic shortest path problems.

The fundamental observation is that, if the problem is unichain, the average reward of a stationary policy must equal the ratio of the expected total reward and the expected total cost between two visits to a reference (recurrent) state. Thus, the idea is to ``separate'' those two visits to the start and the end of an episodic task. This is achieved splitting a recurrent state into an initial and a terminal state. Assuming the task is unichain and $s_I$ is a recurrent state, we refer to a \emph{Bertsekas split} as the resulting problem with
\begin{itemize}
	\item State space $\mathcal{S}\bigcup\{s_T\}$, where $s_T$ is an artificial terminal state that, as defined above, once reached transitions to itself with probability one, reward zero, and, for numerical stability, cost zero.
	\item Action space $\mathcal{A}$.
	\item Transition probabilities
			\begin{align*}
			\mathcal{P}_{ss'}^a &\leftarrow \left\{
			\begin{array}{ll}
			P(s'|s,a)&\forall s'\neq s_I,\,s_T, \\ 
			0 & s'=s_I, \\ 
			P(s_I|s,a) & s'=s_T.
			\end{array}
			\right.
			\end{align*}
\end{itemize}

In the restricted setting of average-reward MDPs, \citet{bertsekas1998new} proved the convergence of the dynamic programming algorithm with coupled iterations
\begin{align*}
H_{t+1}(s) &\leftarrow \max_{a\in\mathcal{A}_s}\left[ \sum_{s'}\mathcal{P}_{ss'}^a\left( \mathcal{R}_{ss'}^a + \sum_{s'}H_t(s') \right) - \rho_t \right] \quad\forall s\neq s_T\;\;,\\
\rho_{t+1} &\leftarrow \rho_t + \beta_t\,H_t(s_I)\;\;,
\end{align*}
where $H(s_T)$ is set to zero for all epochs (that is, $s_T$ is terminal) and $\beta_t$ is a decaying learning rate. The derivation of this algorithm includes a proof that, when $\rho_t$ equals the optimal gain $\rho^*$, the corrected value of the initial state is zero. This is to be expected, since $\rho_t$ is subtracted from all $\mathcal{R}$, and if it equals the expected average reward, the expectation for $s_I$ vanishes. Observe that, when this is the case, $\rho_t$ stops changing between iterations. We provide below an alternative derivation of this fact, from the perspective of fractional programming.

\citet{abounadi2002learning} extended the ideas behind this algorithm to model-free methods with the stochastic shortest path Q-learning algorithm (synchronous), SSPQ, with $Q$ and $\rho$ updates, after taking an action, 
\begin{align*}
Q_{t+1}(s_t,a_t) &\leftarrow (1-\alpha_t)\;Q_t(s_t,a_t) + \alpha_t\,\left( r_{t+1} -\rho_t + \max_aQ_t(s_{t+1},a) \right)\;\;,\\
\rho_{t+1} &\leftarrow \Gamma(\rho_t + \beta_t\,\max_aQ_t(s_I,a))\;\;,
\end{align*}
where $\Gamma$ is the projection to an interval $[-K,K]$ known to contain $\rho^*$.

\begin{remark}
Both SSP methods just described belong to generic family described by Algorithm \ref{alg:generic}. Moreover, the action value update of SSPQ is identical to the MDP version of average corrected Q-updates, Equation (\ref{eq:adjq}).
\end{remark}

The convergence proof of SSPQ makes the relationship between the two learning rates explicit, requiring that
\begin{align*}
\beta_t = o(\alpha_t)\;\;,
\end{align*}
making the gain update considerably slower than the value update. This is necessary so the Q-update can provide sufficient approximation of the value of the current policy for there to be any improvement. If, in the short term, the Q-update sees $\rho$ as (nearly) constant, then the update actually resembles that of a cumulative reward problem, with rewards $\mathcal{R} - \rho$.  The method presented below uses a Bertsekas split of the SMDP, and examines the extreme case in which the $\rho$ updates occur only when the value of the best policy for the current gain can be regarded as known.

\subsection{A Motivating Example}
\label{ssec:exmpl}
We will use a simple average-reward example from the discussion of Schwartz's R-learning by \citet[][see section 6.7]{sutton1998reinforcement} to study the behaviour of some of the algorithms just described and compare them with the method proposd in this paper. 

In the \emph{access control queuing task}, at the head of a single queue that manages access to $n=10$ servers, customers of priorities $\{8,4,2,1\}$ arrive with probabilities $\{0.4,0.2,0.2,0.2\}$, respectively. At each decision epoch, the customer at the head of the queue is either assigned to a free server (if any are available), with a pay-off equal to the customer's priority; or rejected, with zero pay-off. Between decision epochs, servers free independently with probability $p=0.06$. Naturally, the goal is to maximize the expected average reward. The states correspond to the combination of the priority of the customer at the head of the queue and the number of free servers, and the actions are simply ``accept'' and ``reject''. For simplicity, there is a single state corresponding to no free servers for any priority, with only the "reject" action available an reward zero.

To ensure that our assumptions hold for this task, we use the following straightforward observation:
\begin{proposition}
In the access control queuing task, all optimal policies must accept customers of priority 8 whenever servers are available.
\end{proposition}

Thus, if we make ``accept'' the only action available for states with priority 8 and any number of free servers, the resulting task will have the same optimal policies as the original problem, and the state with all servers occupied will become recurrent. Indeed, for any state with $m$ free servers, and any policy, there is a nonzero probability that all of the next $m$ customers will have priority 8 and no servers will free in the current and the next $m$ decision epochs. Since the only available action for customers with priority 8 is to accept them, all servers would fill, then, so for any state and policy there is a nonzero probability of reaching the state with no free servers, making it recurrent. Moreover, since this recurrent state can be reached from any state, there must be a single recurrent class per policy, containing the all-occupied state and all other states that can be reached from it under the policy, so the unichain condition also holds for the resulting task.

\begin{figure}[t]
\begin{center}
\includegraphics{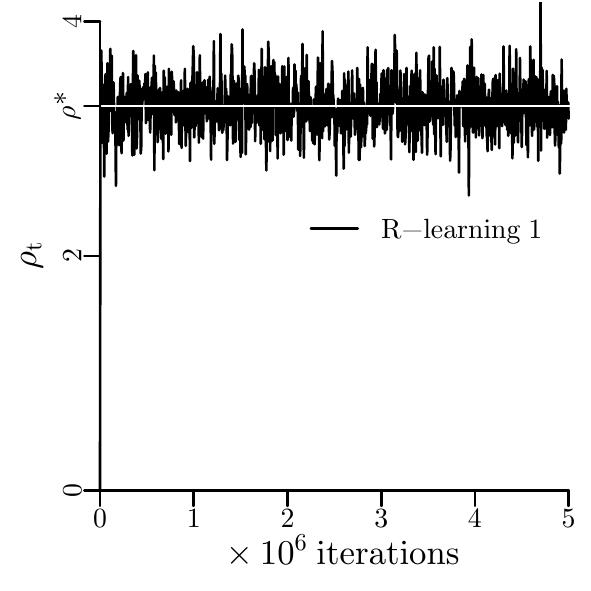}
\includegraphics{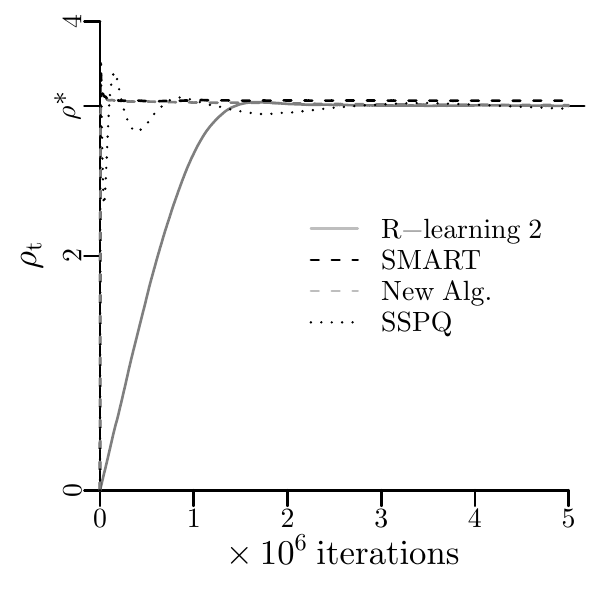}\\
\includegraphics{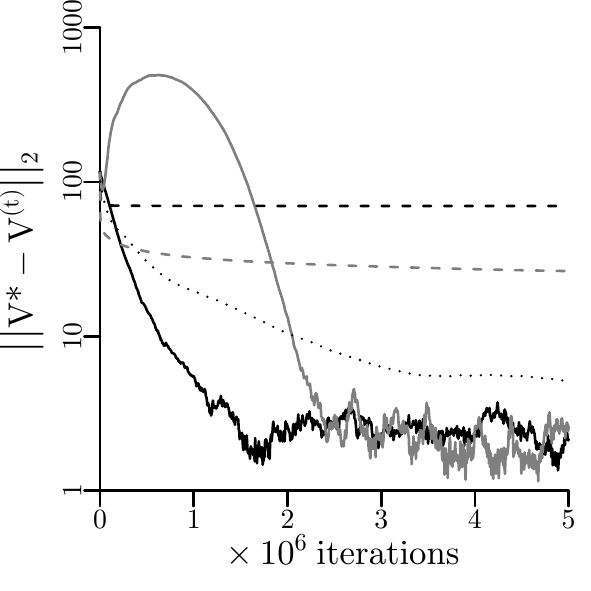}
\includegraphics{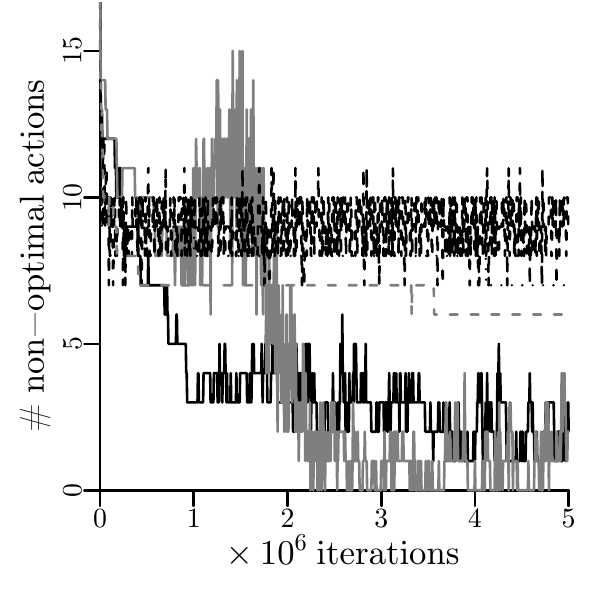}
\end{center}
\caption{Performance of average-reward reinforcement learning algorithms in the queuing task. Top left: gain in "R-Learning 1". Top right: gain in the other set-ups. Bottom left: convergence to the value of the optimal policy for all states. Bottom right: number of states for which the current policy differs from the optimal.}
\label{fig:algs}
\end{figure}

The unique optimal policy for this task, its average-adjusted value and average reward ($\rho^*\approx3.28$) can be easily found using dynamic programming, and will be used here to measure algorithm performance. \citet{sutton1998reinforcement} show the result of applying R-learning to this task using $\varepsilon$-greedy action selection with $\varepsilon=0.1$ (which we will keep for all experiments), and parameters\footnote{The on-line errata of the book corrects $\alpha$ to 0.01, instead of the stated 0.1.} $\alpha_t=\beta_t=0.01$. We call this set-up ``R-Learning 1.'' However, instead of a single run of samples, which results in states with many free servers being grossly undersampled, in order to study the convergence of action values over all states we reset the process to an uniformly-selected random state every 10 steps in all experiments.

The other algorithms tested are: ``R-learning 2,'' with a smaller rate for the $\rho$ updates ($\alpha_t=0.01$, $\beta_t=0.000001$), ``SMART'' \citep{das1999solving}, with parameters $\alpha_0=1$ and $\alpha_\tau=10^6$ for the DCM update of $\alpha_t$, the ``New Algorithm'' of \citet{gosavi2004reinforcement} with individual decaying learning rates equal to the inverse of the number of past occurrences of each state-action pair, and ``SSPQ''-learning with rates $\alpha_t=\frac{1}{t^{0.51}}$ and $\beta_t=\frac{1}{t}$. In all cases, $Q(s,a)$ and $\rho$ are initialized to 0.

Figure \ref{fig:algs} shows the results of one run (to avoid flattening the curves due to averaging) of the different methods for five million steps. All algorithms reach a neighbourhood of $\rho^*$ relatively quickly, but this doesn't guarantee an equally fast approximation of the optimal policy or its value function. The value of $\beta_t=0.01$ used in ``R-learning 1'' and by \citeauthor{sutton1998reinforcement} causes a fast approximation followed by oscillation (shown in a separate plot for clarity). The approximations for the smaller $\beta$ and the other approaches are more stable.

Overall, both R-learning set-ups, corresponding to solid black and grey lines in the plots in Figure \ref{fig:algs}, achieve a policy closer to optimal and a better approximation of its value, considerably faster than the other algorithms. On the other hand, almost immediately SMART, Gosavi's ``New Algorithm'', and SSPQ reach policies that differ from the optimal in about 5-10 states, but remain in that error range, whereas after a longer transient the R-learning variants find policies with less than 5 non-optimal actions and much better value approximations.

Remarkably, ``R-learning 2'' is the only set-up for which the value approximation and differences with the optimal policy increase at the start, before actually converging faster and closer to the optimal than the other methods. Only for that algorithm set-up is the unique optimal policy visited sometimes. This suggests that, for the slowest updating $\rho$, a different part of the policy space is being visited during the early stages of training.
Moreover, this appears to have a beneficial effect, for this particular task, on both the speed and quality of final convergence.

Optimal nudging, the method introduced below, goes even further in this direction, freezing the value of $\rho$ for whole runs of some reinforcement learning method, and only updating the gain when the value of the current best-known policy is closely approximated. 

\begin{figure}[h]
	\begin{center}
		\includegraphics{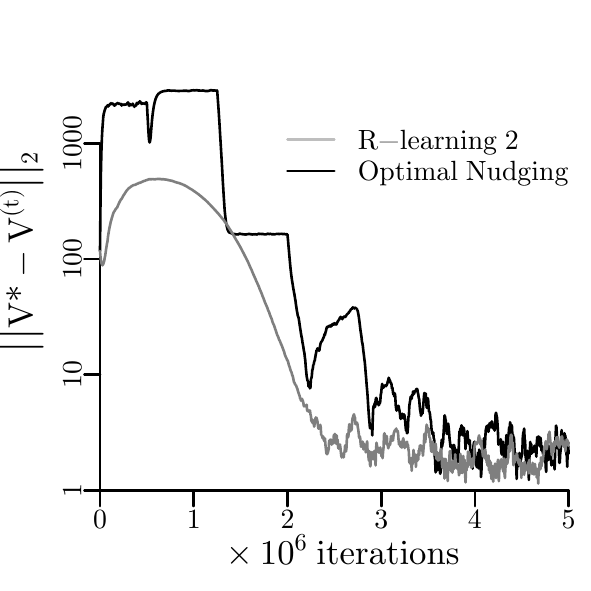}
	\end{center}
	\caption{Convergence to the value of the optimal policy for all states under optimal nudging, compared with the best R-learning experimental set-up.}
	\label{fig:optnudr}
\end{figure}

Figure \ref{fig:optnudr} shows the result of applying a vanilla version of optimal nudging to the access control queuing task, compared with the best results above, corresponding to ``R-learning 2''. After a 0-step--required to compute a parameter--which takes 500.000 samples, the optimal nudging steps proper are a series of blocks of 750.000 transition observations for a fixed $\rho$. The taking of each action is followed by a Q-learning update with the same $\alpha$ and $\varepsilon$ values as the R-learning experiments. The edges of plateaus in the black curve to the left of the plot signal the points at which the gain was updated to a new value following the optimal nudging update rule.

The figure shows that both algorithms have similar performance, reaching comparably good approximations to the value of the optimal policy in roughly the same number of iterations. Moreover, optimal nudging finds a similarly good approximation to the optimal policy, differing from it at the end of the run in only one state. However, optimal nudging has a number of further advantages: it has one parameter less to adjust, the $\beta$ learning rate, and consequently it performs one update less per iteration/action taking. This can have a dramatic effect in cases, such as this one, when obtaining sample transitions is quick and the updates cannot be done in parallel with the transitions between states. Additionally, there is no fine tuning in our implementation of the underlying Q-learning method; whereas in ``R-learning'' $\alpha$ was adjusted by Sutton and Barto to yield the best possible results, and we did the same when setting $\beta$ in ``R-learning 2'', the implementation of optimal nudging simply inherits the relevant parameters, without any adjustments or guarantees of best performance. Even setting the number of samples to 750.000 between changes of $\rho$ is a parameter that can be improved. The plateaus to the left of the plot suggest that in these early stages of learning good value approximations could be found much faster than that, so possibly an adaptive rule for the termination of each call to the reinforcement learning algorithm might lead to accelerated learning or free later iterations for finer approximation to optimal value.

\section{Optimal nudging.}
In this section we present our main algorithm, \emph{optimal nudging}. While belonging to the realm of generic Algorithm \ref{alg:generic}, the philosophy of optimal nudging is to disentangle the gain updates from value learning, turning an average-reward or semi-Markov task into a sequence of cumulative-reward tasks. This has the dual advantage of letting us treat as a black box the reinforcement learning method used (inheriting its speed, convergence and complexity features), while allowing for a very efficient update scheme for the gain.

Our method favours a more intuitive understanding of the $\rho$ term, not so much as an approximation to the optimal average reward (although it remains so), but rather as a punishment for \emph{taking actions}, which must be compensated by the rewards obtained afterwards. It exploits the Bertsekas split to focus on the value and cost of successive visits to a reference state, and their ratio. The $w-l$ space is introduced as an arena in which it is possible to update the gain and ensure convergence to the solution. Finally we show that updates can be performed optimally in a way that requires only a \emph{small} number of calls to the black-box reinforcement learning method.

\paragraph{Summary of Assumptions.}The algorithm derived below requires the average-reward semi-Markov decision process to be solved to have finite state and action sets, to contain at least one recurrent state $s_I$ and to be unichain. Further, it is assumed that the expected cost of every policy is positive and (without loss of generality) larger than one, and that the magnitude of all non-zero action costs is also larger than one.

To avoid a duplication of the cases considered in the mathematical derivation that would add no further insight, it will also be assumed that at least one policy has a positive average reward. If this is the case, naturally, any optimal policy will have positive gain. 

\subsection{Fractional Programming.}
Under our assumptions, a Bertsekas split is possible on the states, actions, and transition probabilities of the task. We will refer to the \emph{value} of a policy as the expected reward from the induced initial state, $s_I$, following the policy,
\begin{align}
\mathcal{v}^\pi = v^\pi(s_I) = \mathbb{E}\left[ \sum_{t=0}^\infty r(s_t,\pi(s_t)) \,|\, s_0=s_I,\,\pi \right]\;\;,
\end{align}
to its \emph{cost} as the expected cumulative action costs,
\begin{align*}
\mathcal{c}^\pi = c^\pi(s_I) = \mathbb{E}\left[ \sum_{t=0}^\infty k(s_t,\pi(s_t)) \,|\, s_0=s_I,\,\pi \right]\;\;,
\end{align*}
and, as customary, to its \emph{gain} as the value/cost ratio,
\begin{align*}
\rho^\pi = \frac{\mathcal{v}^\pi}{\mathcal{c}^\pi}\;\;.
\end{align*}

The following is a restatement of the known fractional programming result \citep{charnes1962programming}, linking in our notation the optimization of the gain ratio to a parametric family of linear problems on policy value and cost.

\begin{lemma}[Fractional programming]\label{lem:fracprog}
The following average-reward and linear-combination-of-rewards problems share the same optimal policies,
\begin{align*}
\argmax_{\pi\in\Pi}\frac{\mathcal{v}^\pi}{\mathcal{c}^\pi} = \argmax_{\pi\in\Pi}{\mathcal{v}^\pi + \rho^*\,(-\mathcal{c}^\pi)}\;\;,
\end{align*}
for an appropriate value of $\rho^*$ such that
\begin{align*}
\max_{\pi\in\Pi}\mathcal{v}^\pi + \rho^*\,(-\mathcal{c}^\pi)=0\;\;.
\end{align*}
\end{lemma}
\begin{proof}
The Lemma is proved by contradiction. Under the stated assumptions, for all policies, both $\mathcal{v^\pi}$ and $\mathcal{c^\pi}$ are finite, and $\mathcal{c^\pi\geq1}$.
Let a gain-optimal policy be
\begin{align*}
\pi^*\in\argmax_{\pi\in\Pi}\frac{\mathcal{v}^\pi}{\mathcal{c}^\pi}\;\;.
\end{align*}
If $\mathcal{v^*=v^{\pi^*}}$ and $\mathcal{c^*=c^{\pi^*}}$, let
\begin{align*}
\max_{\pi\in\Pi}\frac{\mathcal{v}^\pi}{\mathcal{c}^\pi}=\frac{\mathcal{v}^*}{\mathcal{c}^*}=\rho^*\;\;;
\end{align*}
then
\begin{align*}
\mathcal{v^*+\rho^*\,(-c^*)}=0\;\;.
\end{align*}
Now, assume there existed some policy $\hat{\pi}$ with corresponding $\mathcal{\hat{v}}$ and $\mathcal{\hat{c}}$ that had a better fitness in the linear problem,
\begin{align*}
\mathcal{\hat{v}} + \rho^*\,(-\mathcal{\hat{c}})>0\;\;.
\end{align*}
It must then follow (since all $\mathcal{c}^\pi$ are positive) that
\begin{align*}
\mathcal{\hat{v}}&>\rho^*\mathcal{\hat{c}}\;\;,\\
\frac{\mathcal{\hat{v}}}{\mathcal{\hat{c}}}&>\rho^*\;\;,
\end{align*}
which would contradict the optimality of $\pi^*$.
\end{proof}

This result has deep implications. Assume $\rho^*$ is known. Then, we would be interested in solving the problem
\begin{align*}
\pi^* &\in \argmax_{\pi\in\Pi}\mathcal{v}^\pi - \rho^*\,\mathcal{c}^\pi\;\;,\\
 &= \argmax_{\pi\in\Pi} \mathbb{E}\left[ \sum_{t=0}^\infty r(s_t,\pi(s_t)) \,|\, s_0=s_I \right] - \rho^*\,  \mathbb{E}\left[ \sum_{t=0}^\infty k(s_t,\pi(s_t)) \,|\, s_0=s_I \right]\;\;,\\
 &= \argmax_{\pi\in\Pi} \mathbb{E}\left[ \sum_{t=0}^\infty r(s_t,\pi(s_t)) - \rho^*\,k(s_t,\pi(s_t)) \,|\, s_0=s_I \right]\;\;,
\end{align*}
which is equivalent to a single cumulative reward problem with rewards ${r-\rho^*k}$, where, as discussed above, the rewards $r$ and costs $k$ are functions of $(s,a,s')$.

\subsection{Nudging.}
Naturally, $\rho^*$ corresponds to the optimal gain and is unknown beforehand. In order to compute it, we propose separating the problem in two parts: finding by reinforcement learning the optimal policy and its value for some fixed gain, and independently doing the gain-update. Thus, value-learning becomes method-free, so any of the robust methods listed in Section \ref{ssec:cumrews} can be used for this stage. The original problem can be then turned into a sequence of MDPs, for a series of temporarily fixed $\rho_i$. Hence, while remaining within the bounds of the generic Algorithm \ref{alg:generic}, we propose not to hurry to update $\rho$ after every step or Q-update. Additionally, as a consequence of Lemma \ref{lem:fracprog}, the method comes with the same solid termination condition of SSP algorithms: the current optimal policy $\pi^*_i$ is gain-optimal if $\mathcal{v}^{\pi^*_i}=0$.

This suggests the \emph{nudged} version of the learning algorithm, Algorithm \ref{alg:nudged}. The term nudged comes from the understanding of $\rho$ as a measure of the punishment given to the agent after each action in order to promote receiving the largest rewards as soon as possible. The remaining problem is to describe a suitable $\rho$-update rule.

\algsetblock[Name]{Repeat}{Stop}{3}{0.5cm}
\begin{algorithm}[H]
\caption{Nudged Learning}
\label{alg:nudged}
\begin{algorithmic}[0]
	\State Set Bertsekas split
	\State Initialize ($\pi$, $\rho$, and $H$ or $Q$)
	\Repeat
	\State Set reward scheme to $(r-\rho k)$
	\State Solve by any RL method
	\State Update $\rho$ 
	\State \textbf{until} $H^{\pi^*}(s_I)=Q^{\pi^*}(s_I,\pi^*(s_I))=0$
\end{algorithmic}
\end{algorithm}

\subsection{The $w-l$ Space.}
We will present an variation of the $w-l$ space, originally introduced by \citet{uribe2011discount} for the restricted setting of Markov decision process games, as the realm to describe $\rho$ uncertainty and to propose a method to reduce it optimally.

Under the assumptions stated above, the only additional requirement for the definition of the $w-l$ space is a bound on the value of all policies:
\begin{definition}[$D$]\label{def:D}
Let $D$ be a bound on unsigned, unnudged reward, such that
\begin{align*}
 D&\geq\max_{\pi\in\Pi} \mathcal{v}^\pi\;\;,\\
\intertext{and}
-D&\leq\min_{\pi\in\Pi} \mathcal{v}^\pi\;\;.
\end{align*}
\end{definition}

Observe that $D$ is a---possibly very loose---bound on $\rho^*$. However, it can become tight in the case of a task and gain-optimal policy in which termination from $s_I$ occurs in one step with reward of magnitude $D$. Importantly, under our assumptions and Definition \ref{def:D}, all policies $\pi\in\Pi$ will have finite real expected value $-D \leq\mathcal{v}^\pi\leq D$ and finite positive cost $\mathcal{c}^\pi\geq 1$.

\subsubsection{The $w-l$ Mapping.}
We are now ready to propose the simple mapping of a policy $\pi\in\Pi$, with value $\mathcal{v}^\pi$ and cost $\mathcal{c}^\pi$, to the 2-dimensional $w-l$ space using the transformation equations:
\begin{align}
w^\pi=\frac{D+\mathcal{v}^\pi}{2\mathcal{c}^\pi}\;\;,\quad\quad l^\pi=\frac{D-\mathcal{v}^\pi}{2\mathcal{c}^\pi}\;\;.\label{eq:wldef}
\end{align}
The following properties of this transformation can be easily checked from our assumptions:

\begin{proposition}[Properties of the $w-l$ space.] \label{prop:wlprops}
For all policies $\pi\in\Pi$,
\begin{enumerate}
	\item $0\leq w^\pi\leq D$; $0\leq l^\pi\leq D$.
	\item $w^\pi+l^\pi=\displaystyle\frac{D}{\mathcal{c}^\pi}\leq D$.
	\item If $\mathcal{v}^\pi=D$, then $l^\pi=0$.
	\item If $\mathcal{v}^\pi=-D$, then $w^\pi=0$.
	\item $\displaystyle\lim_{\mathcal{c^\pi}\rightarrow \infty}(w^\pi,l^\pi) = (0,0).$
\end{enumerate}
\end{proposition}

As a direct consequence of Proposition \ref{prop:wlprops}, the whole policy space, which has $O(\tilde{|\mathcal{A}_s|}^{|\mathcal{S}|})$ elements, with $\tilde{|\mathcal{A}_s|}$ equal to the average number of actions per state, is mapped into a \emph{cloud} of points in the $w-l$ space, bounded by a triangle with vertices at the origin and the points $(D,0)$ and $(0,D)$. This is illustrated in Figure \ref{fig:picloud}.

\begin{figure}[h]
\begin{center}
\includegraphics{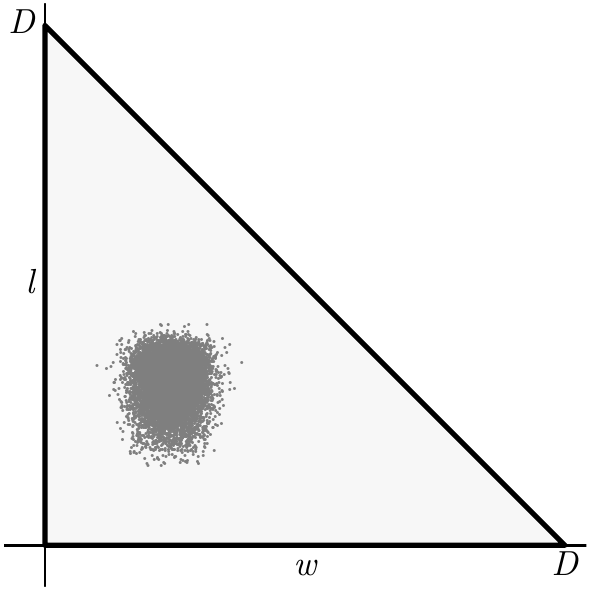}
\end{center}
\caption{Mapping of a policy cloud to the $w-l$ space.}
\label{fig:picloud}
\end{figure}

\subsubsection{Value and Cost in $w-l$.}
Cross-multiplying the equations in (\ref{eq:wldef}), it is easy to find an expression for value in $w-l$,
\begin{align}
\mathcal{v}^\pi=D\,\frac{w^\pi-l^\pi}{w^\pi+l^\pi}\;\;.\label{eq:wlvalue}
\end{align}

All of the policies of the same value (for instance $\mathcal{\hat{v}}$) lie on a level set that is a line with slope $\frac{D-\mathcal{\hat{v}}}{D+\mathcal{\hat{v}}}$ and intercept at the origin. Thus, as stated in Proposition \ref{prop:wlprops}, policies of value $\pm D$ lie on the $w$ and $l$ axes and, further, policies of expected value 0 lie on the $w=l$ line. Furthermore, geometrically, the value-optimal policies must subtend the smallest angle with the $w$ axis and vertex at the origin. Figure \ref{fig:vclss} (left) shows the level sets of Equation (\ref{eq:wlvalue}). 

\begin{figure}[h]
\begin{center}
\includegraphics{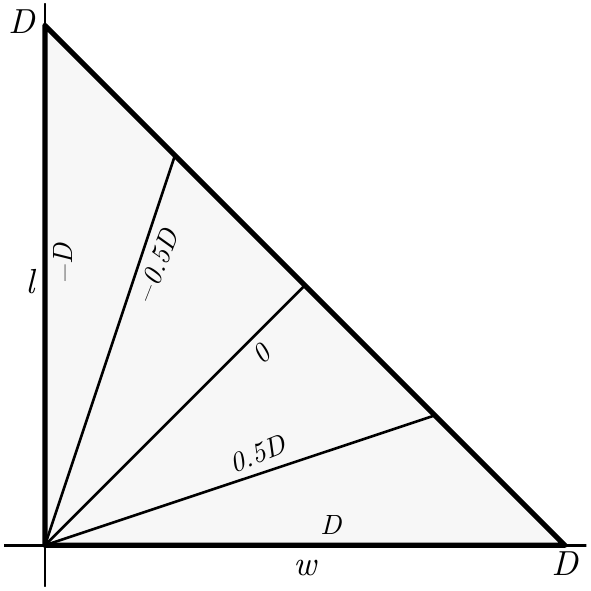}
\includegraphics{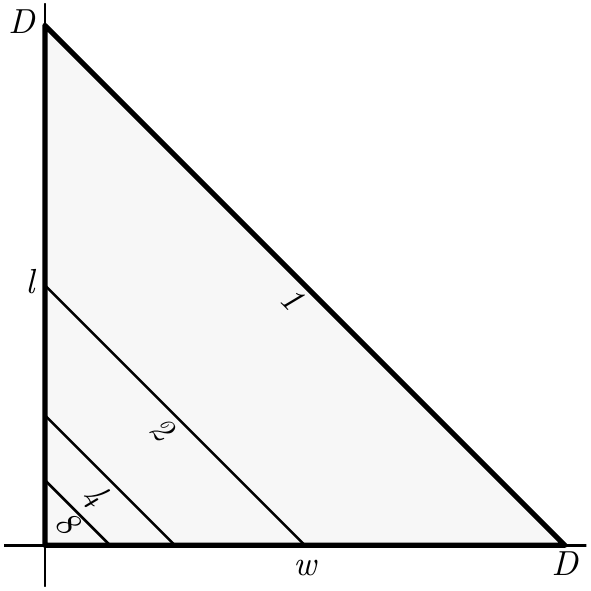}
\end{center}
\caption{Value (left) and cost (right) level sets in the $w-l$ space.}
\label{fig:vclss}
\end{figure}

On the other hand, adding both Equations in (\ref{eq:wldef}), cost in the $w-l$ space is immediately found as
\begin{align}
\mathcal{c}^\pi = \frac{D}{w^\pi+l^\pi}\;\;. \label{eq:wlcost}
\end{align}

This function also has line level sets, in this case all of slope $-1$. The $w+l=D$ edge of the triangle corresponds to policies of expected cost one and, as stated in the properties Proposition \ref{prop:wlprops}, policies in the limit of infinite cost should lie on the origin. Figure \ref{fig:vclss} (right) shows the cost level sets in the $w-l$ space. An interesting question is whether the origin actually belongs to the $w-l$ space. Since it would correspond to policies of infinite expected cost, that would contradict either the unichain assumption or the assumption that $s_I$ is a recurrent state, so the origin is excluded from the $w-l$ space. 

The following result derives from the fact that, even though the value and cost expressions in Equations (\ref{eq:wlvalue}) and (\ref{eq:wlcost}) are not convex functions, both value and cost level sets are lines, each of which divides the triangle in two polygons:

\begin{lemma}[Value and cost-optimal policies in $w-l$] \label{lem:cnvxhll}
	The policies of maximum or minimum value and cost map in the $w-l$ space to vertices of the convex hull of the point cloud.
\end{lemma}
\begin{proof}
	All cases are proved using the same argument by contradiction. Consider for instance a policy of maximum value $\pi^*$, with value $\mathcal{v^*}$. The $\mathcal{v^*}$ level set line splits the $w-l$ space triangle in two, with all points corresponding to higher value \emph{below} and to lower value \emph{above} the level set. If the mapping $(w^{\pi^*},l^{\pi^*})$ is an interior point of the convex hull of the policy cloud, then some points on an edge of the convex hull, and consequently at least one of its vertices, must lie on the region of value higher than $\mathcal{v}^*$. Since all vertices of this cloud of points correspond to actual policies of the task, there is at least one policy with higher value than $\pi^*$, which contradicts its optimality.
	The same argument extends to the cases of minimum value, and maximum and minimum cost.
\end{proof}

\subsubsection{Nudged Value in $w-l$.}
Recall, from the fractional programming Lemma \ref{lem:fracprog}, that for an appropriate $\rho^*$ these two problems are optimized by the same policies:
\begin{align}
\argmax_{\pi\in\Pi}\frac{\mathcal{v}^\pi}{\mathcal{c}^\pi} = \argmax_{\pi\in\Pi}{\mathcal{v}^\pi + \rho^*\,(-\mathcal{c}^\pi)}\;\;,\label{eq:problemma}
\end{align}

By substituting on the left hand side problem in Equation (\ref{eq:problemma}) the expressions for $w^\pi$ and $l^\pi$, the original average-reward semi-Markov problem becomes in the $w-l$ space the simple linear problem
\begin{align}
\argmax_{\pi\in\Pi}\frac{\mathcal{v}^\pi}{\mathcal{c}^\pi} = \argmax_{\pi\in\Pi}w^\pi-l^\pi\;\;.\label{eq:w-l}
\end{align}
Figure \ref{fig:wminusl} illustrates the slope-one level sets of this problem in the space. Observe that, predictably, the upper bound of these level sets in the triangle corresponds to the vertex at $(D,0)$, which, as discussed above, in fact  would correspond to a policy with value $D$ and unity cost, that is, a policy that would receive the highest possible reinforcement from the recurrent state and would return to it in one step with probability one.

\begin{figure}[h]
\begin{center}
\includegraphics{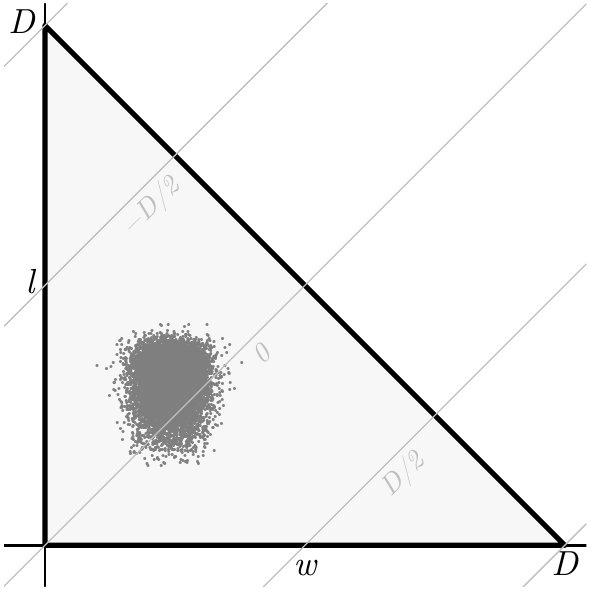}
\end{center}
\caption{Level sets of the linear correspondence of an average-reward SMDP in the $w-l$ space.}
\label{fig:wminusl}
\end{figure}

Conversely, for some $\rho_i$ (not necessarily the optimal $\rho^*$), the problem on the right hand side of Equation ($\ref{eq:problemma}$) becomes, in the $w-l$ space,
\begin{align*}
\argmax_{\pi\in\Pi}\mathcal{v}^\pi - \rho_i\,\mathcal{c}^\pi = \argmax_{\pi\in\Pi}D\,\frac{w^\pi - l^\pi - \rho_i}{w^\pi+l^\pi}\;\;,
\end{align*}
where we opt not to drop the constant $D$. We refer to the nudged value of a policy, $h^\pi_{\rho_i}$ as,
\begin{align*}
h_{\rho_i}^\pi=D\,\frac{w^\pi - l^\pi - \rho_i}{w^\pi+l^\pi}\;\;.
\end{align*}
All policies $\hat{\pi}$ sharing the same nudged value $\hat{h}_{\rho_i}$ lie on the level set
\begin{align}
l = \frac{D-\hat{h}_{\rho_i}}{D+\hat{h}_{\rho_i}}\,w - \frac{D}{D+\hat{h}_{\rho_i}}\,\rho_i\;\;, \label{eq:nudgedlevelset}
\end{align}
which, again, is a line on the $w-l$ space whose slope, further, depends only on the common nudged value, and not on $\rho_i$. Thus, for instance, for any $\rho_i$, the level set corresponding to policies with zero nudged value, such as those of interest for the termination condition of the fractional programming Lemma \ref{lem:fracprog} and Algorithm \ref{alg:nudged}, will have unity slope.

There is a further remarkable property of the line level sets corresponding to policies of the same value, summarized in the following result:

\begin{lemma}[Intersection of nudged value level sets] \label{lem:pencil}
	For a given nudging $\rho_i$, the level sets of all possible $\hat{h}_{\rho_i}$ share a common intersection point, $\left(\frac{\rho_i}{2},\, -\frac{\rho_i}{2}\right)$.
\end{lemma}
\newcommand{\hhat}{\hat{h}_{\rho_i}}
\newcommand{\htil}{\tilde{h}_{\rho_i}}
\begin{proof}
	This result is proved through straightforward algebra. Consider, for a fixed $\rho_i$, the line level sets corresponding to two nudged values, $\hhat$ and $\htil$. Making the right hand sides of their expressions in the form of Equation (\ref{eq:nudgedlevelset}) equal, and solving to find the $w$ component of their intersection yields:
	\begin{align*}
	\frac{D-\hhat}{D+\hhat}w - \frac{D}{D+\hhat}\rho_i &= \frac{D-\htil}{D+\htil}w - \frac{D}{D+\htil}\rho_i\;\;,\\
	2D(\htil-\hhat)w &= D(\htil-\hhat)\rho_i\;\;,\\
	w &= \frac{\rho_i}{2}\;\;.
	\end{align*}
	Finally, replacing this in the level set for $\hhat$,
	\begin{align*}
	l&=\frac{D-\hhat}{D+\hhat}\frac{\rho_i}{2}-\frac{D}{D+\hhat}\rho_i\;\;,\\
	l&=\frac{-D-\hhat}{D+\hhat}\frac{\rho_i}{2}=-\frac{\rho_i}{2}\;\;.
	\end{align*}
\end{proof}

Thus, for a set $\rho_i$, all nudged level sets comprise what is called a \emph{pencil of lines}, a parametric set of lines with a unique common intersection point. Figure \ref{fig:pencil} (right) shows an instance of such a pencil.

\begin{figure}[h]
\begin{center}
\includegraphics{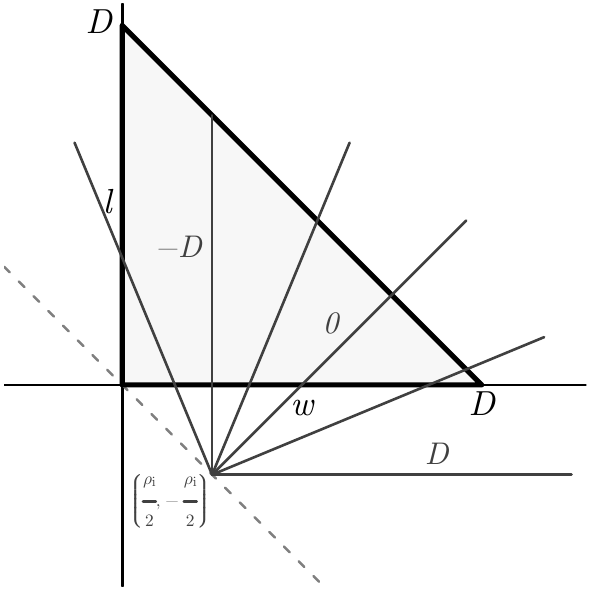}
\includegraphics{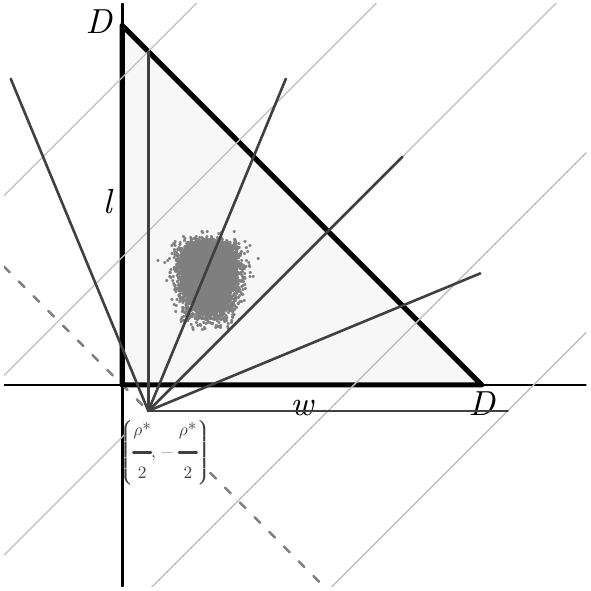}
\end{center}
\caption{Nudged value in the $w-l$ space. Left, value level sets as a pencil of lines. Right, solution to both problems with $\rho^*$ in Lemma \ref{lem:fracprog}.}
\label{fig:pencil}
\end{figure}

The termination condition states that the nudged and average-reward problems share the same solution space for $\rho^*$ such that the nudged value of the optimal policy is 0. Figure \ref{fig:pencil} (left) illustrates this case: the same policy in the example cloud simultaneously solves the linear problem with level sets of slope one (corresponding to the original average-reward task in Equation \ref{eq:w-l}), and the nudged problem with zero-nudged value.

\subsection{Minimizing $\rho$-uncertainty.}
We have now all the elements to start deriving an algorithm for iteratively enclosing $\rho^*$ quickly in the smallest neighborhood possible. Observe that from the outset, since we are assuming that policies with non-negative gain exist, the bounds on the optimal gain are the largest possible,
\begin{align}
	0\leq\rho^*\leq D\;\;. \label{eq:initrhounc}
\end{align}
Geometrically, this is equivalent to the fact that the vertex of the pencil of lines for $\rho^*$ can be, a priori, anywhere on the segment of the $w=-l$ line between $(0,\,0)$, if the optimal policy has zero gain, and $\left(\frac{D}{2},\,-\frac{D}{2}\right)$, if the optimal policy receives reinforcement $D$ and terminates in a unity-cost step (having gain $D$).
For our main result, we will propose a way to reduce the best known bounds on $\rho^*$ as much as possible every time a new $\rho_i$ is computed.

\subsubsection{Enclosing Triangles, Left and Right $\rho$-uncertainty.}
\label{sec:enctrslrunc}

In this section we introduce the notion of an enclosing triangle. The method introduced below to reduce gain uncertainty exploits the properties of reducing enclosing triangles to optimally reduce uncertainty between iterations.

\begin{definition}[Enclosing triangle]\label{def:encltr}
A triangle in the $w-l$ space with vertices $A~=(w_A,\,l_A)$, $B=(w_B,\,l_B)$, and $C=(w_C,\,l_C)$ is an \emph{enclosing triangle} if it is known to contain the mapping of the gain-optimal policy and, additionally,
\begin{enumerate}
	\item $w_B\geq w_A$; $l_B\geq l_A$.
	\item $\frac{l_B-l_A}{w_B-w_A}=1$.
	\item $w_A\geq l_A$; $w_B\geq l_B$; $w_C\geq l_C$.
	\item $P=\frac{w_B-l_B}{2}=\frac{w_A-l_A}{2}\leq\frac{w_C-l_C}{2}=Q$.
	\item $0\leq\frac{l_C-l_A}{w_C-w_A}\leq1$
	\item $\left| \frac{l_C-l_B}{w_C-w_B} \right| \geq 1$
\end{enumerate}
\end{definition}

\begin{figure}[!h]
	\begin{center}
		\includegraphics{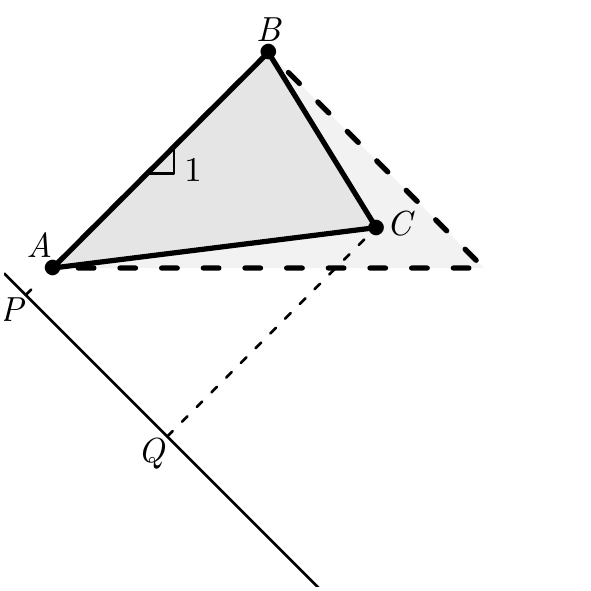}
	\end{center}
	\caption{Illustration of the definition of an enclosing triangle. The segment joining vertices $A$ and $B$ belongs to a line with slope one. Vertex $C$ can lie anywhere on the light grey triangle, including its edges. $P$ and $Q$ are the ($w$ components of the) slope-one projection of the vertices to the $w=-l$ line}
	\label{fig:defenctr}
\end{figure}

Figure \ref{fig:defenctr} illustrates the geometry of enclosing triangles as defined.
The first two conditions in Definition \ref{def:encltr} ensure that the point $B$ is \emph{above} $A$ and the slope of the line that joins them is unity. The third condition places all three points in the part of the $w-l$ space on or below the $w=l$ line, corresponding to policies with non-negative gain.

In the fourth condition, $P$ is defined as the $w$-component of the intersection of the line with slope one that joins points $A$ and $B$, and the $w=-l$ line; and $Q$ as the $w$ value of the intersection of the line with slope one that crosses point $C$, and $w=-l$. Requiring $P\leq Q$ is equivalent to forcing $C$ to be \emph{below} the line that joins $A$ and $B$.

The fifth and sixth conditions confine the possible location of $C$ to the triangle with vertices $A$, $B$, and the intersection of the lines that cross $A$ with slope zero and $B$ with slope minus one. This triangle is pictured with thick dashed lines in Figure \ref{fig:defenctr}.

\begin{remark}[Degenerate enclosing triangles]\label{rm:degenc}
	Observe that, in the definition of an enclosing triangle, some degenerate or indeterminate cases are possible. First, if $A$ and $B$ are concurrent, then $C$ must be concurrent to them---so the ``triangle'' is in fact to a point---and some terms in the slope conditions (2, 5, 6) in Definition \ref{def:encltr} become indeterminate. Alternately, if $P=Q$ in condition 4, then $A$, $B$, and $C$ must be collinear.
	We admit both of these degenerate cases as \emph{valid} enclosing triangles since, as is discussed below, they correspond to instances in which the solution to the underlying average-reward task has been found.
\end{remark}

Since we assume that positive-gain policies, and thus positive-gain \emph{optimal} policies exist, direct application of the definition of an enclosing triangle leads to the following Proposition:

\begin{proposition}[Initial enclosing triangle]
$A=(0,\,0)$, $B=\left(\frac{D}{2},\,\frac{D}{2}\right)$, $C=(D,\,0)$ is an enclosing triangle.
\end{proposition}

In order to understand the reduction of uncertainty after solving the reinforcement learning task for a fixed gain, consider the geometry of setting $\rho_1$ to some value within the uncertainty range of that initial enclosing triangle, for example $\rho_1=\frac{D}{4}$. 

If, after solving the reinforcement learning problem with rewards $r(s,a,s')-\rho_1k(s,a,s')$ the value of the initial state $s_I$ for the resulting optimal policy were $\mathcal{v}^*_1=0$, then the semi-Markov task would be solved, since the termination condition would have been met. Observe that this would not imply, nor require, any knowledge of the value and cost, or conversely, the $w$ and $l$ coordinates of that optimal policy. However, the complete solution to the task would be known: the optimal gain would be $\rho_1$, and the optimal policy and its gain-adjusted value would be, respectively, the policy and (cumulative) value just found. In the $w-l$ space, the only knowledge required to conclude this is that the coordinates of the optimal policy to both problems would lie somewhere inside the $w-l$ space on the line with slope $\frac{D-0}{D+0}=1$ that crosses the point $(\frac{\rho_1}{2},\,\frac{\rho_1}{2})$.

On the other hand, if the optimal policy after setting $\rho_1=\frac{D}{4}$ and solving were some positive value $\mathcal{v^*_1>0}$ (for example $\mathcal{v^*_1}=\frac{D}{3}$), the situation would be as shown in Figure \ref{fig:rho1}. The termination condition would not have been met, but the values of $\rho_1$ and $\mathcal{v}$ would still provide a wealth of exploitable information.

\begin{figure}[!h]
\begin{center}
\includegraphics{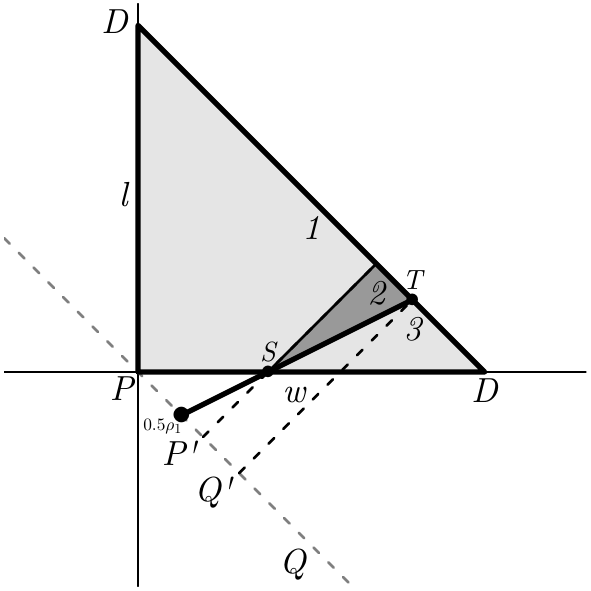}
\end{center}
\caption{Geometry of the solution for a fixed $\rho_1$ in the initial enclosing triangle. The nudged optimal policy maps somewhere on the $\overline{ST}$ segment, so the gain-optimal policy must map somewhere on the area shaded ``2'', and all points in regions ``1'' and ``3'' can be discarded.}
\label{fig:rho1}
\end{figure}

Setting $\rho_1=\frac{D}{4}$ causes the pencil of value level set lines to have common point $\left(\frac{D}{8},\,\frac{D}{8}\right)$. The optimal value $\mathcal{v^*_1}=\frac{D}{3}$ corresponds to a level set of slope 0.5, pictured from that point in thick black. The nudged-value optimal policy in $w-l$, then, must lie somewhere on the segment that joins the crossings of this line with the lines $l=0$ (point $S$) and $w+l=D$ (point $T$). This effectively divides the space in three regions with different properties, shaded and labelled from ``1'' to ``3'' in Figure \ref{fig:rho1}.

First, no policies of the task can map to points in the triangle ``3'', since they would have higher $\rho_1$-nudged value, contradicting the optimality of the policy found. Second, policies with coordinates in the region labelled ``1'' would have lower gain than all policies in the $\overline{ST}$ segment, which is known to contain at least one policy, the nudged optimizer. Thus, the gain-optimal policy must map in he $w-l$ space to some point in region ``2'', although not necessarily on the $\overline{ST}$ segment. Moreover, clearly ``2'' is itself a new enclosing triangle, with vertices $A'=S$, $C'=T$, and $B'$ on the intersection of the $w+l=D$ line and the slope-one line that crosses the point $S$.

As a direct consequence of being able to discard regions ``1'' and ``3'', the uncertainty range for (one-half of) the optimal gain reduces from $Q-P$, that is, the initial $\frac{D}{2}$ that halves the range in Equation (\ref{eq:initrhounc}), to the difference of the 1-projections of the vertices of triangle ``2'' to the $w=-l$ line, $Q'-P'$. For the values considered in the example, total uncertainty reduces approximately fivefold, from $D$ to $\frac{5D}{24}$.

Thus, running the reinforcement learning algorithm to solve the nudged problem, with $\rho_1$ within the bounds of the initial enclosing triangle allows, first, the determination of a new, smaller enclosing triangle and second, a corresponding reduction on the gain uncertainty.

\begin{figure}[!h]
	\begin{center}
		\includegraphics{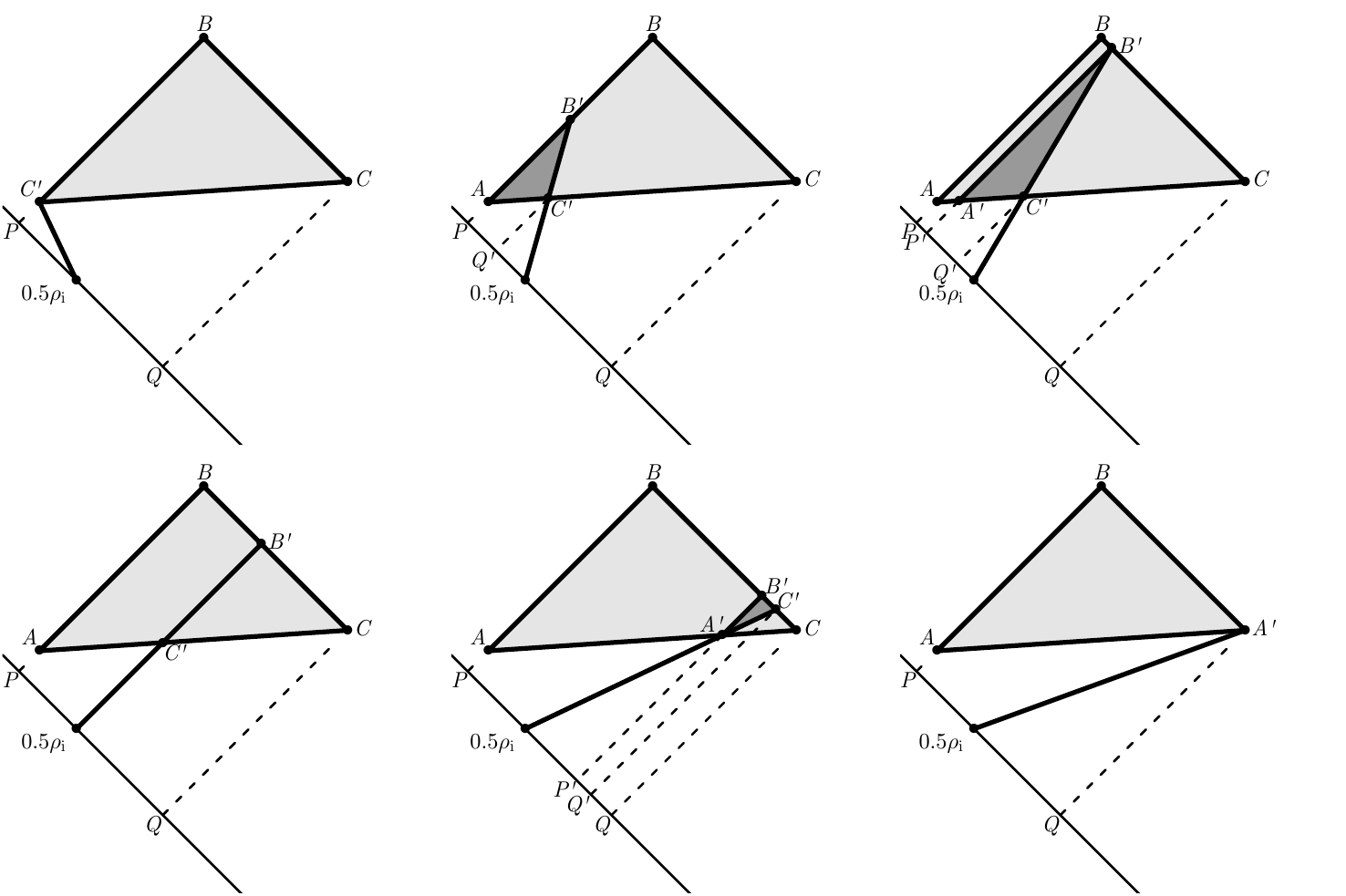}
	\end{center}
	\caption{Possible outcomes after solving a nudged task within the bounds of an arbitrary enclosing triangle. In all cases either the uncertainty range vanishes to a point (top left and bottom right) or a line segment (bottom left), and the problem is solved; or a smaller enclosing triangle results (top middle and right, bottom middle).}
	\label{fig:encltr}
\end{figure}

Both of these observations are valid for the general case, as illustrated in Figure \ref{fig:encltr}. Consider an arbitrary enclosing triangle with vertices $A$, $B$, and $C$. Some $\rho_i$ is set within the limits determined by the triangle and the resulting nudged problem is solved, yielding a nudged-optimal policy $\pi^*_i$, with value $\mathcal{v}^*_i$. 

In the $w-l$ space, in all cases, the line from the point $(\frac{\rho_i}{2}, -\frac{\rho_i}{2})$, labelled simply $0.5\rho_i$ in the plots, has slope $\frac{D-\mathcal{v}^*_i}{D+\mathcal{v}^*_i}$ and intercepts the $\overline{AC}$ and $\overline{AB}$ or $\overline{BC}$ segments. The three possible degenerate cases are reduction to a point, if the level set from $\rho_i$ crosses the points $A$ or $C$ (top left and bottom right plots in Figure \ref{fig:encltr}, respectively), and the situation with $\mathcal{v}^*_i=0$, where the problem is solved and, geometrically, the uncertainty area reduces from an enclosing triangle to a line segment. In these three instances, the method stops because the new gain uncertainty reduces to $P'-Q'=0$.

If $\mathcal{v}^*_i$ is negative, the level set can intercept $\overline{AB}$ (top  middle plot in Figure \ref{fig:encltr}) or $\overline{BC}$ (top left plot). In the first case, the new enclosing triangle, shaded dark grey, has the same $A'=A$ vertex, $B'$ is the intercept with $\overline{AB}$, and $C'$ is the intercept with $\overline{AC}$. Thus, $P'=P$ and $Q'<Q$, for a strict reduction of the uncertainty. Points in the light grey region to the right of this new triangle cannot correspond to any policies, because they would have solved the $\rho_i$ nudged problem instead, so the gain optimal policy must be inside the triangle with vertices $A'$, $B'$, and $C'$. In the second case, the level set intercept with $\overline{AC}$ again becomes $C'$, its intercept with $\overline{BC}$ is the new $B'$, and the slope-1 projection of $B'$ to $\overline{AC}$ is the new $A'$. Again, points in the light grey triangle with vertices $B'$, $C'$, $C$ cannot contain mapped policies, or that would contradict the nudged optimality of $\pi^*_i$ and, furthermore, points in the trapeze to the left of the new enclosing triangle cannot contain the gain-optimal policy, since any policies on $\overline{C'B'}$, which includes $\pi^*_i$, have larger gain. The new $Q'$ is the 1-projection of $C'$ to $w=-l$ and it is smaller than, not only the old $Q$, but also $\rho_i$. The new $P'$ is the 1-projection of $B'$ (or $A'$) to $w=-l$ and it is larger than $P$, resulting in a strict reduction of the gain uncertainty.

If $\mathcal{v}^*_i$ is positive (bottom middle plot in Figure \ref{fig:encltr}), the new $C'$ is the $\overline{BC}$ intercept of the level set, $A'$ is its $\overline{AC}$ intercept and $B'$ is the 1-projection of $A'$ to $\overline{BC}$. Otherwise, the same arguments of the preceding case apply, with $P'$ larger than $P$ and $\rho_i$, and $Q'$ smaller that $Q$.

These observations are formalized in the following result:
\begin{lemma}[Reduction of enclosing triangles]\label{lem:encltr}
	Let the points $A$, $B$, and $C$ define an enclosing triangle. Setting $w_B-l_B<\rho_i\leq w_C-l_C$ and solving the resulting task with rewards $(r-\rho_i\,k)$ to find $v^*_i$ results in a strictly smaller, possibly degenerate enclosing triangle. 
\end{lemma}

\begin{proof}
	The preceding discussion and Figure \ref{fig:encltr} show how to build the triangle in each case and why it must contain the mapping of the optimal policy. In Appendix A we show that the resulting triangle is indeed \emph{enclosing}, that is, that it holds the conditions in Definition \ref{def:encltr}, and that it is strictly smaller than the original enclosing triangle.
\end{proof}

\subsubsection{An additional termination condition}
Another very important geometrical feature arises in the case when the same policy is nudged-optimal for two different nudges, $\rho_1$ and $\rho_2$, with optimal nudged values of different signs. This is pictured in Figure \ref{fig:swsn}. Assume that the gain is set to some value $\rho_1$ and the nudged task is solved, resulting on some nudged-optimal policy $\pi_1^*$ with positive value (of the recurrent state $s_I$) $\mathcal{v}_1^*>0$. The geometry of this is shown in the top left plot. As is the case with Lemma \ref{lem:encltr} and Figure \ref{fig:rho1}, the region shaded light grey cannot contain the $w-l$ mapping of any policies without contradicting the optimality of $\pi_1^*$, while the dark grey area represents the next enclosing triangle. If, next, the gain is set to $\rho_2$ and the same policy is found to be optimal, $\pi_1^*=\pi_2^*$, but now with $\mathcal{v}_2^*<0$, not only is the geometry as shown in the top right plot, again with no policies mapping to the light gray area, but remarkably we can also conclude that the optimizer of both cases is also the gain optimal policy of the global task. Indeed, the bottom plot in Figure \ref{fig:swsn} shows in light grey the union of the areas that cannot contain policy mappings. in principle this would reduce the uncertainty region to the dark grey enclosing triangle. However, $\pi_1^*=\pi_2^*=\pi^*$ is known to map to a point in each of the two solid black line segments, so it must be on their intersection. Since this is the extreme vertex of the new enclosing triangle in the direction of increase of the level sets of the average-reward problem and a policy, namely $\pi^*$, is known to reside there, it then must solve the task.

\begin{figure}[!h]
	\begin{center}
		\includegraphics{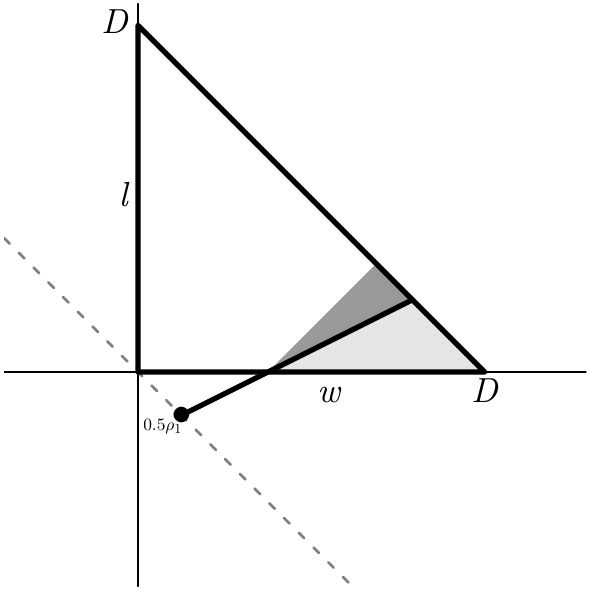}
		\includegraphics{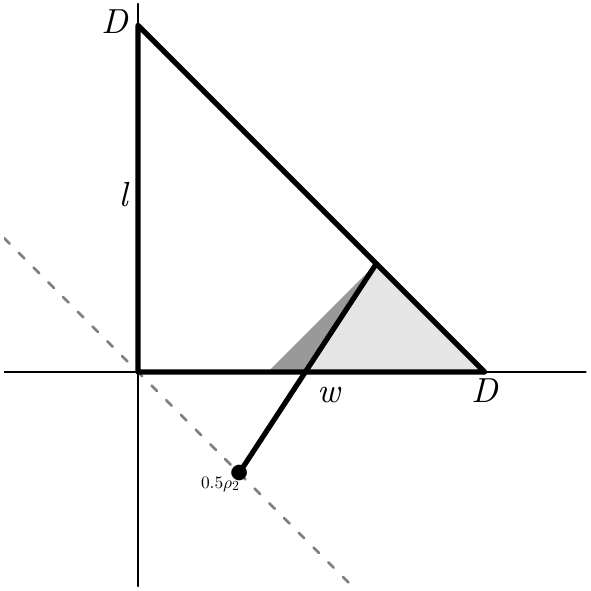}\\
		\includegraphics{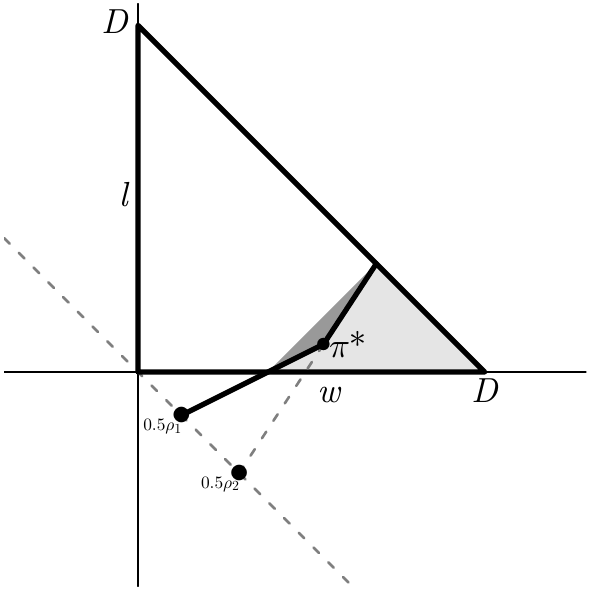}
	\end{center}
	\caption{Top: Geometry of solving a nudged task for two different set gains $\rho_1$ and $\rho_2$. Bottom: If the same policy $\pi^*$ is the optimizer in both cases, it lies on the intersection of the solid black lines and solves the average-reward task.}
	\label{fig:swsn}
\end{figure}

Observe that this argument holds for the opposite change of sign, from $\mathcal{v}_1^*<0$ to $\mathcal{v}_2^*>0$. The general observation can be, thus, generalized:

\begin{lemma}[Termination by zero crossing]
	If the same policy $\pi_i^* = \pi^*_{i+1}$ is optimal for two consecutive nudges $\rho_i$ and $\rho_{i+1}$ and the value of the recurrent state changes signs between them, the policy is gain-optimal.
\end{lemma}

\begin{remark}
	A similar observation was made by \citet{bertsekas1998new}, who observed that the optimal value of a reference state in SSPs is a concave, monotonically decreasing, \emph{piecewise linear} function of gain. However, as a consequence of value-drift during learning, unlike all of the methods in the literature only optimal nudging can rely for termination on the first zero-crossing. 
\end{remark}

The final step remaining to our formulation is, at the start of an iteration of Algorithm \ref{alg:nudged}, deciding the set gain, which geometrically corresponds to choosing the location of vertex of the next pencil of lines. The next section shows how to do this optimally.

\subsubsection{\texttt{minmax} Uncertainty}
\label{sec:minmaxunc}
We have already discussed the implications of setting the current gain/nudging to some value $\rho_i$ (for which upper and lower bounds are known), and then solving the resulting cumulative reward task. The problem remains finding a good way to update $\rho_i$. It turns out that the updates can be done optimally, in a minmax sense.

We start from an arbitrary enclosing triangle with vertices $A$, $B$, and $C$. By definition of enclosing triangle, the slope of $\overline{AB}$ is unity. We will refer to the slope of the $\overline{BC}$ segment as $m_\beta$ and that of $\overline{AC}$ as $m_\gamma$. For notational simplicity, we will refer to the projection to the $w=-l$ line, in the direction of some slope $m_\zeta$, of a point $X$ with coordinates $(w_X,\,l_X)$,  as the ``$\zeta$ projection of $X$'', $X_\zeta$. This kind of projection has the simple general form
\begin{align*}
X_\zeta&=\frac{m_\zeta\,w_X - l_X}{m_\zeta+1}\;\;.
\end{align*}
Naturally, attempting a $-1$-projection leads to an indetermination. On the other hand, if $m_\zeta$ is infinite, then $X_\zeta=w_X$.

An enclosing triangle with vertices $A=(0,0)$, $B=(\frac{D}{2}, \frac{D}{2})$ and $C=(D,0)$, then, has a $\rho$-uncertainty of the form
\begin{align*}
P&\leq\frac{\rho^*}{2}\leq Q\;\;,\\
A_1=B_1&\leq\frac{\rho^*}{2}\leq C_1\;\;,\\
w_A-l_A=w_B-l_B&\leq\rho^*\leq w_C-l_C\;\;.
\end{align*}
Assume $\rho_x$ is set to some value inside this uncertainty region. The goal is to find $\rho_i$ as the \emph{best} location for $\rho_x$. Solving the $\rho_x$-nudged problem, that is, the cumulative task with rewards $(r - \rho_x\,k)$, the initial state $s_I$ has optimal nudged value $\mathcal{v}^*_x$. Disregarding the cases in which this leads to immediate termination, the resulting geometry of the problem would be similar to that shown in the top middle, top right, and bottom middle plots in Figure \ref{fig:encltr}. In all three cases, the resulting reduction in uncertainty is of the form
\begin{align*}
P'=B'_1&\leq\frac{\rho^*}{2}\leq C'_1=Q'\;\;.
\end{align*}
We call the resulting, reduced uncertainty for both cases with $\mathcal{v}^*_x<0$ (Figure \ref{fig:encltr} top middle and top right) \emph{left uncertainty}. It has the following, convenient features:

\begin{lemma}[Left Uncertainty]\label{lem:leftunc}
For any enclosing triangle with vertices $A$, $B$, $C$:
\begin{enumerate}
	\item For any $A_1=B_1\leq\frac{\rho_x}{2}\leq C_1$, the maximum possible left uncertainty, $u_l^*$ occurs when the line with slope $\frac{D-\mathcal{v}^*_x}{D+\mathcal{v}^*_x}$ from $\left(\frac{\rho_x}{2},\,-\frac{\rho_x}{2}\right)$ intercepts $\overline{AB}$ and $\overline{BC}$ at point $B$.
	\item When this is the case, the maximum left uncertainty is
		\begin{align}
		u_l^*=2\frac{(\frac{\rho_x}{2}-B_1)}{(\frac{\rho_x}{2}-B_\gamma)}\,(C_\gamma - B_\gamma)\;\;.\label{eq:leftunc}
		\end{align}
	\item This expression is a monotonically increasing function of $\rho_x$.
	\item The minimum value of this function is zero, and it is attained when $\frac{\rho_x}{2}=B_1$.
\end{enumerate}
\end{lemma}

The proof of Lemma \ref{lem:leftunc} is presented in Appendix B.

The maximum left uncertainty for some $\rho_x$ in an enclosing triangle, described in Equation (\ref{eq:leftunc}), is a conic section, which by rearrangement of its terms can also be represented using the homogeneous form
\begin{align}
\left(
  \begin{array}{ccc}
    \rho_x & u_l^* & 1 \\
  \end{array}
\right)
\left(
  \begin{array}{ccc}
    0 & 1 & (B_\gamma-C_\gamma) \\
    1 & 0 & -B_\gamma \\
    (B_\gamma-C_\gamma) & -B_\gamma & -2B_1(B_\gamma-C_\gamma) \\
  \end{array}
\right)
\left(
  \begin{array}{c}
    \rho_x \\
    u_l^* \\
    1 \\
  \end{array}
\right)&=0\;\;.\label{eq:lunconic}
\end{align}

Conversely to left uncertainty, we call the resulting new uncertainty for $\mathcal{v}^*_x>0$ (Figure \ref{fig:encltr}, bottom middle) \emph{right uncertainty}. The derivation and optimization of right uncertainty as a function of $\rho_x$ is considerably more intricate than left uncertainty. The following result summarizes its features.

\begin{lemma}[Right Uncertainty]\label{lem:rightunc}
For any enclosing triangle with vertices $A$, $B$, $C$:
\begin{enumerate}
	\item For any $A_1=B_1\leq\frac{\rho_x}{2}\leq C_1$, the maximum possible right uncertainty, $u_r^*$, is
		\begin{align}
			u_r^*&=\frac{2s\sqrt{abcd(\frac{\rho_x}{2}-C_\gamma)(\frac{\rho_x}{2}-C_\beta)}+ad(\frac{\rho_x}{2}-C_\gamma) +bc(\frac{\rho_x}{2}-C_\beta)}{e}\;\;,\label{eq:rightunc}\\
			\intertext{with}
			s&=\sign(m_\beta-m_\gamma)\;\;,\nonumber\\
			a&=(1-m_\beta)\;\;,\nonumber\\
			b&=(1+m_\beta)\;\;,\nonumber\\
			c&=(1-m_\gamma)\;\;,\nonumber\\
			d&=(1+m_\gamma)\;\;,\nonumber\\
			c&=(d-b)=(m_\gamma-m_\beta)\;\;.\nonumber
		\end{align}
	\item This is a monotonically decreasing function of $\rho_x$,
	\item whose minimum value is zero, attained when $\frac{\rho_x}{2}=C_1$.
\end{enumerate}
\end{lemma}

The maximum right uncertainty for some $\rho_x$ in an enclosing triangle, described by Equation (\ref{eq:rightunc}) is also a conic section, with homogeneous form

\begin{align}
\left(
  \begin{array}{ccc}
    \rho_x & u_r^* & 1 \\
  \end{array}
\right)
\left(
  \begin{array}{ccc}
    c & -(b+a) & -C_1\,c \\
    -(b+a) & c & (C_\beta\, a+ C_\gamma\, b) \\
    -C_1\, c & (C_\beta\, a+ C_\gamma\, b) & C_1^2\,c \\
  \end{array}
\right)
\left(
  \begin{array}{c}
    \rho_x \\
    u_r^* \\
    1 \\
  \end{array}
\right) &=0\;\;. \label{eq:runconic}
\end{align}

The criterion when choosing $\rho_i$, must be to minimize the largest possible uncertainty, left or right, for the next step. Given the features of both uncertainty functions, this is straightforward:

\begin{theorem}[\texttt{minmax} Uncertainty]\label{th:minmax}
\begin{align*}
\rho_i&=\argmin_{\rho_x}\max\left[u_l^*,u_r^*\right]\\
\intertext{is a solution to}
u_l^*&=u_r^*\;\;.
\end{align*}
\end{theorem}
\begin{proof}
Since maximum left and right uncertainty are, respectively, monotonically increasing and decreasing functions of $\rho_x$, and both have the same minimum value, zero, the maximum between them is minimized when they are equal.
\end{proof}

Thus, in principle, the problem of choosing $\rho_i$ in order to minimize the possible uncertainty of the next iteration reduces to making the right hand sides of Equations (\ref{eq:leftunc}) and (\ref{eq:rightunc}) equal and solving for $\rho_x$. Although finding an analytical expression for this solution seems intractable, clearly any algorithm for root finding can be readily used here, particularly since lower and upper bounds for the variable are known (i.e., $A_1\leq\frac{\rho_x}{2}\leq C_1$).

However, even this is not necessary. Since both maximum and minimum left uncertainty are conic sections, and the homogeneous form for both is known, their intersection is straightforward to find in $O(1)$ time, following a process described in detail by \citet{perwass2008geometric}. This method only requires solving a $3\times3$ eigenproblem, so the time required to perform the computation of $\rho_i$ from the $w-l$ coordinates of the vertices of the current enclosing triangle is negligible. For completeness, the intersection process is presented in Appendix D, below.

Algorithm \ref{alg:optnudg} summarizes the optimal nudging approach.

\algsetblock[Name]{Repeat}{Stop}{4}{0.5cm}
\begin{algorithm}[H]
	\caption{Optimal Nudging}
	\label{alg:optnudg}
	\begin{algorithmic}[0]
		\State Set Bertsekas split
		\State Initialize ($\pi$ and $H$ or $Q$)
		\State Estimate $D$
		\State Initialize ($A=(0,0)$, $B=(\frac{D}{2},\frac{D}{2})$, $C=(0,D)$)
		\Repeat
		\State Compute $\rho_i$ (conic intersection)
		\State Set reward scheme to $(r-\rho_i k)$
		\State Solve by \textbf{any} RL method
		\State Compute from $\mathcal{v}_i^*$ the coordinates of the new enclosing triangle
		\State \textbf{until} Zero-crossing termination or $H^\pi(s_I)=0$
	\end{algorithmic}
\end{algorithm}

In the following two sections we will discuss the computational complexity of this method and present some experiments of its operation.

\section{Complexity of Optimal Nudging.}
In this section, we are interested in finding bounds on the number of calls to the ``black box'' reinforcement learning solver inside the loop in Algorithm \ref{alg:optnudg}. It is easy to see that therein lies the bulk of computation, since the other steps only involve geometric and algebraic computations that can be done in constant, negligible time. Moreover the type of reinforcement learning performed (dynamic programming, model-based or model free) and the specific algorithm used will have their own complexity bounds and convergence guarantees that are, in principle, transparent to optimal nudging.

In order to study the number of calls to reinforcement learning inside our algorithm, we will start by introducing a closely related variant and showing that it immediately provides a (possibly loose) bound for optimal nudging.

\algsetblock[Name]{Repeat}{Stop}{4}{0.5cm}
\begin{algorithm}[H]
	\caption{$\alpha$ Nudging}
	\label{alg:alpnudg}
	\begin{algorithmic}[0]
		\State Set Bertsekas split
		\State Set $0<\alpha\leq1$
		\State Initialize ($\pi$ and $H$ or $Q$)
		\State Initialize ($A=(0,0)$, $B=(\frac{D}{2},\frac{D}{2})$, $C=(0,D)$)
		\Repeat
		\State Set $\frac{\rho_i}{2}=(1+\alpha)B_1+\alpha C_1$
		\State Set reward scheme to $(r-\rho_i k)$
		\State Solve by any RL method
		\State Determine from $\mathcal{v}_i^*$ the coordinates of the new enclosing triangle
		\State \textbf{until} Zero-crossing termination or $H^\pi(s_I)=0$
	\end{algorithmic}
\end{algorithm}

Consider the ``$\alpha$-nudged'' Algorithm \ref{alg:alpnudg}. In it, in each gain-update step the nudging is set a fraction $\alpha$ of the interval between its current bounds, for a fixed $\alpha$ throughout. An easy upper bound on the reduction of the uncertainty in an $\alpha$-nudging step is to assume that the largest among the whole interval between $B_1$ and $\frac{\rho_i}{2}$, that is, the whole set of possible left uncertainty, or the whole interval between $\frac{\rho_i}{2}$ and $C_1$, that is, the complete space of possible right uncertainty, will become the uncertainty of the next step.

Thus, between steps, the uncertainty would reduce by a factor of $\hat{\alpha}=\max\left[\alpha,1-\alpha\right]$. Since the initial uncertainty has length $D$, it is easy to see that to bound the uncertainty in an interval of length at most $\varepsilon$ requires a minimum of 
\begin{align*}
n&\geq\frac{-1}{\log\hat{\alpha}}\log\left(\frac{D}{\varepsilon}\right)
\end{align*}
calls to reinforcement learning. Consequently, for any $\alpha$, $\alpha$-nudging has logarithmic complexity, requiring the solution of $O \left( \log \left( \frac{D}{\varepsilon} \right) \right)$ cumulative MDPs. Furthermore the constant term $\frac{-1}{\log\hat{\alpha}}$ is smallest when $\hat{\alpha}=\alpha=0.5$.

Therefore, setting $\alpha=0.5$ ensures that the uncertainty range will reduce at least in half between iterations. This is obviously a first bound on the complexity of optimal nudging: since Algorithm \ref{alg:optnudg} is for practical purposes adaptively adjusting $\alpha$ between iterations and is designed to minimize uncertainty, the Algorithm with $\alpha=0.5$ can never outperform it. The remaining question is whether the logarithmic bound is tight, that is, if for some enclosing triangle, the gain update is the midpoint of the uncertainty range and the best possible reduction of uncertainty is in half.

This turns out to be \emph{almost} the case. Consider the enclosing triangle with vertices $A=(0,0)$, $B=(l_B,l_B)$, and $C=(w_C,0)$, for a small value of $w_C$. Further, suppose that the gain-update step of the optimal nudging Algorithm \ref{alg:optnudg} finds a $\rho_i$ that can also be expressed as an adaptive $\alpha$ of the form $\frac{\rho_i}{2}=(1-\alpha_i)B_1+\alpha_iC_1$. Through direct substitution, the following values in the expressions for left and right uncertainty can be readily found:
\begin{align*}
m_\gamma&=0\;\;,\\
m_\beta&=\frac{l_B}{l_B-w_C}\;\;,\\
B_1&=0\;\;,\\
C_1&=\frac{w_C}{2}\;\;,\\
B_\gamma&=-l_B\;\;,\\
s&=1\;\;,\\
a&=\frac{2l_B-w_C}{l_B-w_C}\;\;,\\
b&=-\frac{w_C}{l_B-w_C}\;\;,\\
c&=-\frac{2l_B}{l_B-w_C}\;\;,\\
C_\beta&=\frac{l_Bw_C}{2l_B-w_C}\;\;,\\
C_\gamma&=0\;\;.
\end{align*}
Making the right hand sides of the expressions for left, Equation (\ref{eq:leftunc}), and right uncertainty, Equation (\ref{eq:rightunc}) equal, in order to find the minmax gain update, and then substituting, we have
\begin{align*}
\frac{(\frac{\rho_x}{2}-B_1)}{(\frac{\rho_x}{2}-B_\gamma)}\,(C_\gamma - B_\gamma) &= \frac{2s\sqrt{ab(\frac{\rho_x}{2}-C_\beta)(\frac{\rho_x}{2}-C_\gamma)}+a(\frac{\rho_x}{2}-C_\beta) +b(\frac{\rho_x}{2}-C_\gamma)}{c}\;\;,\\
\frac{\alpha_il_B}{\alpha_iw_C+2l_B}&= \frac{l_B(1-\alpha_i) + \alpha_iw_C + \sqrt{\alpha_iw_C(2l_B(1-\alpha_i)+\alpha_iw_C)}}{2l_B}\;\;.
\end{align*}
Making $w_C\rightarrow0$, to find $\alpha_i$,
\begin{align*}
\frac{\alpha_il_B}{2l_B}&=\frac{(1-\alpha_i)l_B + 0 + 0}{2l_B}\;\;,\\
\alpha_i&=1-\alpha_i\;\;,\\
\alpha_i&=\frac{1}{2}\;\;.
\end{align*}
Thus, for the iteration corresponding to this enclosing triangle as $w_C$ tends to zero, optimal nudging in fact approaches $\alpha$-nudging with $\alpha_i=0.5$. However, the maximum possible resulting uncertainty reduction for this case is slightly but strictly smaller than half,
\begin{align*}
u_r^*=u_l^*&=\frac{\alpha_iw_Cl_B}{\alpha_iw_C+2l_B}= \frac{\alpha_iw_C}{\alpha_i+\frac{2l_B}{w_C}}\;\;,\\
&=\frac{w_C}{4+\frac{w_C}{l_B}} < \frac{w_C}{4}\;\;.
\end{align*}

Hence, the same bound applies for $\alpha$ and optimal nudging, and the number of calls required to make the uncertainty interval smaller than $\varepsilon$ is
\begin{align*}
n=O\left(\log\left(\frac{D}{\varepsilon}\right)\right)\;\;.
\end{align*}

A number of further qualifications to this bound are required, however. Observe that, at any point during the run of optimal nudging, after setting the gain $\rho_i$ to some fixed value and solving the resulting cumulative-reward task, any of the two termination conditions can be met, indicating that the global SMDP has been solved. The bound only considers how big the new uncertainty range can get \emph{in the worst case} and \emph{for the worst possible enclosing triangle}. Thus, much faster operation than suggested by the ``reduction-of-uncertainty-in-half'' bound can be expected.

To illustrate this, we sampled five million valid enclosing triangles in the $w-l$ space for $D=1$ and studied how much optimizing $\rho_i$ effectively reduces the uncertainty range. For the sampling procedure, to obtain an enclosing triangle, we first generated an uniformly sampled value for $l_B+w_B$ and thereon, always uniformly, in order, $l_B-w_B$, $l_A+w_A$, $l_C+w_C$, and $l_C-w_C$. From these values we obtained the $A$, $B$, and $C$ coordinates of the triangle vertices and, from them, through Equations (\ref{eq:leftunc}) and (\ref{eq:rightunc}), the location of the gain that solves the minmax problem and the corresponding maximum new uncertainty.

\begin{figure}[!h]
	\begin{center}
		\includegraphics{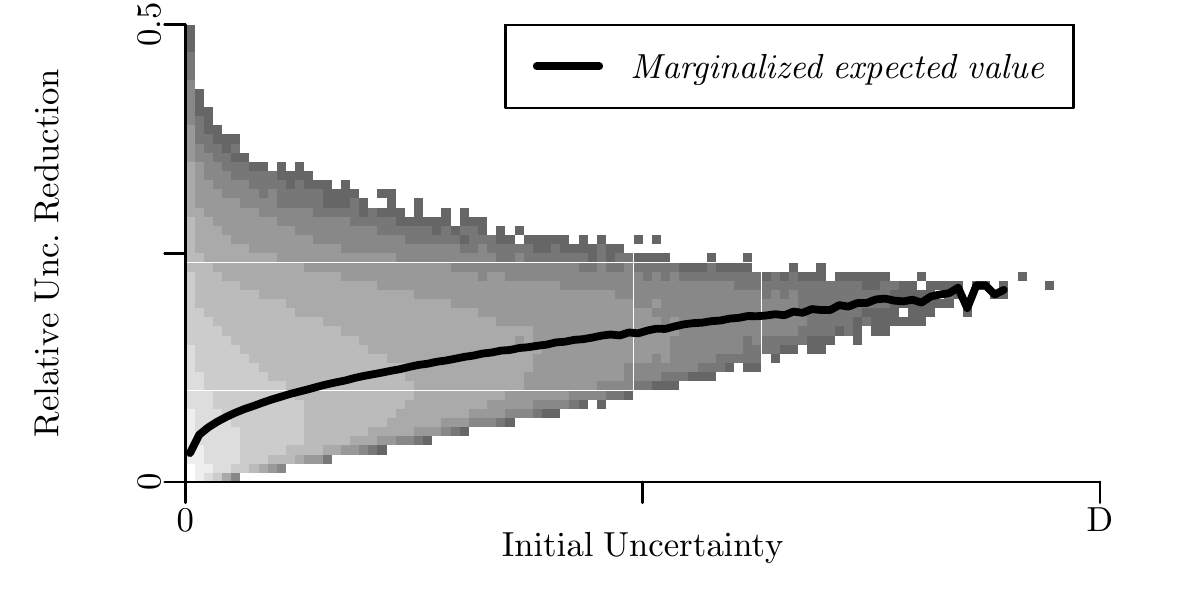}
	\end{center}
	\caption{Distribution and expectation of relative uncertainty reduction as a function of initial uncertainty, for a sample of five million artificial enclosing triangles.}
	\label{fig:5mega}
\end{figure}

Figure \ref{fig:5mega} shows an approximation to the level sets of the distribution of the relative uncertainty reduction and its marginal expectation over the length of the initial uncertainty. After the bound computed above, one can expect that for some triangles with a small initial uncertainty, namely those close to the origin, the uncertainty will reduce in just about half, as well as for all triangles to have reductions below the 0.5 line. This is indeed the case, but, for the sample, the expected reduction is much larger, to, in average, below 20\% or less of the original uncertainty. Thus, although obviously the actual number of calls to the reinforcement learning algorithm inside the loop in optimal nudging is task-dependent, gain-uncertainty can be expected to reduce in considerably more tan half after each call.

Moreover, nothing bars the use of \emph{transfer learning} between iterations. That is, the optimal policy found in an iteration for a set $\rho_i$, and its estimated value, can be used as the starting point of the learning stage after setting $\rho_{i+1}$, and if the optimal policies, or the two gains are \emph{close}, one should expect the learning in the second iteration to converge much faster than having started from scratch. As a consequence of this, not only few calls to the ``black box'' learning algorithm are required (typically even fewer than predicted by the logarithmic bound), but also as the iterations progress these calls can yield their solutions faster.

Although transfer in reinforcement learning is a very active research field, to our knowledge the problem of transfer between tasks whose only difference lies on the reward function has not been explicitly explored. We suggest that this kind of transfer can be convenient in practical implementations of optimal nudging, but leave its theoretical study as an open question.

\section{Experiments}
In this section we present a set of experimental results after applying optimal nudging to some sample tasks. The goal of these experiments is both to compare the performance of the methods introduced in this paper with algorithms from the literature, as well as to study certain features of the optimal nudging algorithm itself, such as its complexity and convergence, its sensitivity to the unichain condition, and the effect of transfer learning between iterations.

\subsection{Access Control Queuing Revisited}
Recall the motivating example task from Section \ref{ssec:exmpl}. Figure \ref{fig:optnudr} shows that a simple implementation of optimal nudging is approximately as good as the best finely tuned version of R-learning for that task, while, as discussed having less parameters and updates per step, as well as room for speed improvement in other fronts.

In this section we return to that task to explore the effect of the $D$ parameter. From Definition \ref{def:D}, remember that $D$ is a bound on unsigned, unnudged reward, and thus a possibly loose bound on gain. Since all the rewards in this task are non-negative, in the optimal nudging results presented Section \ref{ssec:exmpl}, $D$ could be approximated by setting the gain to zero and finding the policy of maximum expected reward before returning to the recurrent state with all servers occupied. The first half million samples of the run were used for this purpose. The exact value of $D$ that can be found in this task by this method (using dynamic programming, for example) is the tightest possible, that is, that value would be $D=\max_{\pi\in\Pi} \mathcal{v}^\pi$, in Definition \ref{def:D}.

More generally, for tasks with rewards of both signs, a looser approximation of $D$ can also be computed by setting the gain to zero and solving for maxim expected gain from the recurrent state with rewards $\hat{r}=|r|$ for all possible transitions. Notice that this process of $D$-approximation would add one call (that is, $O(1)$ calls) to the complexity bounds found in the preceding Section, and would thus have a negligible effect on the complexity of the algorithm.

\begin{figure}[!h]
	\begin{center}
		\includegraphics{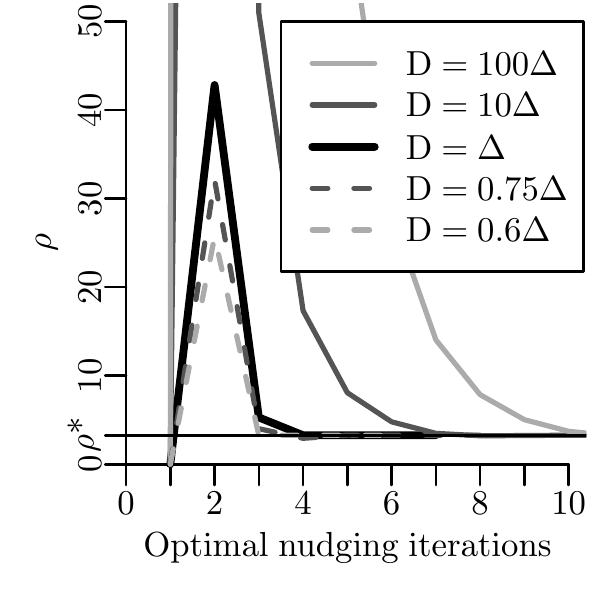}
		\includegraphics{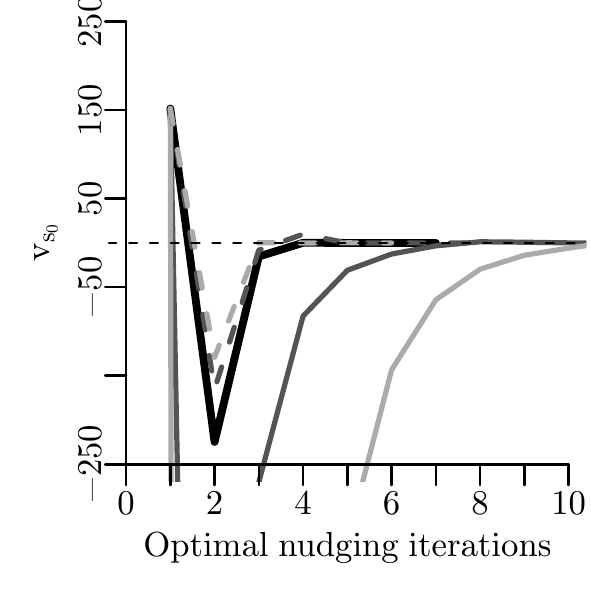}
	\end{center}
	\caption{Effect of varying $D$ on convergence. Left, to the optimal gain. Right, to the termination condition.}
	\label{fig:nudgD}
\end{figure}

For this experiment, let $\Delta$ be value of the zero-gain cumulative-optimal policy from the recurrent state, found exactly using dynamic programming ($\Delta=151.7715$). Figure \ref{fig:nudgD} shows, in solid black lines, the performance of optimal nudging when $D=\Delta$. As expected, in a very small number of iterations--around four--the current approximation of the gain (left plot) becomes nearly optimal and the nudged value of the recurrent state (right plot) approaches zero, the termination condition.

The solid grey lines in both plots show the effect of overestimating $D$ to ten and a hundred times $\Delta$. Since the number of calls to the black-box reinforcement learning algorithm inside the optimal nudging loop grows with the logarithm of $D$, the number of iterations required to achieve the same performance as above increases, in about three for each tenfold increase in $D$.

On the other hand, the dashed grey lines in the plots show the effect of underestimating $D$. Since, for this case, $\Delta$ is the smallest valid value for $D$, setting the parameter below can cause the loss of theoretical guarantees for the method. However, for some values of $D<\Delta$ the method still works, resulting in slightly faster convergence to the optimal gain and the termination condition. This is the case as long as the mapped policy cloud remains contained inside the triangle in the resulting $w-l$ space. For values of $D<0.6\Delta$, some of the policies visited by the algorithm have a negative $w$ or $l$ and our implementation of the method fails due to negative arguments to the root in the right uncertainty Equation (\ref{eq:rightunc}).

Thus, although overestimating $D$ naturally causes the algorithm to converge in a larger number of iterations, this increase is buffered by the logarithmic complexity of optimal nudging with respect to that parameter. Conversely, underestimating $D$ below its tightest bound can accelerate learning somewhat, but there is a cost in loss of theoretical guarantees and the eventual failure of the method.

\subsection{The Bertsekas Experimental Testbed}
For this second set of experiments, we consider a number of tasks from the paper by \cite{bertsekas1998new}, that introduces the state-space splitting process that we use for optimal nudging. That paper presents two versions of dynamic-programming stochastic-shortest-path updates, SSP-Jacobi and SSP-Gauss-Seidel.

We consider the first two kinds of randomly generated tasks, called simply ``Problems of Table 1 [or 2]'' in the paper; respectively ``T1'' and ``T2''' hereon. In both types, each task has $n$ states, the reward\footnote{In the context of the original paper, taking actions causes in a positive \emph{cost} and the goal is to minimize average cost. For consistency with our derivation, without altering algorithmic performance, we instead maximize positive reward.} of taking each action $(s,a)$ is randomly selected from the range $(0,\,n)$ according to a uniform distribution, for each pair $(s,a)$, the states $s'$ for which the transition probability is non-zero are selected according to some rule, and all non-zero transition probabilities are drawn from the uniform distribution in $(0,\,1)$ and normalized. In all cases, to ensure compliance of the unichiain condition, we set the $n$-th state as recurrent for the Bertsekas split and, before normalization, add small transition probabilities ($p=10^{-6}$) to it, from all states and actions. All the results discussed below are the average of five runs (instead of Bertsekas's two runs) for each set up.

The results listed below  only count the number of \emph{sweeps} of the methods. A sweep is simply a pass updating the value of all states. The SSP methods perform a continuous run of sweeps until termination, while each optimal nudging iteration is itself comprised of a number of sweeps, which are then added to determine the total number for the algorithm.

It is worth noting that for each sample, that is, for the update of the value of each state, while the SSP methods also update two bounds for the gain and the gain itself, optimal nudging doesn't perform any additional updates, which means that the nudged iterations are, from the outset, considerably faster in all cases.

The T1 tasks have only one action available per state, so obviously there is only one policy per task and the problem is policy evaluation rather than improvement. As in the source paper, we consider the cases with $n$ between 10 and 50. The transition probability matrix is sparse and unstructured. Each transition is non-zero with probability $q$ and we evaluate the cases with $q\in\{0.5,0.1,0.05\}$. As Bertsekas notes, there is a large variance in the number of iterations of either method in different tasks generated from the same parameters, but their relative proportions are fairly consistent.

The implementations of optimal nudging for the T1 tasks mirror the Jacobi and Gauss-Seidel updates from the original paper for the reinforcement learning step, and for both of them include two cases: \emph{raw} learning in which after each gain update the values of all states are reset to zero and the change-of-sign termination condition is ignored, and learning with transfer of the latest values between gain updates and termination by zero crossing. Preliminary experiments showed that for this last configuration, any significant reduction in the number of sweeps comes from the termination condition.

\begin{table}[t]
	\begin{center}
		\begin{small}
			\begin{tabular}{llrrrr}
				\toprule
				& & \multicolumn{2}{c}{Jacobi} & \multicolumn{2}{c}{Gauss-Seidel}\\
				$n$ & $q$ & ON/SSP & ONTS/SSP & ON/SSP & ONTS/SSP\\
				\midrule
				10 & 0.05 & 3.09 &0.79 &4.11 &0.94\\
				10 & 0.1  & 3.30 &0.49 &3.86 &0.61\\
				10 & 0.5  & 20.94 &2.56 & 8.92 &1.20\\
				\midrule
				20 & 0.05 & 1.66 &0.34 &2.02 &0.43\\
				20 & 0.1  &17.45 &2.60 &5.46 &0.85\\
				20 & 0.5  &27.54 &3.67 &18.32 &2.70\\
				\midrule
				30 & 0.05 &3.34 &0.53 &2.29 &0.39\\
				30 & 0.1  &5.16 &0.75 &5.91 &0.73\\
				30 & 0.5  &27.62 &4.11 &12.82 &2.04\\
				\midrule
				40 & 0.05 &3.92 &0.62 &3.96 &0.61\\
				40 & 0.1  &19.74 &3.01 &7.81 &1.27\\
				40 & 0.5  &29.00 &4.47 &15.57 &2.55\\
				\midrule
				50 & 0.05 &7.37 &1.16 &3.43 &0.58\\
				50 & 0.1  &38.35 &6.20 &19.11 &3.26\\
				50 & 0.5  &32.65 &5.20 &17.53 &2.97\\
				\bottomrule
			\end{tabular}
		\end{small}
		\caption{T1 tasks. Comparison of the ratio of sweeps of optimal nudging (ON), and optimal nudging with transfer and zero-crossing checks (ONTS) over SSP methods. Averages over five runs.}
		\label{tab:BksT1}
	\end{center}
\end{table}

Table \ref{tab:BksT1} summarizes the ratio of the number of sweeps required by optimal nudging over those required by the SSP algorithms. The effect of considering the termination by zero-crossing is striking. Whereas the raw version of optimal nudging can take in average up to 40 times as many sweeps altogether to reach similar results to SSP, including the change-of-sign condition never yialds a ratio higher 10. Considering that inside the sweeps each nudged update is faster, this means that in most cases for this task optimal nudging has a comparable performance to the SSP methods. Moreover, in many cases, remarkably in the more sparse--most difficult--ones with the smallest $q$, optimal nudging requires in average less sweeps than the other algorithms.

The T2 tasks also have only one action available from each state, for $n$ between 10 and 50, but the transitions are much more structured; the only non-zero transition probabilities from a state $i\in\{1,\,n\}$ are to states $i-1$, $i$, and $i+1$ (with the obvious exceptions for states 1 and $n$). 

Once more for this task we compare the performance of the raw and change-of-sign condition versions of optimal nudging with the SSP methods for both Jordan and Gauss-Seidel updates.

\begin{table}[h]
	\begin{center}
		\begin{small}
			\begin{tabular}{lrrrr}
				\toprule
				& \multicolumn{2}{c}{Jacobi} & \multicolumn{2}{c}{Gauss-Seidel}\\
				$n$  & ON/SSP & ONTS/SSP & ON/SSP & ONTS/SSP\\
				\midrule
				10 & 8.78 & 1.55 & 7.86 & 1.45\\
				20 & 9.04 & 1.79 & 8.23 & 1.69\\
				30 & 5.10 & 0.98 & 4.68 & 0.94\\
				40 & 3.38 & 0.68 & 3.27 & 0.69\\
				50 & 3.56 & 0.79 & 3.42 & 0.78\\
				\bottomrule
			\end{tabular}
		\end{small}
		\caption{T2 tasks. Comparison of the ratio of sweeps of optimal nudging (ON), and optimal nudging with transfer and zero-crossing checks (ONTS) over SSP methods. Averages over five runs.}
		\label{tab:BksT2}
	\end{center}
\end{table}

Table \ref{tab:BksT2} summarizes for these the ratio of the number of sweeps required by optimal nudging over those required by the SSP algorithms. In this case, even for the raw version of optimal nudging, it never takes over 10 times as many sweeps as SSP and, notably, the ratio reduces consistently as size of the problem grows. Once the zero-crossing condition is included, both optimal nudging versions become much faster, requiring less sweeps than the SSP methods for the largest tasks.

\subsection{Discrete Tracking}
The final experiment compares the performance of optimal nudging and R-learning on a problem with more complex dynamics and larger action (and policy) space. This task is a discretization of the ``Tracking'' experiment discussed in the paper by \citet{van2007reinforcement}.

\begin{figure}[!h]
	\begin{center}
		\includegraphics[width=6cm]{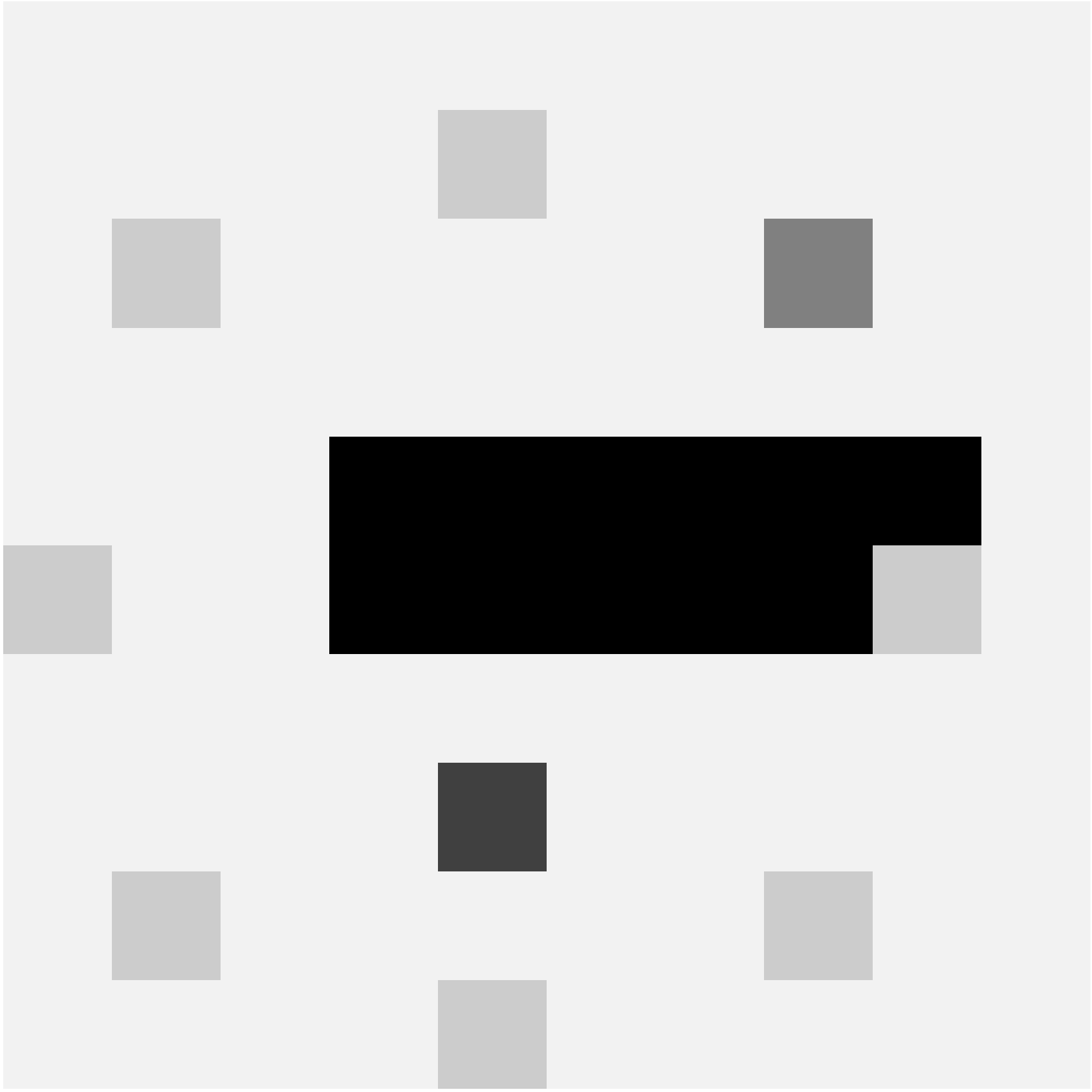}
	\end{center}
	\caption{Environment of the tracking task. The target moves anticlockwise from one light grey cell to another. The agent (dark grey) can move up to four cells in each direction, without crossing the black obstacle or exiting the grid. The goal is for the agent to follow the target as closely as possible, moving as little as possible.}
	\label{fig:trgrid}
\end{figure}

In discrete tracking, inside a $10\times10$ grid with a central obstacle, as pictured in Figure \ref{fig:trgrid}, a moving target follows a circular path, traversing anticlockwise the eight cells shaded in light grey. The goal of the agent is, at each step, to minimize its distance to the target. While the target can cross the obstacle an indeed one of its positions is \emph{on} it, the agent can't pass over and must learn to surround it.

At each decision epoch, the agent's actions are moving to any cell in a $9\times9$ square centred on it, so it can move at most four cells in each direction. This is sufficient to manoeuvre around the obstacle and for the agent to be able to keep up with the target. In case of collision with the obstacle or the edges of the grid, the agent moves to the valid cell closest to the exit or collision point.

After taking an action, moving the agent from state $s(s_x,s_y)$ to $s'(s_x',s_y')$, and the target to $t'(t_x',t_y')$, the reward is
\begin{align*}
r=\frac{1}{\left(1+(s_x'-t_x')^2+(s_y'+t_y')^2\right)^2}\;\;.
\end{align*}
Thus, if the agent is able to move to the exact cell that the target will occupy, the reinforcement is 1, and it will decay quickly the further agent and target are apart.

On the other hand, the cost of each action grows with the distance moved,
\begin{align*}
c=1+|s_x-s_x'|+|s_y-s_y'|\;\;.
\end{align*}
Notice that our assumption that all costs are larger than or equal to one holds for this model, and that the policy that minimizes action costs must require the agent to stand still at each state.

As described, this task isn't unichain and it doesn't have a recurrent state. Indeed, the policy of staying in place has one recurrent set for each state in the task, so the unichain condition doesn't hold, and it is straightforward to design policies that bring and keep the agent in different cells, so their recurrent sets have no (recurrent) states in common.

To overcome this, we arbitrarily set as recurrent the state with both the target and the agent in the cell above and to the right of the bottom-left corner of the grid. From any state with the target in its position just before that (light grey cell near the middle left in Figure \ref{fig:trgrid}), the agent moves to the recurrent state with probability 1 for any action. Observe that making the transition probability to this recurrent state smaller than one would only increase the value of $D$, but would not change the optimal policies, or their gain/value.

We average 10 runs of two algorithms to solve this task, optimal nudging and two different R-learning set-ups. In all cases, $\varepsilon$-greedy action selection is used, with $\varepsilon=0.5$. In order to sample the state space more uniformly, every 10 moves the state is reset to a randomly selected one.

For optimal nudging, we use Q-learning with a learning rate $\alpha=0.5$ and compute the new gain every 250000 samples. $Q$ is initialized in zero for all state-action pairs. No transfer of $Q$ is made between nudging iterations. Although an upper bound on $D$ is readily available ($D<8$, which would be the value of visiting all the same cells as the target, including the one inside the obstacle), we still use the first batch of samples to approximate it. Additionally, although we keep track of the zero-crossing condition, we opt not to terminate the computation when it holds, and allow the algorithm to observe 1.5 million samples.

For R-learning, we set the learning rate $\beta=10^{-4}$, and $\alpha=0.5$ (``R-learning 1'') and $\alpha=0.01$ (``R-learning 2''). The gain is initialized in zero an all other parameters are inherited from the optimal nudging set-up.

\begin{figure}[!h]
	\begin{center}
		\includegraphics{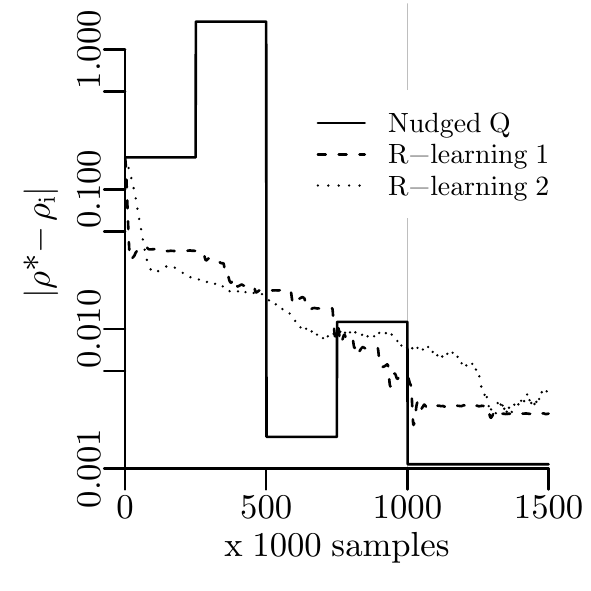}
	\end{center}
	\caption{Approximation to the optimal gain. The vertical grey line indicates that the zero-crossing termination condition has already been met.}
	\label{fig:trgt}
\end{figure}

Figure \ref{fig:trgt} shows the performance of the algorithms, averaged over the 10 runs, to approximate the optimal gain. It is evident that after one million samples optimal nudging converges to a closer value to $\rho^*$ than either R-learning set-up, but in fact the plot describes the situation only partially and our algorithm performs even better than initially apparent.

For starters, recall that optimal nudging has one update less per sample, that of the gain, so there is a significant reduction in the computation time even when the number of samples is the same. Moreover, in six out of the ten optimal nudging runs $\rho^*$ was actually found exactly, up to machine precision. This is something that R-learning couldn't accomplish in any of the experimental runs, on the contrary, while optimal nudging always found the gain-optimal policy, in about half of the cases R-learning (both set-ups) found a policy that differed from the optimal in at least one state.

Additionally, the grey vertical line indicates the point at which the zero-crossing termination condition had been met in all optimal nudging runs, so, in fact for this particular experiment optimal nudging requires only two thirds of the samples to find an approximation to the optimal gain with only about one third of the error of R-learning.

\section{Conclusions and Future Work}
We have presented a novel semi-Markov and average-reward reinforcement learning method that, while belonging to the generic class of algorithms described by Algorithm \ref{alg:generic}, differs in one crucial aspect, the frequency and mode of gain updates.

While all methods described in the preceding literature update the gain after each \emph{sample}, that is after the update of the gain-adjusted value of each state for dynamic programming methods or after taking each action (or greedy action) in model-based or model-free methods, the optimal nudging algorithm consists of a series of cumulative tasks, for fixed gains that are only updated once the current task is considered solved.

This delaying of the gain updates is feasible because, once the nudged value for a fixed gain is known, it is possible and relatively simple to ensure in each update that the maximum possible new gain uncertainty will be minimized. We have introduced the straightforward transformation of the policies into the $w-l$ space and, through geometric analysis therein, shown how to select the new value of the gain, to fix for the next iteration, such that this minmax result is guaranteed.

The disentangling of value and gain updates has the further advantage of allowing the use of any cumulative-reward reinforcement learning method, either exploiting its robustness and speed of approximation or inheriting its theoretical convergence properties and guarantees. Regarding the complexity of optimal nudging itself, we have shown that the number of calls to the learning algorithm are at most logarithmic on the desired precision and the $D$ parameter, which is an upper bound on the gain.

Also, we have proved an additional condition for early termination, when between optimal nudging iterations the same policy optimizes nudged value, and the optimal value for a reference state switches sign. This condition is unique to optimal nudging, since in any other method that continuously updates the gain, many sign changes can be expected for the reference state throughout, and none of them can conclusively guarantee termination. Moreover, our experiments with the set of random tasks from \cite{bertsekas1998new} show that this condition can be very effective in reducing the number of samples required to learn in practical cases.

Additionally, compared with traditional algorithms, while maintaining a competitive performance and sometimes outperforming traditional methods for tasks of increasing complexity, optimal nudging has the advantage of requiring at least one parameter less, the learning rate of the gain updates, as well as between one and three updates less per sample. This last improvement can represent a significant reduction in computation time, in those cases in which samples are already stored or can be observed quickly (compared to the time required to perform the updates).

Finally, a number of lines of future work are open for the study or improvement of the optimal nudging algorithm. First, Figure \ref{fig:optnudr} suggests that each call to the ``balck-box'' reinforcement learning method should have its own termination condition, which could probably be set adaptively to depend on the nudged-optimal value of the recurrent state, since in the early learning iterations the nudged-optimal policies likely require far less precision to terminate without affecting performance.

Likewise, as mentioned above, transfer learning, in this case keeping the state values after updating the gain, can lead to faster termination of the reinforcement learning algorithm, specially towards the end, when the differences between nudged-optimal policies tend to reduce. Although our observation in the Bertsekas testbed is that transfer doesn't have the same effect as the zero-crossing stopping condition, we would suggest to study how it alters learning speed, either on different tasks and for different algorithms or from a theoretical perspective.

On a different topic, it would be constructive to explore to what extent the $w-l$ transformation could be directly applied to solve average-reward and semi-Markov tasks with continuous state/action spaces. This is a kind of problem that has received, to our knowledge, little attention in the literature, and some preliminary experiments suggest that the extension of our results to that arena can be relatively straightforward.

The final avenue for future work suggested is the extension of the complexity results of Section 5 to the case in which the black-box algorithm invoked inside the optimal nudging loop is PAC-MDP. Our preliminary analysis indicates that the number of calls to such an algorithm would be dependent on the term $\frac{1}{\underline{\alpha}}$, where $\underline{\alpha}$ is the smallest $\alpha$ such that $\frac{\rho_i}{2}=(1-\alpha)B_1 + \alpha C_1$ is a valid optimal nudging update.

Some sampling (five million randomly-generated valid enclosing triangles, for $D=1$) suggests that not only is $\underline{\alpha}$ positive, but it is also not very small, equal to approximately 0.11. However, since there is no analytic expression for $\rho_i$, and thus neither for $\alpha$, proving that this is indeed the case is by no means trivial, and the question of the exact type of dependence of the number of calls on the inverse of $\underline{\alpha}$ remains open anyway.

\newpage

\appendix
\section*{Appendix A.}
\label{app:enctr}

In this Appendix we prove the remainder of the ``reduction of enclosing triangles'' Lemma \ref{lem:encltr} from Section~\ref{sec:enctrslrunc}:

\noindent{\bf Lemma} {\it Let the points $A=(w_A,l_A)$, $B=(w_B,l_B)$, and $C=(w_C,l_C)$ define an enclosing triangle. Setting $w_B-l_B<\rho_i\leq w_C-l_C$ and solving the resulting task with rewards $(r-\rho_i\,k)$ to find $v^*_i$ results in a strictly smaller, possibly degenerate enclosing triangle.}

Our in-line discussion and Figure \ref{fig:encltr} show how to find $A'$, $B'$, and $C'$ depending on the sign and optimal value of the recurrent state for the fixed gain $\rho_i$. It remains to show that the resulting triangle $A'B'C'$ is strictly smaller than $ABC$ as well as \emph{enclosing}, that is, that the conditions of Definition \ref{def:encltr} hold. Those conditions are:

\begin{enumerate}
	\item $w_B\geq w_A$; $l_B\geq l_A$.
	\item $\frac{l_B-l_A}{w_B-w_A}=1$.
	\item $w_A\geq l_A$; $w_B\geq l_B$; $w_C\geq l_C$.
	\item $P=\frac{w_B-l_B}{2}=\frac{w_A-l_A}{2}\leq\frac{w_C-l_C}{2}=Q$.
	\item $0\leq\frac{l_C-l_A}{w_C-w_A}\leq1$
	\item $\left| \frac{l_C-l_B}{w_C-w_B} \right| \geq 1$
\end{enumerate}

For brevity, we will only consider one case, reproduced in Figure \ref{fig:encltrd} below. The proof procedure for the other two non-degenerate cases follows the same steps and adds no further insight. (In light of Remark \ref{rm:degenc}, the proof of the degenerate cases is trivial and shall be omitted).

\begin{figure}[!h]
	\begin{center}
		\includegraphics[trim=5cm 5cm 5cm 0, clip=true]{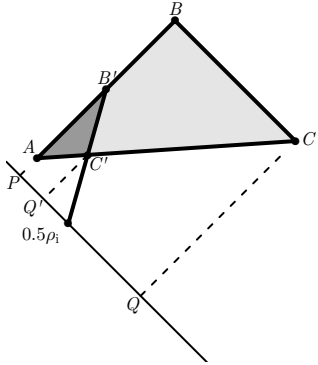}
	\end{center}
	\caption{(Detail from Figure \ref{fig:encltr}) Outcome after solving a nudged task within the bounds of an arbitrary enclosing triangle.}
	\label{fig:encltrd}
\end{figure}

\noindent{\bf Proof} 
For this proof we consider the case in which, after setting the gain to $\rho_i$ and solving to find the nudged-optimal policy $\pi_i^*$ and its value $\mathcal{v}_i^*$, the line with slope $\frac{D-\mathcal{v}_i^*}{D+\mathcal{v}_i^*}$ that crosses the point $\left(\frac{\rho_i}{2},-\frac{\rho_i}{2}\right)$ also intersects the segments $\overline{AC}$ (at the point $C'$) and $\overline{AB}$ (at $B'$). We want to show that the triangle with vertices $A'=A$, $B'$ and $C'$ is enclosing and strictly smaller than $ABC$.

In order to study the whether the conditions listed above hold for $A'B'C'$, we will use a simple affine transformation (rotation and scaling) from the $w-l$ space to an auxiliary space $x-y$ in which all of $A$, $B$, $P$, and $B'$ have the same horizontal-component ($x$) value, the original $P$; and likewise $C$ and $Q$ have the same $x$ value, equal to that of $Q$ in $w-l$.

This transformation simplifies the analysis somewhat, particularly the expressions for the coordinates of $C'$, but it naturally has no effect on the validity of the proof, and the only difference with a proof without the mapping to $x-y$ is simply the complexity of the formulas but not the basic structure or the sequence of steps.

\begin{figure}[!h]
	\begin{center}
		\includegraphics{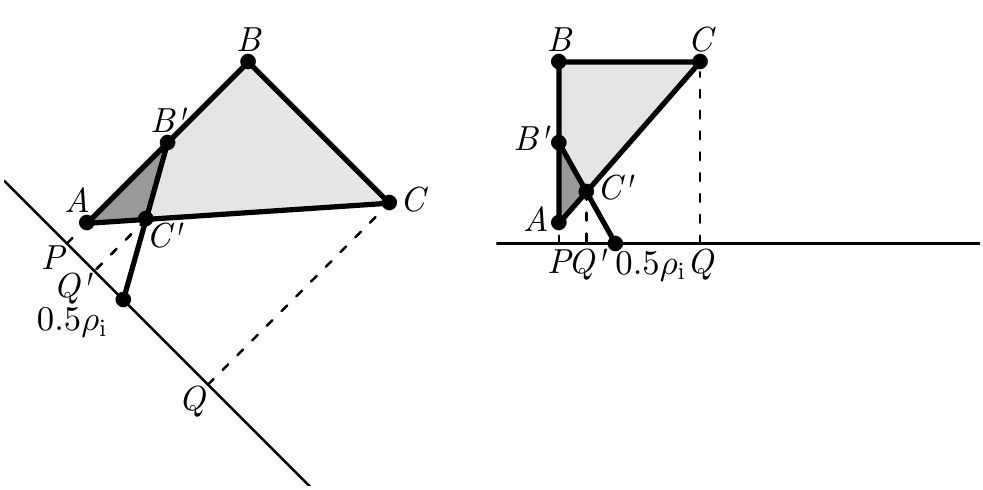}
	\end{center}
	\caption{Affine transformation of the $w-l$ space. Left, original system in $w-l$. Right, rotated and scaled system in the $x-y$ space. $P$ and $Q$ have the same value in the horizontal axis of both planes.}
	\label{fig:encltrrot}
\end{figure}

Figure \ref{fig:encltrrot} shows the geometry of the affine mapping from $w-l$ to $x-y$. The transformation matrix and equations are

\begin{align*}
\left[\begin{array}{r}
x \\ 
y
\end{array} \right] &=
\frac{1}{2}
\left[\begin{array}{rr}
1 & -1 \\ 
1 & 1
\end{array} \right]
\left[\begin{array}{c}
w\\ 
l
\end{array} \right] =
\left[\begin{array}{c}
\frac{w-l}{2} \\ 
\frac{w+l}{2}
\end{array} \right] \;\;,\\
\intertext{while the inverse transformation is}
\left[\begin{array}{r}
w \\ 
l
\end{array} \right] &=
\left[\begin{array}{rr}
1 & 1 \\ 
-1 & 1
\end{array} \right]
\left[\begin{array}{c}
x\\ 
y
\end{array} \right] =
\left[\begin{array}{c}
y+x \\ 
y-x
\end{array} \right] \;\;.
\end{align*}

In both the $w-l$ and $x-y$ planes, we have $P=\frac{w_A-l_A}{2}=\frac{w_B-l_B}{2}$ and $Q=\frac{w_C-l_C}{2}$. The mapped vertices of the original triangle have coordinates $A_{xy}=\left(x_A,y_A\right)=\left(P,\frac{w_A+l_A}{2}\right)$, $B_{xy}=\left(P,\frac{w_B+l_B}{2}\right)$, $C_{xy}=\left(Q,\frac{w_C+l_C}{2}\right)$.

Since the gain is bounded by the projection of the enclosing triangle vertices, we can express, in either plane, $\frac{\rho_i}{2} = (1-\alpha)P + \alpha Q$, for some $0<\alpha\leq1$. We omit the case with $\alpha=0$, since under it and a combination of several other conditions our complexity analysis from Section 5 may not hold.

Regarding the coordinates of the new triangle, as mentioned $A'=A$, so $A_{xy}'=A_{xy}$ and $B'$ is a point on the $\overline{AB}$ segment, so we can express its mapping as $B_{xy}'=(1-\beta)A_{xy} + \beta B_{xy}$, for $0\leq\beta\leq1$. $C'$, on the other hand, must be found analytically, as the intersection point of the $\overline{A_{xy}C_{xy}}$ and $\overline{\frac{\rho}{2}B_{xy}'}$ segments. Observe that the slope of the $\overline{A_{xy}C_{xy}}$ segment is positive. Indeed, since condition 5 holds for $ABC$, we have
\begin{align*}
\frac{l_C-l_A}{w_C-w_A} &\geq 0\;\;,\\
y_C-x_C-y_A+x_A &\geq 0\;\;,\\
m_\gamma=\frac{y_C-y_A}{x_C-x_A} &\geq 1\;\;.
\end{align*}

Solving the system of the two equations of the lines that contain the two segments, the $x-y$ coordinates of $C_{xy}'$ are readily found to be

\begin{align}
x_{C'}&=\frac{\alpha(Q-P)(m_\gamma P-y_A) + (y_A+\beta(y_B-y_A))(P+\alpha(Q-P))} {\alpha m_\gamma(Q-P) + y_A+\beta(y_B-y_A)}\;\;,\label{eq:xcprim}\\
\intertext{and}
y_{C'}&=\frac{(y_A+\beta(y_B-y_A))(\alpha m_\gamma(Q-P)+y_A)} {\alpha m_\gamma(Q-P) + y_A+\beta(y_B-y_A)}\;\;.\label{eq:ycprim}
\end{align}

We are ready to begin evaluating for $A'B'C'$ the conditions in the definition of an enclosing triangle, one by one. From the premise of the Lemma, since ABC \emph{is} an enclosing triangle, all conditions hold for it.

\begin{enumerate}
\item
\begin{align*}
w_{B'}&\geq w_{A'}\;\;,\\
x_{B'}+y_{B'}&\geq x_{A}+y_{A}\;\;,\\
P + (1-\beta)y_{A} + \beta y_{B} &\geq P + y_{A}\;\;,\\
\beta(y_B-y_A)&\geq0\;\;,\\
\intertext{since $\beta\geq0$,}
y_B-y_A&\geq0\;\;,\\
\frac{w_B+l_B}{2} - \frac{w_A+l_A}{2} &\geq 0\;\;,\\
(w_B-w_A) + (l_B-l_A) &\geq 0\;\;,
\intertext{which holds because both terms are positive by this same condition for $ABC$. Conversely,}
l_{B'} &\geq l_{A'}\;\;,\\
y_{B'}-x_{B'} &\geq y_A - x_A\;\;,\\
(1-\beta)y_A + \beta y_B - P &\geq y_A - P\;\;,\\
\beta(y_B-y_A) &\geq 0\;\;,\\
\intertext{and as above.}
\end{align*}

\item
\begin{align*}
\frac{l_{B'}-l_{A'}}{w_{B'}-w_{A'}}&=1\;\;,\\
y_{B'}-x_{B'} - y_A+x_A &= x_{B'}+y_{B'} - x_A-y_A\;\;,\\
2x_{B'}&=2x_A\;\;,\\
P&=P\;\;,\\
\intertext{which is true.}
\end{align*}

\item
\begin{align*}
w_{A'} &\geq l_{A'}\;\;,\\
x_{A}+y_{A} &\geq y_{A}-x_{A}\;\;,\\
x_A = P &\geq 0\;\;,\\
\frac{w_A-l_A}{2} &\geq 0\;\;,\\
w_A &\geq l_A\;\;,
\intertext{which holds. The proof is identical for $B'$, since $x_B'$ also equals $P$. For $C'$,}
w_{C'} &\geq l_{C'}\;\;,\\
x_{C'} &\geq 0\;\;,\\
\intertext{since all terms in the denominator in Equation (\ref{eq:xcprim}) are positive,}
\alpha(Q-P)(m_\gamma P-y_A) + (y_A+\beta(y_B-y_A))(P+\alpha(Q-P)) &\geq 0\;\;,\\
\alpha m_\gamma P(Q-P) + y_AP + \beta(y_B-y_A)(P+\alpha(Q-P)) &\geq 0\;\;,\\
\intertext{which is easy to verify to hold since all factors in all three summands are non-negative.}
\end{align*}

\item
\begin{align*}
P &\leq \frac{w_{C'}-l_{C'}}{2}\;\;,\\
P &\leq x_{C'}\;\;.\\
\intertext{Since, again, the denominator in Equation (\ref{eq:xcprim}) is positive,}
\alpha(Q-P)(m_\gamma P-y_A) + (y_A+\beta(y_B-y_A))(P+\alpha(Q-P)) &\geq \alpha m_\gamma P(Q-P) + P(y_A+\beta(y_B-y_A))\\
\alpha\beta(Q-P)(y_B-y_A) &\geq0\;\;,
\intertext{which holds because all factors are non-negative.}
\end{align*}

\item
\begin{align*}
0\leq\frac{l_{C'}-l_{A'}}{w_{C'}-w_{A'}}\leq1\;\;,
\end{align*}
This condition holds because the $\overline{AC}$ segment contains $\overline{A'C'}$, so both have the same slope.

\item 
\begin{align*}
\left|\frac{l_{B'}-l_{C'}}{w_{B'}-w_{C'}}\right| = \left|\frac{l_{B'}+\rho}{w_{B'}-\rho}\right| &\geq 1\;\;.
\intertext{The numerator is positive. If the denominator is positive as well,}
\frac{l_{B'}+\rho}{w_{B'}-\rho} &\geq 1\;\;,\\
l_{B'}+\rho &\geq {w_{B'}-\rho}\;\;,\\
l_{B'}-w_{B'} + 2\rho &\geq 0\;\;,\\
\rho&\geq\frac{w_{B'}-l_{B'}}{2}\;\;,\\
P+\alpha(Q-P)&\geq P\;\;,\\
\alpha(Q-P)&\geq 0\;\;,
\intertext{which holds because both factors are nonnegative. If the denominator is negative,}
\frac{l_{B'}+\rho}{\rho-w_{B'}} &\geq 1\;\;,\\
l_{B'}+\rho &\geq \rho-w_{B'}\;\;\,\\
l_{B'}+w_{B'} &\geq 0\;\;,
\end{align*}
which was proved above. If the denominator is zero, the slope of the segment is infinity (of either sign), for which the condition trivially holds.
\end{enumerate}

It remains showing that $Q'<Q$, and thus the new enclosing triangle is strictly smaller than the initial one. By direct derivation,

\begin{align*}
x_{C'} &< x_C\;\;,\\
\frac{\alpha(Q-P)(m_\gamma P-y_A) + (y_A+\beta(y_B-y_A))(P+\alpha(Q-P))} {\alpha m_\gamma(Q-P) + y_A+\beta(y_B-y_A)} &< Q\;\;,\\
\alpha m_\gamma(Q-P)^2 + y_A(Q-P) + \beta(Q-P)(y_B-y_A) - \alpha\beta(Q-P)(y_B-y_A)&>0\;\;.\\
\intertext{Since $Q>P$,}
\alpha m_\gamma(Q-P) + y_A + \beta(1-\alpha)(y_B-y_A)&>0\;\;.
\end{align*}
Under our assumption that the $ABC$ triangle is not degenerate, all factors in all summands are nonnegative. Furthermore, since we rule out the possibility that $\alpha=0$, the first term is indeed positive and, thus, the inequality holds.

\hfill\ensuremath{\blacksquare}

\newpage

\appendix
\section*{Appendix B.}
\label{app:leftunc}

In this Appendix we prove the left uncertainty Lemma \ref{lem:leftunc} from Section~\ref{sec:minmaxunc}:

\noindent{\bf Lemma}
{\it For any enclosing triangle with vertices $A$, $B$, $C$,
\begin{enumerate}
	\item For any $A_1=B_1\leq\frac{\rho_x}{2}\leq C_1$, the maximum possible left uncertainty, $u_l^*$ occurs when the line with slope $\frac{D-\mathcal{v}^*_x}{D+\mathcal{v}^*_x}$ from $\left(\frac{\rho_x}{2},\,-\frac{\rho_x}{2}\right)$ intercepts $\overline{AB}$ and $\overline{BC}$ at point $B$.
	\item When this is the case, the maximum left uncertainty is
	\begin{align*}
	u_l^*=2\frac{(\frac{\rho_x}{2}-B_1)}{(\frac{\rho_x}{2}-B_\gamma)}\,(C_\gamma - B_\gamma)\;\;.
	\end{align*}
	\item This expression is a monotonically increasing function of $\rho_x$.
	\item The minimum value of this function is zero, and it is attained when $\frac{\rho_x}{2}=B_1$.
\end{enumerate}
}

\noindent{\bf Proof} 
Figure \ref{fig:lunclem} illustrates the geometry of left uncertainty.

\begin{figure}[!h]
	\begin{center}
		\includegraphics{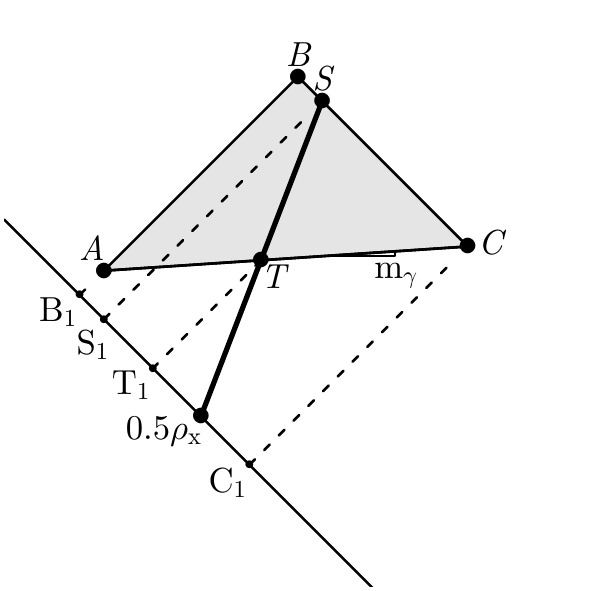}
	\end{center}
	\caption{Geometry and notation for the proof of the left uncertainty Lemma. If gain is fixed to some $B_1\leq\frac{\rho_x}{2}\leq C_1$, depending on the optimal nudged value the point $S$ may either lie on the $\overline{BC}$ (pictured) or $\overline{AB}$ segments, the point $T$ lies on the $\overline{AC}$ segment (which has slope $m_\gamma$), and the left uncertainty is $u_l=2(T_1-S_1)$.}
	\label{fig:lunclem}
\end{figure}

Setting the gain to $\rho_x$, we are interested in computing the left uncertainty as a function of the coordinates of $S$. The $\overline{ST}$ segment belongs to a line with expression

\newcommand{\rxm}{\frac{\rho_x}{2}}
\newcommand{\mg}{m_\gamma}
\newcommand{\Cg}{C_\gamma}
\begin{align}
l + \rxm &= \frac{l_S+\rxm}{w_S-\rxm}\left(w-\rxm\right)\;\;,\nonumber\\
l&=\frac{\left(l_S+\rxm\right)w-(w_S+l_S)\rxm}{w_S-\rxm}\;\;.\label{eq:ST}
\end{align}
Conversely, for points in $\overline{AC}$, including $T$,
\begin{align}
l&=\mg w - (\mg w_C-l_C)\;\;,\nonumber\\
l&=\mg w - (1+\mg)\Cg\;\;.\label{eq:AC}
\end{align}
Thus, for the 1-projection of point $T$,
\begin{align}
T_1 = \frac{w_T-l_T}{2} &= \frac{w_T-(\mg w_T - (1+\mg)\Cg)}{2}\;\;,\nonumber\\
&= \frac{(1-\mg)w_T + (1+\mg)\Cg}{2}\;\;.\label{eq:T1}
\end{align}
Since point $T$ lies at the intersection of segments $\overline{AC}$ and $\overline{ST}$, its $w$-coordinate can be found by making Equations (\ref{eq:AC}) and (\ref{eq:ST}) equal and solving:

\newcommand{\Sg}{S_\gamma}
\begin{align}
\mg w_T - (1+\mg)\Cg &= \frac{\left(l_S+\rxm\right)w_T-(w_S+l_S)\rxm}{w_S-\rxm}\;\;,\nonumber\\
w_T\left(\left(w_S-\rxm\right)\mg-\left(l_S+\rxm\right)\right) &= (1+\mg)\left(w_S-\rxm\right)\Cg - (w_S+l_S)\rxm\;\;,\nonumber\\
w_T &= \frac{(1+\mg)\left(w_S-\rxm\right)\Cg - (w_S+l_S)\rxm} {(1+\mg)\left(\Sg-\rxm\right)}\;\;.\label{eq:wT}
\end{align}
Substituting Equation (\ref{eq:wT}) in (\ref{eq:T1}),
\begin{align*}
T_1 = \frac{(1-\mg^2)\left(w_S-\rxm\right)\Cg - (1-\mg)(w_S+l_S)\rxm + (1+\mg)^2(\Sg-\rxm)\Cg}{2(1+\mg)\left(\Sg-\rxm\right)}\;\;,
\end{align*}
and, finally, after solving, the left uncertainty is
\begin{align}
u_l&=2(T_1-S_1) = (w_T-l_T) - (w_S-l_S)\;\;,\nonumber\\
u_l&=2\frac{\left(\rxm-S_1\right)(\Cg-\Sg)}{\left(\rxm-\Sg\right)}\label{eq:ludeS}
\end{align}

\begin{figure}[!h]
	\begin{center}
		\includegraphics{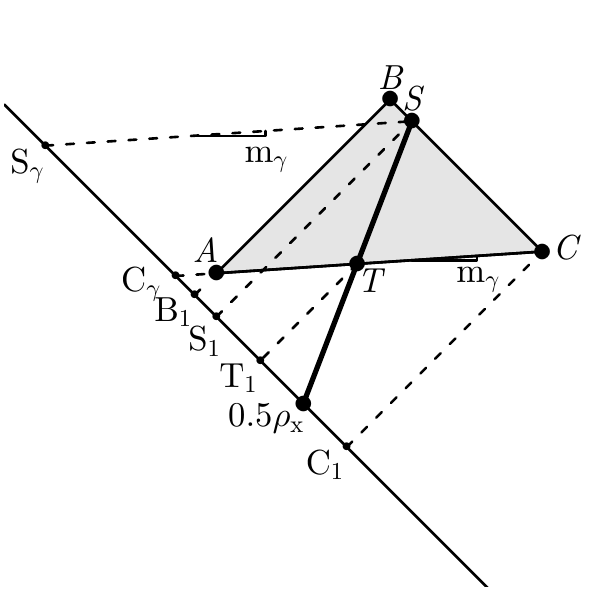}
	\end{center}
	\caption{Expansion of Figure \ref{fig:lunclem} to present the geometry of all relevant terms in the expression of left uncertainty for an arbitrary $S$, Equation (\ref{eq:ludeS}).}
	\label{fig:luncexp}
\end{figure}

Figure \ref{fig:luncexp} expands Figure \ref{fig:lunclem} to include the $m_\gamma$-projections of points $S$ and $C$ to the $w=-l$ line. Observe that the projections and gain are ordered in the sequence
\begin{align}
\Sg < \Cg \leq B_1 \leq S_1 \leq T_1 \leq \rxm \leq C_1\;\;.\label{eq:seqlunc}
\end{align}

The ordering of the terms between $B_1$ and $C_1$ results from setting $\rho_x$ and having left uncertainty. $B_1=S_1$ occurs when the point $S$ lies on the $\overline{AB}$ segment (instead of on $\overline{BC}$ as pictured). $S_1=T_1$ is only possible if the gain is set to $\rho_x=2B_1$, in which case all four points are concurrent.

To see that $\Cg\leq B_1$, observe that
\begin{align*}
A_\gamma = \Cg &\leq B_1 =A_1\;\;,\\
\frac{\mg w_A-l_A}{\mg+1} &\leq \frac{w_A-l_A}{2}\;\;.
\intertext{Since $\mg\geq0$, we can cross-multiply the denominatiors to obtain}
\mg(w_A+l_A)&\leq w_A+l_A\;\;.
\end{align*}
Given $w_A+l_A$ is positive (i.e., $A$ belongs to the $w-l$ space), this requires $\mg\leq1$, which holds by definition of enclosing triangle (see condition 5 in Definition \ref{def:encltr}).

Finally, to ensure that $\Sg\leq\Cg$, the point $B$ must not lie on the $\overline{AC}$ segment, but this condition holds trivially for a non-degenerate enclosing triangle with vertices $A$, $B$, and $C$.

We are now ready to prove all items in the Lemma. 

\paragraph{1.} To see that left uncertainty is maximum when $S=B$, there are two cases:
If $S=(w_S,l_S)\in\overline{AB}$, consider another point, slightly \emph{above} on that segment, with coordinates $S'=(w_S+\varepsilon,l_S+\varepsilon)$, for some valid $\varepsilon>0$. Obviously, $A_1=B_1=S_1=S'_1\leq\rxm$. Leaving all other parameters equal, for the case with $S'$ to have larger left uncertainty than $S$, from Equation \ref{eq:ludeS},
\newcommand{\Spg}{S'_\gamma}
\begin{align*}
2\frac{\left(\rxm-S_1\right)(\Cg-\Sg)}{\left(\rxm-\Sg\right)} < 2\frac{\left(\rxm-S'_1\right)(\Cg-\Spg)}{\left(\rxm-\Spg\right)}\;\;.
\end{align*}
Since $S_1=S'_1$ and, from Equation (\ref{eq:seqlunc}), both are smaller than $\rxm$, the first term in the numerator can be cancelled from both sides. Additionally, also from  Equation (\ref{eq:seqlunc}), since $\rxm$ is larger than $\Sg$ (and $\Spg$), we can cross-multiply the denominators, to obtain
\begin{align*}
\rxm(\Sg-\Spg) &> \Cg(\Sg-\Spg)\;\;.
\intertext{The common term is equivalent (by direct substitution) to}
\Sg-\Spg &= \varepsilon\frac{1-\mg}{1+\mg}\;\;.
\intertext{Given that $\mg\leq1$ and $\varepsilon>0$ this term is nonnegative, so the only requirement is for}
\rxm &> \Cg\;\;,
\end{align*}
which holds after Equation \ref{eq:seqlunc}. Thus, the position of $S$ on $\overline{AB}$ for maximum left uncertainty is the furthest possible \emph{up}, that is, on point $B$. 

\newcommand{\mb}{m_\beta}
Conversely, if $S=(w_S,l_S)\in\overline{BC}$ consider another point in the segment, closer to $B$, with coordinates $S'=(w_S+\frac{\varepsilon}{\mb},l_S+\varepsilon)$, where $|\mb|\geq1$ is the slope of the (line that contains that) segment and, once more, $\varepsilon>0$. For the case with $S'$ to have larger left uncertainty than $S$, the following must hold:
\begin{align}
2\frac{\left(\rxm-S_1\right)(\Cg-\Sg)}{\left(\rxm-\Sg\right)} &< 2\frac{\left(\rxm-S'_1\right)(\Cg-\Spg)}{\left(\rxm-\Spg\right)}\;\;.\label{eq:lula}
\intertext{In this case, although the first terms in the numerator of both sides aren't equal, we still can omit them, since it is easy to show that}
\left(\rxm-S_1\right) &< \left(\rxm-S'_1\right)\;\;,\nonumber
\intertext{which reduces to}
\rho_x-(w_S-l_S) &< \rho_x-(w_S'-l_S')\;\;,\nonumber\\
\varepsilon-\frac{\varepsilon}{\mb}&>0\;\;,\label{eq:emb}
\end{align}
which is true: if $\mb\leq-1$ both summands become positive and if $\mb>1$ the second term is strictly smaller than the first, so their difference is positive. ($\mb=1$ corresponds to a degenerate case that escapes the scope of this Lemma). 

Thus, Equation (\ref{eq:lula}), as in the case above, simplifies to
\begin{align*}
\rxm(\Sg-\Spg) &> \Cg(\Sg-\Spg)\;\;,
\intertext{and in this case $\Sg-\Spg > 0$ reduces itself to the form in Equation (\ref{eq:emb}), which was proved above. Hence, once more we require that $\rxm>\Cg$, which was proved already. Thus, for points in the $\overline{BC}$ segment, left uncertainty is maximum when $S=B$.}
\end{align*}

\paragraph{2.} From the preceding discussion, substituting $S$ for $B$ in Equation (\ref{eq:ludeS}),
\begin{align*}
u_l^*=2\frac{(\frac{\rho_x}{2}-B_1)}{(\frac{\rho_x}{2}-B_\gamma)}\,(C_\gamma - B_\gamma)\;\;.
\end{align*}.

\paragraph{3.} To prove the increasing monotonicity of this expression with respect to $\rho_x$, we fix all other parameters and study the effect of moving the gain from $\rho_x$ to $\rho_x+\varepsilon$, for some $\varepsilon>0$. If the function is monotonically decreasing, the following must hold,
\newcommand{\Bg}{B_\gamma}
\newcommand{\rxem}{\frac{\rho_x+\varepsilon}{2}}
\begin{align*}
2\frac{\left(\rxm-B_1\right)(\Cg-\Bg)}{\left(\rxm-\Bg\right)} &< 2\frac{\left(\rxem-B_1\right)(\Cg-\Bg)}{\left(\rxem-\Bg\right)}\;\;.
\end{align*}

After cancelling the equal terms on both sides and cross-multiplying (the positivity of the denominators was already shown above), the expression immediately reduces to
\begin{align*}
\frac{\varepsilon}{2}\Bg < \frac{\varepsilon}{2}B_1\;\;,
\end{align*}
which trivially holds by the positivity of $\varepsilon$ and Equation (\ref{eq:seqlunc}).

\paragraph{4.} Since $u_l^*$ increases monotonically with $\rho_x$, its minimum value must occur for the smallest $\rho_x$, namely $\rxm=B_1$, for which the first factor in the numerator, and hence the expression, becomes zero.

\hfill\ensuremath{\blacksquare}

\newpage
\appendix
\section*{Appendix C.}
\label{app:righttunc}

In this Appendix we prove the right uncertainty Lemma \ref{lem:rightunc} from Section~\ref{sec:minmaxunc}:

\noindent{\bf Lemma}
{\it For a given enclosing triangle with vertices $A$, $B$, $C$,
	\begin{enumerate}
		\item for any $A_1=B_1\leq\frac{\rho_x}{2}\leq C_1$, the maximum possible right uncertainty, $u_r^*$, is
		\begin{align}
		u_r^*&=\frac{2s\sqrt{abcd(\frac{\rho_x}{2}-C_\gamma)(\frac{\rho_x}{2}-C_\beta)}+ad(\frac{\rho_x}{2}-C_\gamma) +bc(\frac{\rho_x}{2}-C_\beta)}{c}\;\;,\\
		\intertext{with}
		s&=\sign(m_\beta-m_\gamma)\;\;,\nonumber\\
		a&=(1-m_\beta)\;\;,\nonumber\\
		b&=(1+m_\beta)\;\;,\nonumber\\
		c&=(1-m_\gamma)\;\;,\nonumber\\
		d&=(1+m_\gamma)\;\;,\nonumber\\
		e&=(d-b)=(m_\gamma-m_\beta)\;\;.\nonumber
		\end{align}
		\item This is a monotonically decreasing function of $\rho_x$,
		\item whose minimum value is zero, attained when $\frac{\rho_x}{2}=C_1$.
	\end{enumerate}
}

\noindent{\bf Proof} 
Figure \ref{fig:runclem} illustrates the geometry of right uncertainty.

\begin{figure}[!h]
	\begin{center}
		\includegraphics{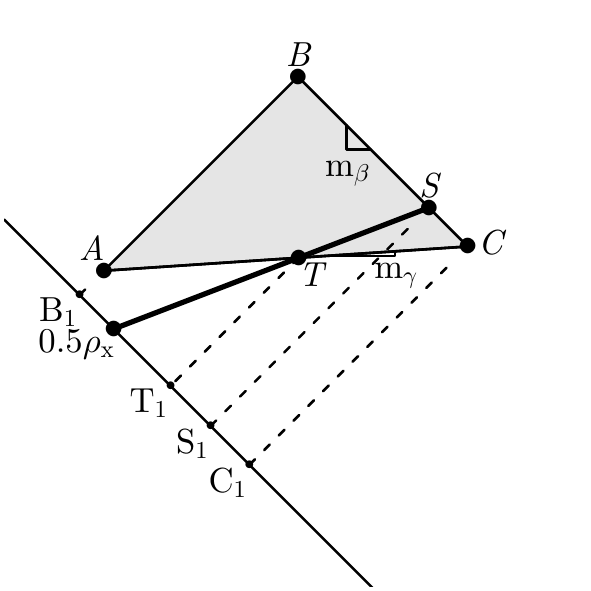}
	\end{center}
	\caption{Geometry and notation for the proof of the rightt uncertainty Lemma. If gain is fixed to some $B_1\leq\frac{\rho_x}{2}\leq C_1$, the point $S$ lies on the $\overline{BC}$ segment (which has slope $m_\beta$), $T$ lies on the $\overline{AC}$ segment (which has slope $m_\gamma$), and the rightt uncertainty is $u_r=2(S_1-T_1)$.}
	\label{fig:runclem}
\end{figure}

In this case, we have the point $S\in\overline{BC}$ and $T=\overline{\rxm S}\bigcap\overline{AC}$. The resulting new right uncertainty is $u_r=2(S_1-T_1)$. Following the same derivation as Equation (\ref{eq:ludeS}) in Appendix B,
\begin{align}
u_r=2\frac{\left(S_1-\rxm\right)(\Cg-\Sg)}{(\rxm-\Sg)}\;\;.\label{eq:rudeS}
\end{align}
From this expression it is clear that in the degenerate cases when $S_1=\rxm$ (reduction of the enclosing triangle to a line segment with slope unity) and $S=C$ (reduction to the point $C$), right uncertainty is zero.

As was the case in for left uncertainty, Figure \ref{fig:runclem} suggests an ordering of projections to the $w=-l$ line that, for non-degenerate initial and reduced enclosing triangles, is
\newcommand{\Cb}{C_\beta}
\begin{align}
\Sg < \Cg < B_1 < \rxm < T_1 < S_1 < C_1 < \Cb\;\;.\label{eq:ordrunc}
\end{align}

As a direct consequence of this ordering, not only is $u_r=0$ at the two extreme positions of $S$, but it is also strictly positive for intermediate points. For an arbitrary, fixed $\rho_x$, then, the location of $S$ of interest is that of maximum right uncertainty. Since point $S$ lies on a segment ($\overline{BC}$) of a line whose expression is known, we can solve the problem for just one of its components. To find the maximizer, we take the derivative (hereon denoted with $\cdot'$) of $u_r$ with respect to $w_S$,
\begin{align*}
u_r'&=2\frac{\left[\left(s_1-\rxm\right)(\Cg-\Sg)\right]'\left(\rxm-\Sg\right) - \left(s_1-\rxm\right)(\Cg-\Sg)\left(\rxm-\Sg\right)'} {\left(\rxm-\Sg\right)^2}\;\;,\\
u_r'&=2\frac{S_1'(\Cg-\Sg)\left(\rxm-\Sg\right) + \Sg'\left(S_1-\rxm\right)\left(\Cg-\rxm\right)}{\left(\rxm-\Sg\right)^2}\;\;.
\end{align*}
At the optimizer, the numerator of this fraction must equal zero.

Since $S$ lies on the line with slope $\mb$ that joins $B$ and $C$,
\begin{align*}
l_S&=\mb w_S-\mb w_C+l_C\;\;\,\\
l_S&=\mb w_S-(1+\mb)\Cb\;\;,\\
l_S'&=\mb\;\;.
\end{align*}
Additionally, from the definition of the projections of $S$,
\begin{align*}
S_1&=\frac{w_S-l_S}{2}\;\;,\\
S_1'&=\frac{1-\mb}{2}
\intertext{and}
\Sg&=\frac{\mg w_S-l_S}{\mg+1}\;\;,\\
\Sg'&=\frac{\mg-\mb}{\mg+1}\;\;.
\end{align*}

If we define
\begin{align*}
a&=1-\mb\;\;,\\
b&=1+\mb\;\;,\\
c&=1-\mg\;\;,\\
d&=1+\mg\;\;,
\intertext{and}
e&=\mg-\mb=d-b\;\;, 
\end{align*}
then, the expression for $u_r'=0$ at the optimizer,
\newcommand{\wss}{w_s^*}
\begin{align*}
S_1'(\Cg-\Sg^*)\left(\rxm-\Sg^*\right) + \Sg'\left(S_1^*-\rxm\right)\left(\Cg-\rxm\right) &= 0\;\;,\\
\intertext{becomes}
a\left(d\rxm-e\wss-b\Cb\right)(d\Cg-e\wss-b\Cb) + de\left(a\wss+b\Cb-\rho_x\right)\left(\Cg-\rxm\right) &= 0\;\;.
\end{align*}
This is a quadratic expression on $\wss$, of the form
\begin{align*}
\mathcal{A}(\wss)^2 + \mathcal{B}\wss +\mathcal{C} = 0\;\;,
\end{align*}
with
\begin{align*}
\mathcal{A}&=ae^2\;\;,\\
\mathcal{B}&=2ae\left(b\Cb-d\rxm\right)\;\;,\\
\mathcal{C}&=(ad^2-2de)\rxm\Cg - (abd+bde)\rxm\Cb+(bde-abd)\Cb\Cg+ ab^2\Cb^2 + 2de\left(\rxm\right)^2\;\;.
\end{align*}
Applying the quadratic formula and solving,
\begin{align}
\wss &= \frac{2ae\left(d\rxm-b\Cb\right)\pm \sqrt{4abcde^2\left(\rxm-\Cg\right)\left(\rxm-\Cb\right)}}{2ae^2}\;\;,\nonumber\\
\wss &=\frac{d\rxm-b\Cb}{e} + \frac{s}{ae} \sqrt{abcd\left(\rxm-\Cg\right)\left(\rxm-\Cb\right)}\;\;,\label{eq:wss}
\end{align}
where $s$ is the appropriate sign for $\wss$ to be a maximizer. Notice that, for an enclosing triangle, the expression inside the radical is always non-negative. Since $0\leq\mg\leq1$, $c$ is positive and $d$ non-negative. Also, as $|\mb|\leq1$, $a$ and $b$ are of opposite signs and, from the ordering in Equation (\ref{eq:ordrunc}), $\left(\rxm-\Cg\right)$ is positive and $\left(\rxm-\Cb\right)$ negative, so there will be, in any case, two negative factors. In order to determine $s$, we take the second derivative of $u_r$ and evaluate it in the point $\wss$, which simplifies to
\begin{align*}
u_r''(\wss) &= \frac{2a^2e}{d^2s\sqrt{abcd\left(\rxm-\Cg\right)\left(\rxm-\Cb\right)}}\;\;.
\end{align*}
For the second order necessary condition for a maximum to hold,
\begin{align*}
u_r''(\wss) &< 0\;\;,
\intertext{which requires that}
s&=-\sign(e)=\sign(\mb-\mg)\;\;.
\end{align*}
Observe that for the opposite sign of $s$, the corresponding $w_S$ is a minimizer.

Substituting Equation (\ref{eq:wss}) in Equation (\ref{eq:rudeS}),
\begin{align*}
u_r^*=\frac{2s\sqrt{abcd\left(\rxm-\Cg\right)\left(\rxm-\Cb\right)} + ad\left(\rxm-\Cg\right) + bc\left(\rxm-\Cb\right)}{e}\;\;.
\end{align*}

To show that this is a monotonically decreasing function of $\rho$, we factor $(-s)$ and use the fact that, since it is a sign, $s^2=(-s)^2=1$, to obtain
\begin{align}
u_r^*&=\frac{ad(-s)\left(\rxm-\Cg\right) - 2\sqrt{ad(-s)\left(\rxm-\Cg\right)bc(-s)\left(\rxm-\Cb\right)} + bc(-s)\left(\rxm-\Cb\right)}{(-s)e}\;\;,\nonumber\\
u_r^*&=\frac{\left(\sqrt{-sad\left(\rxm-\Cg\right)} - \sqrt{-sbc\left(\rxm-\Cb\right)}\right)^2}{|e|}\;\;.\label{eq:rusq}
\end{align}
The expressions inside both radicals are always positive. For a non-degenerate enclosing triangle, recall that $\Cb>\rxm>\Cg$, $a$ and $b$ are positive and $|\mg|>1$. If $\mg>1$, then $s=-1$, so all factors in the first root are positive, and $c$ is negative, so there are exactly two negative factors in the second root. Conversely, if $\mg<-1$, then $s=1$ and $d$ is negative, so both roots contain exactly two negative factors.

To prove decreasing monotonicity, it is sufficient to show that the first derivative of the function is negative in the range of interest. Taking it,
\begin{align*}
\diff{u_r^*}{\rxm} &= \frac{1}{|e|} \left(\sqrt{-sad\left(\rxm-\Cg\right)} - \sqrt{-sbc\left(\rxm-\Cb\right)}\right) \left(\frac{-sad\sqrt{-sbc\left(\rxm-\Cb\right)} + sbc\sqrt{-sad\left(\rxm-\Cg\right)}} {\sqrt{abcd\left(\rxm-\Cg\right)\left(\rxm-\Cb\right)}}\right)\;\;.
\end{align*}
By a similar reasoning to the positivity of the arguments of the roots, the term in the last parenthesis can be shown to be always positive. Thus, for the derivative to be negative,
\begin{align*}
\sqrt{-sad\left(\rxm-\Cg\right)} &< \sqrt{-sbc\left(\rxm-\Cb\right)}\;\;.
\intertext{Since both terms are positive an $-se=|e|$, this is equivalent to}
\frac{ad\Cg - bc\Cb}{2e} &> \rxm\;\;.
\intertext{Substituting the definition of all the terms and solving, this reduces to}
\frac{(\mg-\mb)w_C - (\mg-\mb)l_C}{2(\mg-\mb)} &> \rxm\;\;,\\
C_1 &> \rxm\;\;,
\intertext{which, by the ordering in Equation (\ref{eq:ordrunc}), holds.}
\end{align*}

Finally, to see that maximum right uncertainty is zero in the extreme case when $\rxm=C_1$, notice that in that case the only possible position for $S$, and therefore $S^*$ is the point $C$, so Equation (\ref{eq:rudeS}) becomes
\begin{align*}
u_r^*=2\frac{\left(C_1-\rxm\right)(\Cg-\Cg)}{(\rxm-\Cg)}\;\;,
\end{align*}
which trivially equals zero.

\hfill\ensuremath{\blacksquare}

\newpage
\appendix
\section*{Appendix D.}
\label{app:conic}
In this section we include the conic intersection method described in detail by \citet{perwass2008geometric}. The Lemmas that support the derivation of the method are omitted here, but they are clearly presented in that book.

This method finds the intersection points of two non-degenerate conics represented by the symmetric matrices $A$ and $B$.

\begin{enumerate}
	\item Let $M=B^{-1}A$.
	\item Find $\lambda$, any real eigenvalue of $M$. (Since the matrix is $3\times3$, at least one real eigenvalue exists).
	\item Find the degenerate conic $C=A-\lambda B$.
	\item $C$ represents two lines. Find their intersections with either $A$ or $B$.
	\item These (up to four points) are the intersections of the two conics.
\end{enumerate}

In the case of the gain update step in the optimal nudging Algorithm \ref{alg:optnudg}, $A$ and $B$ are the homogeneous forms of the expressions for maximum left and right uncertainty, from Equations (\ref{eq:lunconic}) and (\ref{eq:runconic}).

For the optimal nudging update stage, an additional step is required in order to determine precisely \emph{which} of the four points corresponds to the intersection of the current left and right uncertainty segments of the conics. This is easily done by finding which of the intercepts corresponds to the gain of point in the $w=-l$ line inside the $\overline{A_1B_1}$ segment.

In our experience, this verification is sufficient to find the updated gain and there are no multiple intercepts inside the segment. However, if several competing candidates do appear, an additional verification step my be required, to determine which is the simultaneous unique solution of Equations (\ref{eq:leftunc}) and (\ref{eq:rightunc}).

\newpage
\bibliography{bblgrph}

\begin{thebibliography}{51}
\providecommand{\natexlab}[1]{#1}
\providecommand{\url}[1]{\texttt{#1}}
\expandafter\ifx\csname urlstyle\endcsname\relax
  \providecommand{\doi}[1]{doi: #1}\else
  \providecommand{\doi}{doi: \begingroup \urlstyle{rm}\Url}\fi

\bibitem[Abounadi et~al.(2002)Abounadi, Bertsekas, and
  Borkar]{abounadi2002learning}
J.~Abounadi, D.~Bertsekas, and V.S. Borkar.
\newblock Learning algorithms for markov decision processes with average cost.
\newblock \emph{SIAM Journal on Control and Optimization}, 40\penalty0
  (3):\penalty0 681--698, 2002.

\bibitem[Azar et~al.(2011)Azar, Munos, Ghavamzadeh, and Kappen]{azar2011speedy}
M.G Azar, R.~Munos, M.~Ghavamzadeh, and H.J. Kappen.
\newblock Speedy {Q}-learning.
\newblock \emph{Advances in Neural Information Processing Systems},
  24:\penalty0 2411--2419, 2011.

\bibitem[Bartlett and Tewari(2009)]{bartlett2009regal}
P.L. Bartlett and A.~Tewari.
\newblock Regal: A regularization based algorithm for reinforcement learning in
  weakly communicating {MDPs}.
\newblock In \emph{Proceedings of the Twenty-Fifth Conference on Uncertainty in
  Artificial Intelligence}, pages 35--42. AUAI Press, 2009.

\bibitem[Baykal-G{\"u}rsoy and G{\"u}rsoy(2007)]{baykal2007semi}
M.~Baykal-G{\"u}rsoy and K.~G{\"u}rsoy.
\newblock Semi-{M}arkov decision processes.
\newblock \emph{Probability in the Engineering and Informational Sciences},
  21\penalty0 (04):\penalty0 635--657, 2007.

\bibitem[Bertsekas(1998)]{bertsekas1998new}
D.P. Bertsekas.
\newblock A new value iteration method for the average cost dynamic programming
  problem.
\newblock \emph{SIAM journal on control and optimization}, 36\penalty0
  (2):\penalty0 742--759, 1998.

\bibitem[Bertsekas and Tsitsiklis(1996)]{bertsekas1996neuro}
D.P. Bertsekas and J.N. Tsitsiklis.
\newblock \emph{Neuro-dynamic programming}.
\newblock Athena Scientific, 1996.

\bibitem[Brafman and Tennenholtz(2003)]{brafman2003rmax}
R.I. Brafman and M.~Tennenholtz.
\newblock R-max -- a general polynomial time algorithm for near-optimal
  reinforcement learning.
\newblock \emph{Journal of Machine Learning Research}, 3:\penalty0 213--231,
  2003.

\bibitem[Charnes and Cooper(1962)]{charnes1962programming}
Abraham Charnes and William~W Cooper.
\newblock Programming with linear fractional functionals.
\newblock \emph{Naval Research logistics quarterly}, 9\penalty0 (3-4):\penalty0
  181--186, 1962.

\bibitem[Darken et~al.(1992)Darken, Chang, and Moody]{darken1992learning}
C.~Darken, J.~Chang, and J.~Moody.
\newblock Learning rate schedules for faster stochastic gradient search.
\newblock In \emph{Proceedings of the 1992 IEEE-SP Workshop in Neural Networks
  for Signal Processing}, pages 3--12. IEEE, 1992.

\bibitem[Das et~al.(1999)Das, Gosavi, Mahadevan, and
  Marchalleck]{das1999solving}
T.K. Das, A.~Gosavi, S.~Mahadevan, and N.~Marchalleck.
\newblock Solving semi-{M}arkov decision problems using average reward
  reinforcement learning.
\newblock \emph{Management Science}, 45\penalty0 (4):\penalty0 560--574, 1999.

\bibitem[Even-Dar and Mansour(2004)]{even2004learning}
E.~Even-Dar and Y.~Mansour.
\newblock Learning rates for {Q}-learning.
\newblock \emph{Journal of Machine Learning Research}, 5:\penalty0 1--25, 2004.

\bibitem[Feinberg(1994)]{feinberg1994constrained}
E.A. Feinberg.
\newblock Constrained semi-{M}arkov decision processes with average rewards.
\newblock \emph{Mathematical Methods of Operations Research}, 39\penalty0
  (3):\penalty0 257--288, 1994.

\bibitem[Feinberg and Yang(2008)]{feinberg2008polynomial}
E.A. Feinberg and F.~Yang.
\newblock On polynomial cases of the unichain classification problem for markov
  decision processes.
\newblock \emph{Operations Research Letters}, 36\penalty0 (5):\penalty0
  527--530, 2008.

\bibitem[Ghavamzadeh and Mahadevan(2001)]{ghavamzadeh2001continuous}
M.~Ghavamzadeh and S.~Mahadevan.
\newblock Continuous-time hierarchical reinforcement learning.
\newblock In \emph{In Proceedings of the Eighteenth International Conference on
  Machine Learning}, pages 186--193. Citeseer, 2001.

\bibitem[Ghavamzadeh and Mahadevan(2007)]{ghavamzadeh2007hierarchical}
M.~Ghavamzadeh and S.~Mahadevan.
\newblock Hierarchical average reward reinforcement learning.
\newblock \emph{Journal of Machine Learning Research}, 8:\penalty0 2629--2669,
  2007.

\bibitem[Gosavi(2004)]{gosavi2004reinforcement}
A.~Gosavi.
\newblock Reinforcement learning for long-run average cost.
\newblock \emph{European Journal of Operational Research}, 155\penalty0
  (3):\penalty0 654--674, 2004.

\bibitem[Jaksch et~al.(2010)Jaksch, Ortner, and Auer]{jaksch2010near}
T.~Jaksch, R.~Ortner, and P.~Auer.
\newblock Near-optimal regret bounds for reinforcement learning.
\newblock \emph{Journal of Machine Learning Research}, 11:\penalty0 1563--1600,
  2010.

\bibitem[Jalali and Ferguson(1989)]{jalali1989computationally}
A.~Jalali and M.~Ferguson.
\newblock Computationally efficient adaptive control algorithms for {M}arkov
  chains.
\newblock In \emph{Proceedings of the 28th IEEE Conference on Decision and
  Control}, pages 1283--1288. IEEE, 1989.

\bibitem[Kaelbling et~al.(1996)Kaelbling, Littman, and
  Moore]{kaelbling1996reinforcement}
L.P. Kaelbling, M.L. Littman, and A.W. Moore.
\newblock Reinforcement learning: A survey.
\newblock \emph{Journal of Artificial Intelligence Research}, 4:\penalty0
  237--285, 1996.

\bibitem[Kakade(2003)]{kakade2003sample}
S.M. Kakade.
\newblock \emph{On the sample complexity of reinforcement learning}.
\newblock PhD thesis, University College London, 2003.

\bibitem[Kallenberg(2002)]{kallenberg2002classification}
L.C.M. Kallenberg.
\newblock Classification problems in {MDPs}.
\newblock In Z.~How, J.A. Filar, and A.~Chen, editors, \emph{Markov processes
  and controlled Markov chains}. Kluwer, 2002.

\bibitem[Kearns and Singh(1998)]{kearns1998near}
M.~Kearns and S.~Singh.
\newblock Near-optimal reinforcement learning in polynomial time.
\newblock In \emph{Proceedings of the Fifteenth International Conference on
  Machine Learning}, pages 260--268, 1998.

\bibitem[Kearns and Singh(1999)]{kearns1999finite}
M.~Kearns and S.~Singh.
\newblock Finite-sample convergence rates for {Q}-learning and indirect
  algorithms.
\newblock \emph{Advances in Neural Information Processing Systems}, pages
  996--1002, 1999.

\bibitem[Kearns and Singh(2002)]{kearns2002near}
M.~Kearns and S.~Singh.
\newblock Near-optimal reinforcement learning in polynomial time.
\newblock \emph{Machine Learning}, 49:\penalty0 209--232, 2002.

\bibitem[Kearns et~al.(2002)Kearns, Mansour, and Ng]{kearns2002sparse}
M.~Kearns, Y.~Mansour, and A.Y. Ng.
\newblock A sparse sampling algorithm for near-optimal planning in large
  {M}arkov decision processes.
\newblock \emph{Machine Learning}, 49:\penalty0 193--208, 2002.

\bibitem[Kemeny and Snell(1960)]{kemeny1960finite}
John~G. Kemeny and James~Laurie Snell.
\newblock \emph{Finite markov chains}.
\newblock van Nostrand, 1960.

\bibitem[Lai and Robbins(1985)]{lai1985asymptotically}
T.L. Lai and H.~Robbins.
\newblock Asymptotically efficient adaptive allocation rules.
\newblock \emph{Advances in Applied Mathematics}, 6\penalty0 (1):\penalty0
  4--22, 1985.

\bibitem[Li et~al.(2008)Li, Littman, and Walsh]{li2008knows}
L.~Li, M.L. Littman, and T.J. Walsh.
\newblock Knows what it knows: a framework for self-aware learning.
\newblock In \emph{Proceedings of the 25th international conference on Machine
  learning}, pages 568--575. ACM, 2008.

\bibitem[Mahadevan(1996)]{mahadevan1996average}
S.~Mahadevan.
\newblock Average reward reinforcement learning: Foundations, algorithms, and
  empirical results.
\newblock \emph{Machine learning}, 22\penalty0 (1-3):\penalty0 159--195, 1996.

\bibitem[Mahadevan(1994)]{mahadevan1994discount}
Sridhar Mahadevan.
\newblock To discount or not to discount in reinforcement learning: A case
  study comparing r learning and q learning.
\newblock In \emph{ICML}, pages 164--172, 1994.

\bibitem[Perwass(2008)]{perwass2008geometric}
Christian Perwass.
\newblock \emph{Geometric algebra with applications in engineering}, volume~4.
\newblock Springer, 2008.

\bibitem[{Puterman}(1994)]{puterman1994markov}
M.L. {Puterman}.
\newblock \emph{Markov decision processes: discrete stochastic dynamic
  programming}.
\newblock Wiley-Interscience, 1994.

\bibitem[Rao and Whiteson(2012)]{rao2012v}
K.~Rao and S.~Whiteson.
\newblock {V-MAX}: tempered optimism for better {PAC} reinforcement learning.
\newblock In \emph{Proceedings of the 11th International Conference on
  Autonomous Agents and Multiagent Systems}, pages 375--382. International
  Foundation for Autonomous Agents and Multiagent Systems, 2012.

\bibitem[Ross(1970)]{ross1970average}
S.M. Ross.
\newblock Average cost semi-{M}arkov decision processes.
\newblock \emph{Journal of Applied Probability}, 7:\penalty0 649--656, 1970.

\bibitem[Schwartz(1993)]{schwartz1993reinforcement}
A.~Schwartz.
\newblock A reinforcement learning method for maximizing undiscounted rewards.
\newblock In \emph{Proceedings of the Tenth International Conference on Machine
  Learning}, pages 298--305. Amherst, Massachusetts, 1993.

\bibitem[Singh(1994)]{singh1994reinforcement}
S.~Singh.
\newblock Reinforcement learning algorithms for average-payoff {M}arkovian
  decision processes.
\newblock In \emph{Proceedings of the twelfth national conference on Artificial
  intelligence (vol. 1)}, pages 700--705. American Association for Artificial
  Intelligence, 1994.

\bibitem[Strehl(2007)]{strehl2007probably}
A.L. Strehl.
\newblock \emph{Probably approximately correct {(PAC)} exploration in
  reinforcement learning}.
\newblock PhD thesis, Rutgers, The State University of New Jersey, 2007.

\bibitem[Strehl and Littman(2005)]{strehl2005theoretical}
A.L. Strehl and M.L. Littman.
\newblock A theoretical analysis of model-based interval estimation.
\newblock In \emph{Proceedings of the 22nd international conference on Machine
  learning}, pages 856--863. ACM, 2005.

\bibitem[Strehl et~al.(2006)Strehl, Li, Wiewiora, Langford, and
  Littman]{strehl2006pac}
A.L. Strehl, L.~Li, E.~Wiewiora, J.~Langford, and M.L. Littman.
\newblock {PAC} model-free reinforcement learning.
\newblock In \emph{Proceedings of the 23rd international conference on Machine
  learning}, pages 881--888. ACM, 2006.

\bibitem[Strehl et~al.(2009)Strehl, Li, and Littman]{strehl2009reinforcement}
A.L. Strehl, L.~Li, and M.L. Littman.
\newblock Reinforcement learning in finite {MDPs}: {PAC} analysis.
\newblock \emph{Journal of Machine Learning Research}, 10:\penalty0 2413--2444,
  2009.

\bibitem[Sutton and Barto(1998)]{sutton1998reinforcement}
R.S. Sutton and A.G. Barto.
\newblock \emph{Reinforcement learning: an introduction}.
\newblock The MIT Press, 1998.

\bibitem[Szepesv{\'a}ri(1998)]{szepesvari1998asymptotic}
C.~Szepesv{\'a}ri.
\newblock The asymptotic convergence-rate of {Q}-learning.
\newblock \emph{Advances in Neural Information Processing Systems}, pages
  1064--1070, 1998.

\bibitem[Szepesv{\'a}ri(2010)]{szepesvari2010algorithms}
C.~Szepesv{\'a}ri.
\newblock \emph{Algorithms for Reinforcement Learning}.
\newblock Morgan \& Claypool Publishers, 2010.

\bibitem[Tadepalli and Ok(1998)]{tadepalli1998model}
P.~Tadepalli and D.~Ok.
\newblock Model-based average reward reinforcement learning.
\newblock \emph{Artificial Intelligence}, 100\penalty0 (1):\penalty0 177--224,
  1998.

\bibitem[Tsitsiklis(2007)]{tsitsiklis2007np}
J.N. Tsitsiklis.
\newblock {NP}-hardness of checking the unichain condition in average cost
  {MDPs}.
\newblock \emph{Operations Research Letters}, 35\penalty0 (3):\penalty0
  319--323, 2007.

\bibitem[Uribe et~al.(2011)Uribe, Lozano, Shibata, and
  Anderson]{uribe2011discount}
R.~Uribe, F.~Lozano, K.~Shibata, and C.~Anderson.
\newblock Discount and speed/execution tradeoffs in {M}arkov decision process
  games.
\newblock In \emph{Computational Intelligence and Games (CIG), 2011 IEEE
  Conference on}, pages 79--86, 2011.

\bibitem[Van~Hasselt and Wiering(2007)]{van2007reinforcement}
Hado Van~Hasselt and Marco~A Wiering.
\newblock Reinforcement learning in continuous action spaces.
\newblock In \emph{Approximate Dynamic Programming and Reinforcement Learning,
  2007. ADPRL 2007. IEEE International Symposium on}, pages 272--279. IEEE,
  2007.

\bibitem[Walsh et~al.(2010)Walsh, Goschin, and Littman]{walsh2010integrating}
T.J. Walsh, S.~Goschin, and M.L. Littman.
\newblock Integrating sample-based planning and model-based reinforcement
  learning.
\newblock In \emph{Proceedings of the twenty-fifth AAAI conference on
  artificial intelligence}, 2010.

\bibitem[Watkins(1989)]{watkins1989learning}
C.J.C.H Watkins.
\newblock \emph{Learning from delayed rewards}.
\newblock PhD thesis, University of Cambridge, 1989.

\bibitem[Watkins and Dayan(1992)]{watkins1992q}
C.J.C.H. Watkins and P.~Dayan.
\newblock Q-learning.
\newblock \emph{Machine learning}, 8\penalty0 (3-4):\penalty0 279--292, 1992.

\bibitem[White(1963)]{white1963dynamic}
DJ~White.
\newblock Dynamic programming, {M}arkov chains, and the method of successive
  approximations.
\newblock \emph{Journal of Mathematical Analysis and Applications}, 6\penalty0
  (3):\penalty0 373--376, 1963.

\end{thebibliography}

\end{document}